%% file: neurips_2026.tex
\documentclass{article}

\PassOptionsToPackage{numbers,square}{natbib}
 \usepackage[main,final]{neurips_2026}

\input{math_commands.tex}

\usepackage[utf8]{inputenc} 
\usepackage[T1]{fontenc}
\usepackage{hyperref}       
\usepackage{url}            
\usepackage{booktabs}       
\usepackage{multirow}       
\usepackage{graphicx}       
\usepackage{needspace}
\usepackage{wrapfig}        
\usepackage{amsfonts}       
\usepackage{nicefrac}       
\usepackage{microtype}      
\usepackage{titletoc}       
\usepackage[table]{xcolor}

\usepackage[most]{tcolorbox}

\renewcommand{\answerYes}[1][]{[Yes]#1}
\renewcommand{\answerNo}[1][]{[No]#1}
\renewcommand{\answerNA}[1][]{[N/A]#1}
\hypersetup{pdftitle={World Models as Adversaries: Multi-Agent Self-Play Fine-Tuning for Robust Motion Planning},pdfauthor={Tong Nie; Yuewen Mei; Junlin He; Yihong Tang; Jian Sun; Wei Ma}}
\definecolor{bestblue}{HTML}{E6F2FF}
\definecolor{nogainorange}{HTML}{FFE6CC}

\definecolor{deltablue}{HTML}{0000FF}
\definecolor{deltaorange}{HTML}{C65A00}
\definecolor{myblue}{HTML}{2F6FDB}

\newcommand{\gain}[2]{\cellcolor{bestblue}\shortstack{#1\\[-1pt]{\scriptsize\textcolor{deltablue}{(+#2)}}}}
\newcommand{\nogain}[2]{\cellcolor{nogainorange}\shortstack{#1\\[-1pt]{\scriptsize\textcolor{deltaorange}{(#2)}}}}
\newcommand{\refscore}[1]{\cellcolor{gray!20}#1}
\newcommand{\normalref}[1]{\cellcolor{gray!15}#1}
\newcommand{\deltacell}[1]{\cellcolor{bestblue}#1}
\newcommand{\normaldiag}[1]{\cellcolor{gray!15}#1}
\newcommand{\awmdiag}[1]{\cellcolor{bestblue}#1}

\usepackage{amsmath}

\usepackage{enumitem}
\usepackage{algorithm}
\usepackage{algpseudocode}

\newtheorem{theorem}{Theorem}[section]
\newtheorem{proposition}[theorem]{Proposition}
\newtheorem{lemma}[theorem]{Lemma}
\newtheorem{corollary}[theorem]{Corollary}

\newtheorem{remark}[theorem]{Remark}
\tcbset{
  awmtheorybox/.style={
    breakable,
    colback=myblue!8,
    colframe=myblue!8,
    boxrule=0pt,
    sharp corners=all,
    boxsep=0pt,
    left=5pt,
    right=5pt,
    top=5pt,
    bottom=5pt,
    before skip=2pt,
    after skip=2pt
  }
}
\tcolorboxenvironment{theorem}{awmtheorybox}
\tcolorboxenvironment{proposition}{awmtheorybox}
\tcolorboxenvironment{lemma}{awmtheorybox}
\tcolorboxenvironment{corollary}{awmtheorybox}
\tcolorboxenvironment{definition}{awmtheorybox}
\tcolorboxenvironment{assumption}{awmtheorybox}
\tcolorboxenvironment{remark}{awmtheorybox}

\title{World Models as Adversaries: Multi-Agent Self-Play Fine-Tuning for Robust Motion Planning}

\author{%
Tong Nie$^{1,2}$\thanks{Equal contribution.}
\quad{}Yuewen Mei$^{1}$\footnotemark[1]
\quad{}Junlin He$^{2}$
\quad{}Yihong Tang$^{3,4}$\\[1mm]
\textbf{Jian Sun}$^{1}$\thanks{Corresponding author: Jian Sun (\texttt{sunjian@tongji.edu.cn}).}
\quad\textbf{Wei Ma}$^{2}$\\[2mm]
$^{1}$Tongji University, Shanghai, China\\
$^{2}$The Hong Kong Polytechnic University, Hong Kong SAR, China\\
$^{3}$McGill University, Montreal, QC, Canada\quad
$^{4}$Mila-Quebec AI Institute, Montreal, QC, Canada
}

\begin{document}

\maketitle

\begin{abstract}
Robust motion planning in dense traffic requires autonomous vehicles to interact in rare and safety-critical scenarios that are underrepresented in naturalistic driving data. Although adversarial training offers a feasible solution, existing methods often rely on external scenario generators, heuristic perturbations, or simulator-heavy rollouts, which makes them difficult to integrate with modern autoregressive planners. 
Here, we cast adversarially robust planner learning as a constrained min-max game and propose \textit{Adversarial World Modeling (AWM)}, a theoretically grounded multi-agent self-play fine-tuning framework. 
Since solving the exact game is intractable, AWM introduces a principled decoupled solver. In the inner minimization, the planner's predictive world model is converted into a role-conditioned adversary that learns sparse, scene-adaptive attack coalitions via counterfactual credit assignment. In the outer maximization, the ego planner optimizes a regret-aware robust best response against the frozen AWM, utilizing tail-risk weighting and reference-anchored trust regions to improve hard-case recovery while preserving nominal driving behavior. 
Experiments on the nuPlan and InterPlan benchmarks demonstrate that our method generates transferable adversarial interactions and yields a robust planner that achieves competitive closed-loop performance in both nominal and highly interactive long-tail scenarios. 
Theoretical analysis justifies the decoupled solver and the main optimization components.

\end{abstract}

\vspace{-18pt}
\section{Introduction}
\label{sec:intro}

\vspace{-10pt}
Developing robust motion planners for closed-loop autonomous driving remains a central open problem~\cite{teng2023motion,karkus2025beyond}. Recently, the leading paradigm has shifted from rule-based systems toward generative motion modeling \cite{seff2023motionlm, philion2023trajeglish,wu2024smart,zheng2025diffusion,zhang2025carplanner}. 
Drawing inspiration from Large Language Models (LLMs), modern discrete architectures  \cite{wu2024smart,zhou2024behaviorgpt,tang2025plan,zhang2025closed,ye2025dap} treat map elements and motion primitives as discrete tokens, exhibiting remarkable performance when trained on large-scale human driving logs.
However, these autoregressive planners are primarily optimized against nominal, naturalistic data distributions.
Consequently, they remain brittle on interactive long-tail cases that are rare but safety-critical, such as aggressive cut-ins, coordinated blocking, and collision-inducing maneuvers from surrounding vehicles \cite{feng2021intelligent,liu2024curse}.
To expose and rectify these vulnerabilities, adversarial training offers a compelling solution by actively searching hard scenarios to robustify the planner \cite{gao2026foundation}.

Despite the promise of adversarial training, existing frameworks are not well aligned to modern autoregressive architecture. 
Many existing adversarial frameworks rely on external adversary agents, heuristic scenario perturbations, or simulator-intensive rollouts~\cite{zhang2023cat, nie2025steerable, liu2025adv,stoler2025seal,nie2026adv}. 
These attacks are often designed around low-dimensional control policies, fixed objects of interest, or handcrafted rules. 
As a result, they are infeasible to scale to the high-dimensional tokenized generation paradigm. 
Rather than forcing incompatible external attackers into this paradigm, we argue that an endogenous solution lies in the environment simulator itself. 
We observe that \textit{behavioral World Models}~\cite{wu2024smart,zhang2025closed}, which inherently learn multi-agent traffic evolution as sequential token predictions, are naturally suited to bridge this gap. 
Operating on the same autoregressive architecture, motion vocabulary, and scene representation as the ego planner, they can be repurposed to parameterize both naturalistic and adversarial traffic within a unified motion manifold. 
Consequently, formulating the environment as a learnable generative process enables the integration of adversarial interactions directly into the closed-loop training rollout, bypassing the need for a separate handcrafted attacker.

However, translating this intuition into a trainable framework poses three significant optimization challenges.
First, coordinating synergistic adversaries in a multi-agent environment can introduce a combinatorial credit assignment ambiguity issue~\cite{busoniu2008comprehensive,feng2023dense}.  
Isolating the fine-grained marginal contribution of individual agents within a world model is difficult in high-dimensional action spaces. Standard strategies such as reward-sharing \cite{peng2024improving,pei2025advancing} struggle to resolve \emph{who} should initiate the attack, \emph{how} the attack emerges, and \emph{whether} multiple agents should coordinate. 
Second, simultaneous min-max optimization of two large autoregressive models is notoriously unstable~\cite{mescheder2017numerics,daskalakis2017training}.
Updating a planning policy against a continuously shifting adversary model often leads to gradient divergence due to the non-stationary adversarial distributions. 
Third, unconstrained adversarial fine-tuning can degrade nominal driving performance. Without explicit behavioral regularization, a planner may become overly conservative on hard cases while sacrificing driving skills in normal traffic~\cite{feng2026breaking}.

To bridge these gaps, we propose \emph{Adversarial World Modeling} (AWM), a multi-agent self-play fine-tuning framework that formulates adversarially robust planner training as a constrained min-max game over a shared tokenized rollout structure. 
To instantiate a tractable solver, we approximate the ideal game through two principled stages. 
In the inner minimization, we introduce a scene-adaptive coalition learning mechanism. Rather than relying on external heuristics, we convert the benign world model into an adversary via role-conditioning. Crucially, we introduce a counterfactual credit assignment approach via role switching within the same model, which efficiently resolves the multi-agent reward ambiguity and identifies sparse, synergistic adversarial coalitions. 
In the outer maximization, the ego planner is optimized as a constrained robust best response to the AWM's calibrated risk distribution. We formulate a regret-aware objective that prioritizes informative tail vulnerabilities, coupled with reference-anchored trust-region regularization to preserve nominal driving performance. 
Our main contributions are summarized as follows:
\begin{itemize}[leftmargin=*, itemsep=-2pt, topsep=-3pt]
    \item We present the first self-play adversarial fine-tuning paradigm that casts robust autoregressive motion planning as a constrained min-max game. With a shared rollout, it bridges principled adversarial world modeling and constrained planner adaptation without relying on external simulators.
    \item 
    We design a tractable decoupled solver to instantiate the game. Leveraging role-conditioned counterfactual credit assignment, the AWM first discovers a sparse, scene-adaptive adversarial risk distribution. Utilizing this calibrated risk, the ego planner then optimizes a regret-aware objective to mitigate adversarial tail-risk while preserving nominal driving via trust-region constraints.
    \item 
    We provide provable theoretical justifications and comprehensive empirical validation. Theoretical analysis bounds the local approximation error of the decoupled solver. Evaluations on nuPlan and InterPlan benchmarks demonstrate that AWM discovers transferable multi-agent attacks and yields a robust planner excelling in both naturalistic and highly interactive long-tail scenarios.
\end{itemize}

\vspace{-10pt}
\section{Preliminary and Problem Formulation}
\label{sec:preliminary}

\vspace{-8pt}
\subsection{Autoregressive Motion Generation}
\label{subsec:autoregressive_generation}

\vspace{-8pt}
\paragraph{Tokenized Autoregressive Rollout.}
Recent generative driving models formulate multi-agent traffic evolution as sequential next-token prediction over a learned discrete motion vocabulary~\cite{seff2023motionlm,philion2023trajeglish,wu2024smart,zhang2025closed}. Consider a scene $s$ with ego agent $e$, non-ego agents $\mathcal V_s=\{1,\ldots,N_s\}$, and context $c_s$ (e.g., map, history). 
At autoregressive step $t\in\{1,\ldots,H\}$, each agent $i\in\{e\}\cup\mathcal V_s$ emits a discrete motion token $z_{i,t}\in\mathcal T$, and the joint multi-agent token set is $\vz_t=\{z_{i,t}\}_{i\in\{e\}\cup\mathcal V_s}$. The generated prefix $\vz_{<t}$ updates the scene rollout through a deterministic transition $x_{t+1}=\Gamma(x_t,\vz_t)$, so the partially generated state can be written as $x_t=\Phi(c_s,\vz_{<t})$. 
Thus, the Markovian generative process is:
$p(\vz_{1:H}\mid c_s)=\prod_{t=1}^{H}p(\vz_t\mid x_t)$.
This formulation yields an interactive rollout where each generated token immediately alters the spatiotemporal state observed by all agents in future steps.

\vspace{-10pt}
\paragraph{Planning-Prediction Decomposition.}
Directly modeling the full joint token distribution over the entire horizon is computationally intractable and entangles ego decision-making with environmental dynamics. 
Following decoupled planning-prediction paradigms~\cite{tang2025plan,zhang2025carplanner}, we factorize the step-wise joint probability into individual marginals and instantiate them with two parameterized policies:

\begin{equation}
    \label{eq:factorization}
    \begin{aligned}
    p(\vz_{1:H}\mid c_s) = \prod_{t=1}^{H} \Big[ p(z_{e,t} \mid x_t) \prod_{i \in \mathcal{V}_s} p(z_{i,t} \mid x_t) \Big]
    &\approx
    \prod_{t=1}^{H}
    \Big[
        \underbrace{\pi_{\mathrm{plan}}(z_{e,t}\mid x_t)}_{\text{ego planner}}
        \prod_{i\in\mathcal V_s}
        \underbrace{\pi_{\mathrm{pred}}(z_{i,t}\mid x_t)}_{\text{non-ego world model}}
    \Big].
    \end{aligned}
\end{equation}
$\pi_{\mathrm{pred}}$ is learned from driving logs and represents naturalistic traffic. 
Eq.~\ref{eq:factorization} is the interface used throughout the paper: 
the ego factor represents the object optimized by the planner, while the non-ego factors represent the environment dynamics.
However, optimizing solely against $\pi_{\mathrm{pred}}$ often leaves $\pi_{\mathrm{plan}}$ vulnerable to rare but critical out-of-distribution events, requiring robust adversarial modeling.

\vspace{-10pt}
\subsection{Problem Formulation: Robust Planning as a Constrained Game}
\label{subsec:problem_formulation}

\vspace{-5pt}
We evaluate a rollout by a scalarized utility that trades off progress and safety. Let $r_t$ denote the progress reward and $c_t^{(m)}$ denote the $m$-th safety cost term (e.g., collision, off-road) at step $t$. 
The step-wise utility and scene-level cumulative utility are
$u_t(x_t,\vz_t)=\lambda_R r_t-\sum_{m=1}^{M}\lambda_m c_t^{(m)}$ and 
$U_s(\pi_{\mathrm{plan}},\pi_{\mathrm{pred}})=\mathbb E_{\tau\sim(\pi_{\mathrm{plan}},\pi_{\mathrm{pred}})}[\sum_{t=1}^{H}\gamma^{t-1}u_t(x_t,\vz_t)]$, respectively.
The standard planning objective is to maximize $U_s$ under the benign prior $\pi_{\mathrm{pred}}$. Robust planning requires the planner to perform well when some non-ego agents behave adversarially, while still preserving normal driving ability. We therefore replace the benign environment with a controllable \textbf{Adversarial World Model} (AWM) $\pi_{\mathrm{awm}}$ and formulate the ideal robust objective as a \textit{constrained two-player zero-sum game}:
\begin{equation}
    \label{eq:ideal_game}
    \begin{aligned}
        \max_{\pi_{\mathrm{plan}}} \min_{\pi_{\mathrm{awm}},\,\{\mathcal A_s\}} \quad & \mathbb{E}_{s \sim \mathcal{D}} \left[ U_s(\pi_{\mathrm{plan}}, \pi_{\mathrm{awm}}, \mathcal{A}_s) \right] \\
        \text{s.t.} \quad & |\mathcal{A}_s| \le K_{\max},\mathcal{A}_s \subseteq \mathcal{V}_s, \quad \forall s \in \mathcal{D} \quad \text{(Attack Budget)} \\
        & \mathbb{E}_{x \sim \mathcal{D}} \left[ \mathcal{D}_\text{trust} \left( \pi_{\mathrm{plan}}(\cdot \mid x) \parallel \pi_{\mathrm{plan}}^{\mathrm{nom}}(\cdot \mid x) \right) \right] \le \epsilon \quad \text{(Trust Region Constraint)},
    \end{aligned}
\end{equation}
where the inner player minimizes $U_s$ by optimizing the generative policy $\pi_{\mathrm{awm}}$, and the outer player maximizes its expected utility against the risk distribution induced by $\pi_{\mathrm{awm}}$. We further introduce two constraints: (1) a budgeted subset of active adversaries $\mathcal{A}_s$ to ensure sparse and reasonable attack; (2) a trust region around a nominal reference planner $\pi_{\mathrm{plan}}^{\mathrm{nom}}$ to prevent behavioral degradation.

\vspace{-5pt}
However, Eq.~\ref{eq:ideal_game} is not directly tractable. First, the inner step involves a combinatorial search over the high-dimensional multi-agent action space, suffering from severe credit assignment ambiguities~\cite{busoniu2008comprehensive}.
Second, simultaneous non-convex min-max updates are notoriously unstable \cite{mescheder2017numerics,daskalakis2017training}, as the planner struggles to converge against a non-stationary distribution. 
Third, directly enforcing behavioral constraints under adversarial rollouts can yield an averaged or overly conservative policy with degraded normal performance.
To overcome these issues, we introduce a principled solver in Section~\ref{sec:method} as a decoupled approximation.

\vspace{-6pt}
\section{Self-Play Fine-tuning with Adversarial World Models}
\label{sec:method}

\vspace{-10pt}
\begin{figure}[!htbp]
\centering
\includegraphics[width=0.99\linewidth]{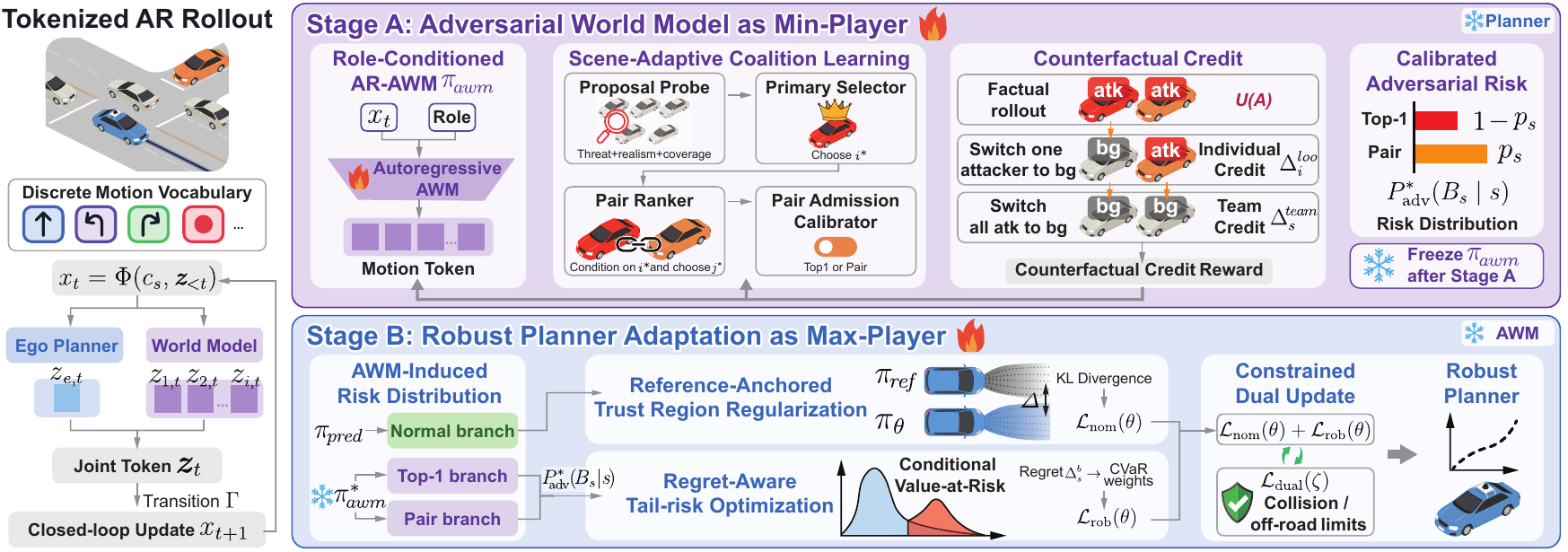}
\vspace{-5pt}
\caption{Decoupled self-play fine-tuning with AWM. The full pipeline is described in Algorithm~\ref{alg:awm_self_play}.}
\label{fig:framework}
\end{figure}

\vspace{-8pt}
We instantiate the idealized game in Eq.~\ref{eq:ideal_game} as a decoupled self-play fine-tuning framework over the shared autoregressive factorization in Eq.~\ref{eq:factorization}. 
We use a single decoupled self-play round, $A_1 \rightarrow B_1$: Stage~A learns the traffic opponent against the reference planner (Section~\ref{sec:awm}), and Stage~B fine-tunes the planner against the resulting frozen opponent (Section~\ref{sec:planner}).
Thus, adversarial training data are generated by interactive joint rollouts within the same token space and transition dynamics, rather than by external heuristic perturbations or a separate simulator. 
Appendix~\ref{app:extended_details} gives the implementation details and Appendix~\ref{app:theory_final} provides the theoretical analysis. 
Figure~\ref{fig:framework} summarizes the two-stage pipeline.

\vspace{-5pt}
\subsection{Decoupled Self-Play Approximation of the Constrained Game}
\label{subsec:staged_approximation}

\vspace{-5pt}
\textbf{Stage A: Adversarial Coalition Learning.}
We first anchor the planner to a frozen reference $\pi_{\mathrm{plan}}^{\mathrm{nom}}$ and approximate the inner minimization over the non-ego factors $\prod \pi_{\mathrm{pred}}(\cdot)$. Instead of enumerating all subsets $\mathcal A_s$, AWM learns a role-conditioned policy and a scene-adaptive coalition distribution:
\begin{equation}
    \label{eq:stage_a_min}
    \pi_{\mathrm{awm}}^*, \mathcal{P}_{\mathrm{adv}}^* = \arg\min_{\pi_{\mathrm{awm}}, \,\{\mathcal A_s\}} \mathbb{E}_{s \sim \mathcal{D}} \left[ U_s(\pi_{\mathrm{plan}}^{\mathrm{nom}}, \pi_{\mathrm{awm}}, \mathcal{A}_s) \right] \quad \text{s.t.} \quad |\mathcal{A}_s| \le K_{\max}.
\end{equation}
This stage yields a frozen adversarial policy $\pi_{\mathrm{awm}}^*$ with a calibrated risk distribution $\mathcal{P}_{\mathrm{adv}}^*$ over sparse multi-agent attacks.
Proposition~\ref{prop:final_frozen_adv_gap} bounds the local error incurred by freezing the adversary, and Corollary~\ref{cor:final_local_value_guarantee} analyzes the local gap in robust value associated with the update anchored to the reference policy.

\vspace{-5pt}
\textbf{Stage B: Constrained Robust Best Response.}
With $(\pi_{\mathrm{awm}}^*, \mathcal{P}_{\mathrm{adv}}^*)$ frozen, the planner solves the outer maximization within a reference trust region.
Rather than optimizing raw adversarial utility, we compare the current planner with a frozen reference planner under the same induced attack. 
This yields the reference-relative regret over $\mathcal{P}_{\mathrm{adv}}^*$. The planner's objective is given by:
\begin{equation}
\begin{aligned}
\pi_{\mathrm{plan}}^{*}
=
\arg\max_{\pi_{\mathrm{plan}}}
\quad
&
-
\mathbb{E}_{s \sim \mathcal{D}}
\left[
\mathfrak R_\tau^{A_s \sim \mathcal{P}_{\mathrm{adv}}^*(\cdot \mid s)}
\left(
\Delta_{s}^{\mathrm{adv}}(\pi_{\mathrm{plan}}, \mathcal{A}_s)
\right)
\right]
\\
\mathrm{s.t.}
\quad
&
\mathcal{D}_{\mathrm{ref}}
\left(
\pi_{\mathrm{plan}}, \pi_{\mathrm{ref}}
\right)
\leq \epsilon, \quad
\mathcal{V}_{\mathrm{safe}}
\left(
\pi_{\mathrm{plan}}; \pi_{\mathrm{awm}}^{*}, \mathcal{P}_{\mathrm{adv}}^*
\right)
\leq \kappa ,
\end{aligned}
\label{eq:stage_b_max}
\end{equation}
where the reference $\pi_{\mathrm{ref}}$ is initialized from $\pi_{\mathrm{plan}}^{\mathrm{nom}}$, and
$
\Delta_{s}^{\mathrm{adv}}(\pi_{\mathrm{plan}}, A_s)
=
U_s(\pi_{\mathrm{ref}}, \pi_{\mathrm{awm}}^{*}, A_s)
-
U_s(\pi_{\mathrm{plan}}, \pi_{\mathrm{awm}}^{*}, A_s)$.
Here, $\mathfrak R_\tau$ denotes a generic tail-risk functional over the regret distribution. The abstract constraint
$\mathcal{D}_{\mathrm{ref}}$ enforces reference-anchored retention, while
$\mathcal{V}_{\mathrm{safe}}$ represents admissible safety-violation constraints.
Importantly, Lemma~\ref{lem:final_ar_sensitivity} justifies this local update by showing finite-horizon utility is smooth when the planner remains within a token-level trust region.

\vspace{-5pt}
\subsection{Adversarial World Modeling via Multi-Agent Coalition Learning}
\label{sec:awm}

\vspace{-5pt}
This section instantiates the inner minimization (Eq.~\ref{eq:stage_a_min}).
Directly optimizing an unstructured set of attackers is computationally prohibitive and suffers from credit assignment issues. 
Moreover, effective attacks are sparse and scene-dependent: a single agent may be sufficient in some scenes, whereas in others the worst failure requires coordinated behavior. 
AWM converts the inner search into a sequence: role conditioning defines \textit{how} an agent attacks, coalition learning decides \textit{who} should attack, counterfactual credit trains the attackers, and calibration decides \textit{whether} a paired attacker should be admitted.
This amortizes the worst coalition while remaining on the same motion manifold.

\vspace{-10pt}
\paragraph{Role-Conditioned Adversarial Generation.}
To keep adversarial rollouts within the same autoregressive structure, AWM changes the behavioral distribution of selected non-ego agents without changing the token vocabulary or transition model. Each non-ego agent receives a latent role assignment $\rho_i\in\{\texttt{bg},\texttt{atk}\}$, where $\rho_i=\texttt{atk}$ iff $i\in\mathcal A_s$. 
At each generation step $t$, for each non-ego agent $i \in \mathcal V_s$ the actions are generated by the shared AWM:
\begin{equation}
    z_{i,t} \sim \pi_{\text{awm}}(\cdot \mid x_t, \rho_i), \quad \forall i \in V_s.
\label{eq:role_generation}
\end{equation}
The ego token $z_{e,t}$ is still sampled from $\pi_{\mathrm{plan}}$.
The role conditioning $\rho_i$ is explicitly injected into both internal agent and decoder representations.
Let $h_i$ be the contextual feature of agent $i$ and $k_{i,t}$ be its decoding state. It modifies output behavioral logits $\ell_{i,t}$ in the generation process via:
\begin{equation}
    \tilde{h}_i = h_i + E_{\mathrm{agent}}(\rho_i), \quad \ell_{i,t} = D(k_{i,t} + E_{\mathrm{dec}}(\rho_i)),
\end{equation}
where $E_{\mathrm{agent}}$ and $E_{\mathrm{dec}}$ are learnable embeddings.  
This makes role switching a behavioral intervention rather than a loss reweighting trick. 
It also enables counterfactual evaluations within the same AWM, which is later used for credit assignment. 
Architectural details are provided in Appendix~\ref{app:architecture}.

\vspace{-5pt}
\paragraph{Sparse Coalition Learning.}
Given the conditioned generator, the next problem is to determine the sparse coalition $A_s$. 
AWM avoids exhaustive subset search through a structured sequence: \textit{proposal probing, primary selection, conditional pairing, and pair admission}. 
A background probe first rolls out all agents with $\rho_i=\texttt{bg}$ and scores each agent by threat, realism, and residual coverage (the exact probe signals and coverage metrics are given in Appendix~\ref{app:awm_probe}). The resulting sparse proposal set $P_s$ removes benign agents while retaining complementary attack potentials. A learned primary selector then chooses $i_s^*=\arg\max_{i\in P_s} f_{\mathrm{pri}}(x_i)$, where $x_i$ combines probe features and proposal evidence. Conditional on $i_s^*$ acting as $\texttt{atk}$, a pair ranker scores relational features $\phi(i_s^*,j)$ and selects $\hat j_s=\arg\max_{j\in P_s\setminus\{i_s^*\}}f_{\mathrm{pair}}(\phi(i_s^*,j))$.
This converts the combinatorial subset search into a tractable sequential approximation. 
Corollary~\ref{cor:final_support_mass_gap} gives an interpretation of the approximation error: the coverage gap decreases when this pipeline captures more of the adversarial support that matters under the full worst-case adversary. Details for $f_{\mathrm{pri}}$ and $f_{\mathrm{pair}}$ are in Appendix~\ref{app:awm_features}.

\bigskip

\vspace{-15pt}
\paragraph{Counterfactual Credit Assignment.} 
The coalition selected above is trained from scene-level utility, but the optimization signal in Eq.~\ref{eq:stage_a_min} needs to be attributed to individual attackers and to their joint behavior. 
Directly isolating and evaluating member-level contribution to reducing the ego utility is challenging, as the credit space scales with the joint action tree.
To achieve fine-grained attribution without enumerating this exploding space, we leverage the role-conditioned nature of the AWM to perform \textit{counterfactual self-play} credit assignment.
Let $\mathcal A_s^{\mathrm{fact}}$ be the factual active coalition in a training rollout, and write $U_s(\mathcal A)$ for the planner utility under $\mathcal A$. 
For each active attacker $i\in\mathcal A_s^{\mathrm{fact}}$, we compute a leave-one-out marginal contribution against counterfactual scenarios:
\begin{equation}
    \Delta_i^{\mathrm{loo}}
    =
    U_s(\mathcal A_s^{\mathrm{fact}}\setminus\{i\})
    -
    U_s(\mathcal A_s^{\mathrm{fact}}).
\end{equation}
A partner might be weak in isolation but critical when synergizing with the primary attacker. 
To preserve non-additive cooperation, we also define the counterfactual team contribution:
\begin{equation}
    \Delta_s^{\mathrm{team}}
    =
    U_s(\operatorname{bg}(\mathcal A_s^{\mathrm{fact}}))
    -
    U_s(\mathcal A_s^{\mathrm{fact}}),
\end{equation}
where $\operatorname{bg}(A_s^{\mathrm{fact}})$ denotes switches the roles of all attackers to $\rho = \texttt{bg}$. Then the attacker's reward is
\begin{equation}\label{eq:hybrid_awm_reward}
    \Delta_i^{\mathrm{hyb}}
    =
    \omega_{\mathrm{loo}}\Delta_i^{\mathrm{loo}}
    +
    \omega_{\mathrm{team}}\Delta_s^{\mathrm{team}}.
\end{equation}
Proposition~\ref{prop:final_hybrid_credit} shows that this hybrid counterfactual reward is gradient-aligned with the adversarial coalition objective, while Lemma~\ref{lem:final_variance_optimal_baseline} interprets the role-switched terms as action-independent baselines for variance reduction. 
The full AWM training objective is given in Appendix~\ref{app:awm_training}.

\vspace{-5pt}
\paragraph{Scene-Adaptive Calibration.}
The final output of Stage-A decides not only who the best pair is, but also whether the pair is worth using in the current scene. To finalize the opponent distribution, we introduce a decisive scene-dependent calibrator.
Given the primary attacker $i_s^*$ and candidate partner $\hat j_s$, we construct a scene-level representation $\psi_s=\mathrm{SceneEnc}(i_s^*,\hat j_s)$ and define the pair gain
$G_s^{\mathrm{pair}}=U_s(\{i_s^*\})-U_s(\{i_s^*,\hat j_s\})$.
A binary calibrator $q_\eta(\psi_s)=\sigma(f_\eta^{\mathrm{cal}}(\psi_s))$ is trained with cross-entropy on $\mathbb I(G_s^{\mathrm{pair}}>0)$ to predict the probability $p_s$ of a positive marginal improvement. 
At inference, the pair is admitted only when $p_s$ exceeds a threshold. 
Proposition~\ref{prop:final_pair_admission_rule} shows that this thresholding rule is the optimal binary admission policy under calibrated pair-gain probabilities. 
Thus, the coalition size is scene-adaptive, preventing attack quality degradation or redundancy.

\vspace{-5pt}
\subsection{Regret-Aware Constrained Robust Planner Optimization}
\label{sec:planner}

\vspace{-3pt}
With the adversary ($\pi_{\mathrm{awm}}^*$ and $q_\eta$) frozen, Stage B instantiates the constrained best-response problem in Eq.~\ref{eq:stage_b_max}.
Exposing the planner to generated attacks via standard adversarial training often leads to performance degradation due to two issues: (1) it forces the planner to adopt overly conservative behaviors while sacrificing nominal driving; and (2) it is risk-insensitive and optimizes an average adversarial loss, diluting the gradients of rare but critical failures. To overcome these issues, we first convert the learned AWM into a conditional risk distribution. Then the planner optimizes adversarial tail-risk using a regret-aware robust objective over the induced distribution. 
To enforce the behavioral regularization in a tractable way, we optimize a reference-anchored dual surrogate: a trajectory-level constraint and a normal-preserving penalty. Safety constraints are handled by dual variables.

\vspace{-8pt}
\paragraph{AWM-Induced Adversarial Risk Distribution.}
For each scene, the frozen AWM defines three rollout branches. 
The normal branch uses the benign world model $\pi_{\mathrm{pred}}$. 
The top-1 branch assigns $\rho_{i_s^*}=\texttt{atk}$ and keeps all other non-ego agents as background. 
The pair branch assigns $\rho_i=\texttt{atk}$ for $i\in\{i_s^*,\hat j_s\}$. Let $\alpha_s^{\mathrm{cal}}=f_\eta^{\mathrm{cal}}(\psi_s)$ be the calibrator logit and $p_s=\sigma(\alpha_s^{\mathrm{cal}}/\tau_{\mathrm{risk}})$ be the temperature-scaled probability of admitting the pair branch. 
AWM induces the discrete branch distribution:
\begin{equation}
    P(B_s=\texttt{top1})=1-p_s,
    \qquad
    P(B_s=\texttt{pair})=p_s.
    \label{eq:adv_branch_dist}
\end{equation}
This converts the coalition output from Stage~A into the calibrated opponent distribution $\mathcal P_{\mathrm{adv}}^*$ used by Stage~B. 
Instead of an unstructured attack buffer, the planner replays each scene and optimizes against a scene-conditional self-play distribution. 
Branch construction details are in Appendix~\ref{app:planner_nominal}.

\paragraph{Reference-Anchored Trust Region Regularization.}
Robust adaptation should not degrade normal driving. We instantiate the reference constraint in Eq.~\ref{eq:stage_b_max} on the normal branch, where the environment remains $\pi_{\mathrm{pred}}$. 
Let $U_{s}^{\mathrm{norm}}(\pi_{\theta})$ denote the normal scene-level utility. 
We use two complementary retention surrogates. First, to preserve nominal driving performance, we define the normal utility margin
$m_{s}^{\mathrm{norm}}(\theta)
=
U_{s}^{\mathrm{norm}}(\pi_{\theta})
-
U_{s}^{\mathrm{norm}}(\pi_{\mathrm{ref}})$ and impose an asymmetric degradation penalty:
$
\psi_{s}^{\mathrm{norm}}(\theta)
=
\left[
-\epsilon_{\mathrm{norm}}
-
m_{s}^{\mathrm{norm}}(\theta)
\right]_{+}^{2}
$. 
Second, to prevent large behavioral drift even when the scalar utility remains acceptable, we also regularize the token distribution along the normal rollout as:
$
\mathcal{D}_{\mathrm{KL}}^{\mathrm{norm}}(s;\theta)
=
\frac{1}{H}
\sum_{t=1}^{H}
D_{\mathrm{KL}}
\left(
\pi_{\theta}(\cdot \mid x_{s,t}^{\mathrm{norm}})
\;\middle\|\;
\pi_{\mathrm{ref}}(\cdot \mid x_{s,t}^{\mathrm{norm}})
\right).
$
The normal-branch objective is:
\begin{equation}
\mathcal L_{\mathrm{nom}}(\theta)
=
\mathbb E_{s\sim\mathcal D}
\left[
\lambda_{\mathrm{pg}}\ell_s^{\mathrm{norm}}(\theta)
+
\lambda_{\mathrm{kl}}\mathcal D_{\mathrm{KL}}^{\mathrm{norm}}(s;\theta)
+
\lambda_{\mathrm{margin}}\psi_s^{\mathrm{norm}}(\theta)
\right].
\label{eq:nominal_trust_region}
\end{equation}
Here $\ell_s^{\mathrm{norm}}$ is the normal-branch policy-gradient loss. 
These terms provide a tractable surrogate for the trust-region constraint. 
Proposition~\ref{prop:final_nominal_floor} proves that they imply a nominal-performance bound, ensuring a conservative update to obtain robustness. Implementation details are in Appendix~\ref{app:planner_nominal}.

\vspace{-10pt}
\paragraph{Adversarial Tail-Risk Optimization.}
Absolute adversarial utility can be dominated by intrinsic scene difficulty rather than the planner's vulnerabilities. 
We therefore propose to optimize reference-relative regret under each adversarial branch to realize the robust objective in Eq.~\ref{eq:stage_b_max}:
\begin{equation}
    \Delta_s^{b}(\theta) = U_s^{b}(\pi_{\mathrm{ref}}) - U_s^{b}(\pi_\theta), \quad \text{for } b \in \{\mathrm{top1}, \mathrm{pair}\}.
\label{eq:adv_regret}
\end{equation}
A positive regret indicates that $\pi_\theta$ underperforms the baseline under the specific attack $b$. For each scene, CVaR operates on the distribution of regret relative to the reference planner, with two support points induced by the top-1 and pair attack modes and probabilities $1-p_s$ and $p_s$:
\begin{equation}
    \mathcal{D}_s^\theta = \left\{ \left(\Delta_s^{\mathrm{top1}}(\theta), 1-p_s\right), \left(\Delta_s^{\mathrm{pair}}(\theta), p_s\right) \right\}.
\label{eq:regret_dist}
\end{equation}
Proposition~\ref{prop:final_regret_control_variate} shows that regret preserves the branchwise gradient direction of utility maximization while removing intrinsic difficulty variation. 
Standard expected loss would dilute the gradients of these rare but critical attacks by their low probabilities.
Instead, we apply a tail-focused risk operator $\mathfrak{R}_\tau$ at tail level $\tau \in (0, 1]$ and instantiate $\mathfrak{R}_\tau$ by analytically solving the Conditional Value-at-Risk (CVaR) allocation \cite{rockafellar2000optimization} over the induced two-point regret distribution $\mathcal D_s^\theta$.
Since the distribution has only two support points, Proposition~\ref{prop:final_cvar_weights} gives closed-form weights $w_s^{\mathrm{top1}}(\tau)$ and $w_s^{\mathrm{pair}}(\tau)$. 
These weights amplify low-probability but high-regret pair attacks when they occupy the adversarial tail. 
Let $\ell_s^{\texttt{top1}}(\theta)$ and $\ell_s^{\texttt{pair}}(\theta)$ be the policy gradient losses, the adversarial tail-risk objective is:
\begin{equation}
    \mathcal L_{\mathrm{rob}}(\theta)
    =
    \mathbb E_{s\sim\mathcal D}
    \left[
    w_s^{\mathrm{top1}}(\tau)\ell_s^{\mathrm{top1}}(\theta)
    +
    w_s^{\mathrm{pair}}(\tau)\ell_s^{\mathrm{pair}}(\theta)
    \right].
\label{eq:rob_loss}
\end{equation}
Implementations are detailed in Appendix~\ref{app:cvar_weights}.
Regret-CVaR allocates tail weights to the top-1 and pair modes according to their regrets and admission probabilities.

\vspace{-10pt}
\paragraph{Safety-constrained dual update.}
\label{subsubsec:dual_update}
Finally, to prevent the planner from exploiting the objective by violating traffic rules (e.g., evading an attacker by driving off-road), we enforce explicit safety constraints in Eq.~\ref{eq:stage_b_max}.
Let $v_{s,m}^b\in[0,1]$ be the $m$-th safety violation indicator under branch $b$. The expected violation rate is $\bar{v}_{s,m} = (1-p_s) v_{s,m}^{\texttt{top1}} + p_s v_{s,m}^{\texttt{pair}}$.
Given admissible thresholds $\kappa_m$, we update Lagrangian dual variables $\zeta$ to dynamically adjust the cost multipliers $\lambda_m(\zeta)$:
$
\mathcal L_{\mathrm{dual}}(\zeta)
    =
    \sum_{m=1}^{M}
    \lambda_m(\zeta)
    \left(
        \mathbb E_{s\sim\mathcal D}[\bar v_{s,m}]
        -
        \kappa_m
    \right).
$
Training alternates between minimizing
$\mathcal L_{\mathrm{planner}}(\theta)=\mathcal L_{\mathrm{nom}}(\theta)+\mathcal L_{\mathrm{rob}}(\theta)$
over planner parameters and maximizing $\mathcal L_{\mathrm{dual}}$ over $\zeta$. 
Proposition~\ref{prop:final_softmax_dual} shows that the softmax-parameterized dual behaves as a smooth maximum-violation penalty, concentrating weight on the most violated safety channel while keeping the reward scale bounded. 
The multiplier parameterization is provided in Appendix~\ref{app:dual_update}.

\vspace{-8pt}
\section{Experiments}
\label{sec:experiment}

\vspace{-10pt}
We organize the empirical study around three questions. First, does adversarial self-play improve the final closed-loop planner on standard and interaction-heavy benchmarks? Second, does the learned AWM expose transferable planner vulnerabilities when it is used as a simulator-side traffic model? Third, which components of the two-stage optimization account for the observed gains? Experimental protocols, metric definitions, and implementation details are provided in Appendix~\ref{app:exp_setups}.

\vspace{-10pt}
\subsection{Closed-Loop Planning and Adversarial Testing}

\vspace{-5pt}
\paragraph{Closed-Loop Planning Performance.}
We first evaluate whether the robust best-response update improves ordinary closed-loop planning performance. Table~\ref{tab:main_benchmark} reports benchmark results on \texttt{nuPlan}~\cite{nuplan}, \texttt{InterPlan}~\cite{hallgarten2024can}, and \texttt{InterPlan-LongTail} (Appendix~\ref{app:interplan_longtail}). Relative to the Plan-R1 reference, AWM-Planner (ours) improves five of the six \texttt{nuPlan} settings, with the largest gains on Test14-hard NR and R (\(+2.60\) and \(+1.28\)). The Test14-random NR score is essentially unchanged (\(-0.01\)), while Test14-random R increases by \(+0.67\), indicating that the hard-case improvement is not obtained by broadly sacrificing normal driving behavior.
The interaction benchmarks provide a more targeted assessment of the intended mechanism. AWM-Planner improves Plan-R1 by \(+2.10\) on full \texttt{InterPlan}, \(+4.97\) on InterPlan10, and \(+5.81\) on \texttt{InterPlan-LongTail}. The per-template breakdown in Table~\ref{tab:interplan_longtail_templates} shows improvements on seven of eight designed templates, with the strongest gains in close straight-driving and medium-density yielding cases. These results suggest that AWM primarily improves recoverable interaction failures exposed by structured adversarial traffic.

\begin{table*}[!t]
\caption{Closed-loop planning results on \texttt{nuPlan} and \texttt{InterPlan} benchmarks. Higher is better. NR/R: non-reactive/reactive mode. Among learning-based planners, the best result is in \textbf{bold} and the second best result is \underline{underlined}.
\protect\textcolor{gray!20}{\rule{1.2em}{0.9em}} is the reference baseline (Plan-R1), \protect\textcolor{bestblue}{\rule{1.2em}{0.9em}} indicates an improvement over the reference, while \protect\textcolor{nogainorange}{\rule{1.2em}{0.9em}} indicates no improvement or degradation.}
\label{tab:main_benchmark}
\centering
\setlength{\tabcolsep}{1.5pt}
\renewcommand{\arraystretch}{1.1}
\begin{small}
\resizebox{1.0\linewidth}{!}{%
\begin{tabular}{@{}llccccccccc@{}}
\toprule
Type & Planner & \multicolumn{2}{c}{Val14} & \multicolumn{2}{c}{Test14-hard} & \multicolumn{2}{c}{Test14-random} & \multicolumn{3}{c}{InterPlan} \\
\cmidrule(lr){3-4} \cmidrule(lr){5-6} \cmidrule(lr){7-8} \cmidrule(lr){9-11}
 & & NR $\uparrow$ & R $\uparrow$ & NR $\uparrow$ & R $\uparrow$ & NR $\uparrow$ & R $\uparrow$ & Full $\uparrow$ & InterPlan10 $\uparrow$ & InterPlanLT $\uparrow$ \\
\midrule
\textcolor{gray}{Expert} & \textcolor{gray}{Log-Replay} & \textcolor{gray}{93.53} & \textcolor{gray}{80.32} & \textcolor{gray}{85.96} & \textcolor{gray}{68.80} & \textcolor{gray}{94.03} & \textcolor{gray}{75.86} & \textcolor{gray}{14.76} & \textcolor{gray}{--} & \textcolor{gray}{--} \\
\midrule
\multirow{3}{*}{\shortstack[l]{Rule}} & IDM / IDM+Mobil & 75.60 & 77.33 & 56.15 & 62.26 & 70.39 & 72.42 & 47.07 & 31.28 & 38.49 \\
 & PDM-Closed*~\cite{PDM} & 92.84 & 92.12 & 65.08 & 75.19 & 90.05 & 91.64 & 69.64 & 41.81 & 45.93 \\
 & PDM-Hybrid*~\cite{PDM} & 92.77 & 92.11 & 65.99 & 76.07 & 90.10 & 91.28 & 41.61 & 26.12 & 27.79 \\
\midrule
\multirow{10}{*}{Learning} & UrbanDriver~\cite{urbandriver} & 68.57 & 64.11 & 50.40 & 49.95 & 51.83 & 67.15 & 5.56 & 3.62 & 7.75 \\
 & PDM-Open~\cite{PDM} & 53.53 & 54.24 & 33.51 & 35.83 & 52.81 & 57.23 & 26.22 & 24.67 & 17.51 \\
 & GameFormer~\cite{Gameformer} & 13.32 & 8.69 & 7.08 & 6.69 & 11.36 & 9.31 & 11.07 & 9.12 & 8.13 \\
 & PlanTF~\cite{PlanTF} & 84.27 & 76.95 & 69.70 & 61.61 & 85.62 & 79.58 & 47.72 & 33.00 & 32.38 \\
 & PLUTO~\cite{Pluto} & 88.89 & 78.11 & 70.03 & 59.74 & 89.90 & 78.62 & 57.74 & \underline{42.87} & \underline{38.46} \\
 & Diff. Planner$_{\text{DPM}}$~\cite{zheng2025diffusion} & 89.87 & 82.80 & 75.99 & 69.22 & 89.19 & 82.93 & 50.07 & 24.75 & 29.96 \\
 & Diff. Planner$_{\text{DDIM}}$~\cite{zheng2025diffusion} & 89.81 & 82.94 & 76.01 & 68.18 & 89.14 & 82.63 & 49.86 & 24.11 & 28.98 \\
 & Flow Planner~\cite{FlowPlanner} & \textbf{90.43} & 83.31 & 76.47 & 70.42 & 89.88 & 82.93 & \textbf{61.82} & 35.42 & 31.06 \\
  & PlannerRFT~\cite{li2026plannerrft} & \underline{89.96} & 84.46 & 77.16 & 72.21 & 90.76 & 85.80 & -- & -- & -- \\
  & Plan-R1~\cite{tang2025plan}
 & \refscore{88.98}
 & \refscore{\underline{87.69}}
 & \refscore{\underline{77.45}}
 & \refscore{\underline{77.20}}
 & \refscore{\textbf{91.23}}
 & \refscore{\underline{90.04}}
 & \refscore{56.64}
 & \refscore{40.95}
 & \refscore{34.89} \\
\midrule
Ours & AWM-Planner
 & \gain{89.72}{0.74}
 & \gain{\textbf{88.44}}{0.75}
 & \gain{\textbf{80.05}}{2.60}
 & \gain{\textbf{78.48}}{1.28}
 & \nogain{\underline{91.22}}{-0.01}
 & \gain{\textbf{90.71}}{0.67}
 & \gain{\underline{58.74}}{2.10}
 & \gain{\textbf{45.92}}{4.97}
 & \gain{\textbf{40.70}}{5.81} \\
\bottomrule
\end{tabular}
}
\end{small}
\vspace{0pt}
\end{table*}

\vspace{-10pt}
\paragraph{AWM as a Transferable Adversarial Simulator.}
We next test whether the learned adversary captures planner-agnostic interaction stress. Table~\ref{tab:cross_planner_simagent} deploys AWM as the non-ego sim-agent traffic model for seven planners. AWM reduces the sim-agent closed-loop score for every planner, decreasing the mean score from \(70.24\) to \(67.46\). Although Plan-R1 has the largest drop, rule-based, optimization-based, diffusion-based, and transformer-based planners are all affected, indicating that the adversarial behavior generalizes beyond the planner used in Stage A.
The same evaluation also separates adversarial strength from planner robustness after fine-tuning. AWM-Planner starts from a higher score under the normal host than Plan-R1 (\(77.67\) vs. \(75.86\)) and suffers a smaller AWM-induced degradation (\(-3.14\) vs. \(-4.15\)). The submetric pattern is consistent with controlled interaction stress: the largest changes occur in progress and TTC, whereas drivable-area, speed-limit, and driving-direction scores remain nearly unchanged. 
The simulator diagnostics in Table~\ref{tab:app_simagent_diagnostics} further show that AWM utilizes nearly the full two-agent budget while remaining constrained by feasibility gates, indicating that the score drop is not a result of an unconstrained simulator collapse.

\begin{table*}[!htbp]
\vspace{0pt}
\caption{Cross-planner Test14-hard sim-agent evaluation. \textsc{Normal} is the learned non-adversarial sim-agent baseline, while a held-out \textsc{AWM} uses the adversarial host as the simulator-side non-ego policy. Thus, the R-Score is an AWM/Normal sim-agent closed-loop score, not an ordinary planner closed-loop score. \(\Delta_R\) is computed as \(R_{\textsc{AWM}} - R_{\textsc{Normal}}\), where negative values indicate the R-Score drop induced by the adversarial host. Prog., Coll., TTC, Driv., Speed, and Dir. denote progress, collision, time-to-collision, drivable-area, speed-limit, and driving-direction submetrics. \protect\colorbox{gray!15}{\phantom{xx}} mark the corresponding \textsc{Normal} reference R-Score, and \protect\colorbox{bestblue}{\phantom{xx}} mark the AWM-induced R-Score change.}
\label{tab:cross_planner_simagent}
\centering
\setlength{\tabcolsep}{3pt}
\renewcommand{\arraystretch}{1.12}
\begin{small}
\resizebox{1.0\linewidth}{!}{%
\begin{tabular}{@{}ll|cc|ccccccc@{}}
\toprule
Planner & World & R-Score $\uparrow$ & $\Delta_R$ & Prog. $\uparrow$ & Coll. $\uparrow$ & TTC $\uparrow$ & Driv. $\uparrow$ & Comfort $\uparrow$ & Speed $\uparrow$ & Dir. $\uparrow$ \\
\midrule
\multirow{2}{*}{Diffusion Planner~\cite{zheng2025diffusion}}
& Normal & \normalref{$70.81$} & -- & $88.06$ & $84.71$ & $73.33$ & $95.41$ & $86.20$ & $96.50$ & $98.96$ \\
& AWM & $68.10$ & \deltacell{$-2.71$} & $84.25$ & $83.64$ & $69.49$ & $94.85$ & $84.56$ & $96.48$ & $98.90$ \\
\midrule
\multirow{2}{*}{IDM+Mobil}
& Normal & \normalref{$55.64$} & -- & $67.55$ & $80.51$ & $65.44$ & $86.40$ & $88.97$ & $96.93$ & $98.53$ \\
& AWM & $53.21$ & \deltacell{$-2.43$} & $62.96$ & $80.15$ & $62.13$ & $86.03$ & $88.24$ & $96.97$ & $98.53$ \\
\midrule
\multirow{2}{*}{PDM-Closed~\cite{PDM}}
& Normal & \normalref{$65.45$} & -- & $73.97$ & $87.68$ & $70.96$ & $94.85$ & $83.09$ & $99.54$ & $99.08$ \\
& AWM & $64.20$ & \deltacell{$-1.25$} & $70.37$ & $88.97$ & $69.49$ & $94.85$ & $83.46$ & $99.54$ & $99.08$ \\
\midrule
\multirow{2}{*}{PLUTO~\cite{Pluto}}
& Normal & \normalref{$76.63$} & -- & $81.64$ & $92.50$ & $85.39$ & $96.57$ & $89.02$ & $97.04$ & $96.44$ \\
& AWM & $73.60$ & \deltacell{$-3.03$} & $75.28$ & $92.10$ & $81.99$ & $96.69$ & $87.50$ & $96.99$ & $96.69$ \\
\midrule
\multirow{2}{*}{PlanTF~\cite{PlanTF}}
& Normal & \normalref{$69.62$} & -- & $83.52$ & $86.21$ & $78.68$ & $94.49$ & $89.34$ & $97.03$ & $96.69$ \\
& AWM & $66.89$ & \deltacell{$-2.73$} & $79.20$ & $86.95$ & $74.63$ & $94.12$ & $83.09$ & $97.06$ & $96.69$ \\
\midrule
\multirow{2}{*}{Plan-R1~\cite{tang2025plan}}
& Normal & \normalref{$75.86$} & -- & $87.93$ & $89.15$ & $79.41$ & $95.22$ & $99.26$ & $99.07$ & $96.69$ \\
& AWM & $71.71$ & \deltacell{$-4.15$} & $82.83$ & $88.24$ & $73.16$ & $94.85$ & $98.16$ & $99.08$ & $97.06$ \\
\midrule
\multirow{2}{*}{\textbf{AWM-Planner (Ours)}}
& Normal & \normalref{\textbf{77.67}} & -- & $83.60$ & $93.38$ & $80.51$ & $96.69$ & $95.22$ & $99.47$ & $98.90$ \\
& AWM & \textbf{74.53} & \deltacell{$-3.14$} & $79.07$ & $91.54$ & $77.94$ & $96.69$ & $95.96$ & $99.47$ & $98.90$ \\
\bottomrule
\end{tabular}
}
\end{small}
\vspace{0pt}
\end{table*}

\vspace{-10pt}
\subsection{Algorithmic Analysis and Discussion}
\vspace{-5pt}
\paragraph{Planner-World Cross-Validation.} Closed-loop scores couple the planner policy with the simulator dynamics. To isolate the effect of planner adaptation, we evaluate pre- and post-training planners under matched normal and adversarial world branches. Figure~\ref{fig:main_quantitative_diagnostics}(a) summarizes this matrix, and Table~\ref{tab:planner_prepost_world_matrix} provides the full values. The post-trained planner improves nominal reward and progress under the normal world, while the larger gains occur under the AWM world: adversarial reward increases by \(0.0119\), and Tail-CVaR increases from \(0.8660\) to \(0.9095\). This asymmetry matches the objective of Stage B, which should lift the adversarial lower tail while preserving nominal rollout quality.
Table~\ref{tab:planner_offline_branch_matrix} further decomposes the AWM host into Top-1, Pair, and Adaptive branches. Pair and Adaptive branches produce larger lower-tail stress than the single-agent Top-1 branch, and the adaptive branch remains less aggressive than forced Pair because the calibrator admits two-agent coalitions only when they are predicted to be useful. 
Thus, the planner gains in Table~\ref{tab:main_benchmark} indicate a targeted policy adaptation that specifically improves the response to structured adversarial interactions.

\begin{figure}[!htbp]
\centering
\begin{minipage}[t]{0.48\linewidth}
\centering
\includegraphics[width=\linewidth]{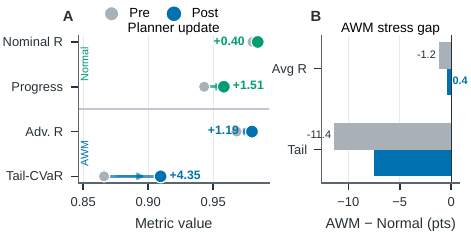}
\vspace{5pt}
{\small\textbf{(a)} Planner-world cross-validation}
\end{minipage}
\hfill
\begin{minipage}[t]{0.48\linewidth}
\centering
\includegraphics[width=\linewidth]{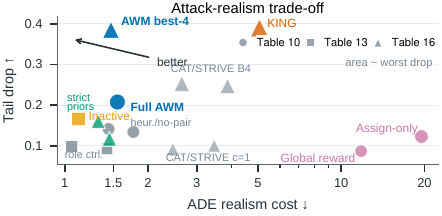}
\vspace{5pt}
{\small\textbf{(b)} Attack and realism trade-off}
\end{minipage}
\vspace{-5pt}
\caption{(a) Cross-validation shows that robust fine-tuning improves the adversarial tail while preserving nominal rollout quality (full values in Table~\ref{tab:planner_prepost_world_matrix}). (b) The attack and realism plot summarizes the AWM ablations in Tables~\ref{tab:awm_credit_assignment}, \ref{tab:awm_coalition_structure}, \ref{tab:awm_diagnostics}, \ref{tab:awm_role_conditioning} and the comparison in Table~\ref{tab:prior_generator_pilot}. AWM achieves a balanced operating point, whereas aggressive search increases attack strength at a clear realism cost.}
\label{fig:main_quantitative_diagnostics}
\vspace{0pt}
\end{figure}

\paragraph{Mechanistic Analysis of AWM.}
The AWM is useful for planner training only if its adversarial rollouts are both informative and plausible. The AWM-side ablations in Tables~\ref{tab:awm_credit_assignment}--\ref{tab:awm_role_conditioning} support this requirement. Role-conditioned counterfactual credit produces substantially stronger lower-tail degradation than global reward broadcast or assignment-only variants, while avoiding the large ADE/FDE collapse observed in coarse credit baselines. This indicates that role switching provides a more reliable training signal for coordinated adversarial behavior than assigning the same scene-level reward to all agents.
The coalition ablations show a complementary effect. Forced \(K{=}2\) attacks can increase raw stress, but they move the generator toward a less realistic mode. Instead, the adaptive calibrated AWM preserves most of the controlled attack effect while keeping pair usage sparse and well calibrated. 
Figure~\ref{fig:main_quantitative_diagnostics}(b) gives a compact view of this attack and realism trade-off (full results in Table~\ref{tab:prior_generator_pilot}). 
Matched best-4 search and the KING-style~\cite{hanselmann2022king} aggressive bound can produce larger raw disruption, but they also incur larger realism costs and background-collision proxies. In contrast, single-candidate CAT$^*$~\cite{zhang2023cat} and STRIVE$^*$~\cite{rempe2022generating} remain closer to ordinary AWM generation but produce weaker tail stress. AWM therefore provides the intended middle regime: sufficient adversarial stress for robust training without becoming a purely destructive test-time attacker.

\begin{table*}[!htbp]
\vspace{0pt}
\caption{Training paradigm comparison. Open-loop rewards are on the planner's rollout utility scale; CL scores are ordinary nuPlan closed-loop scores. Relative GPU-hours are normalized to the decoupled two-stage pipeline, and plateau updates denote thousands of planner update steps until validation Tail-CVaR stops improving. Stable seeds count runs that finish without divergence or failure, and \textsc{Hard std.} is the seed-wise standard deviation of Test14-hard NR score.}
\label{tab:training_paradigm}
\centering
\setlength{\tabcolsep}{3pt}
\renewcommand{\arraystretch}{1.2}
\begin{small}
\resizebox{1\linewidth}{!}{%
\begin{tabular}{@{}lccccccccc@{}}
\toprule
\multirow{2}{*}{Training paradigm} & \multicolumn{3}{c}{Open-loop rollout} & \multicolumn{2}{c}{Planner CL} & \multicolumn{4}{c}{Computation and stability} \\
\cmidrule(lr){2-4}\cmidrule(lr){5-6}\cmidrule(lr){7-10}
 & Nom. R $\uparrow$ & Adv. R $\uparrow$ & Tail $\uparrow$ & Hard NR $\uparrow$ & Random R $\uparrow$ & \shortstack{Rel. GPU-h\\$\downarrow$} & \shortstack{Plateau k\\$\downarrow$} & \shortstack{Stable seeds\\$\uparrow$} & \shortstack{Hard std.\\$\downarrow$} \\
\midrule
Planner-only refinement & $0.9804$ & $0.9685$ & $0.8467$ & 77.45 & 90.04 & $0.62$ & $5.0$ & $3/3$ & $0.55$ \\
Simultaneous alternating & $0.9808$ & $0.9726$ & $0.8663$ & $78.30$ & $89.30$ & $2.35$ & $18.5$ & $2/3$ & $1.65$ \\
\shortstack[l]{Alternating warmup\\+ fixed AWM} & $0.9812$ & $0.9744$ & $0.8810$ & $79.30$ & $90.20$ & $1.45$ & $12.0$ & $3/3$ & $0.85$ \\
\rowcolor{bestblue}\textbf{Decoupled Stage-A/B (ours)} & $\mathbf{0.9818}$ & $\mathbf{0.9756}$ & $\mathbf{0.8919}$ & \textbf{80.05} & \textbf{90.71} & $1.00$ & $9.0$ & $3/3$ & $\mathbf{0.40}$ \\
\bottomrule
\end{tabular}
}
\end{small}
\end{table*}

\vspace{-10pt}
\paragraph{Planner Adaptation Ablations.}
We finally examine whether the planner's gain depends on the proposed decoupled optimization and regret-aware objective. Table~\ref{tab:training_paradigm} compares planner-only refinement, simultaneous alternating self-play, alternating warmup followed by fixed refinement, and the proposed decoupled pipeline. Our framework achieves the best adversarial Tail-CVaR, the highest closed-loop scores, and the lowest variability. Simultaneous alternating updates improve over planner-only refinement, but they require substantially more computation and exhibit weaker stability. This supports the design of first learning an opponent and then optimizing a constrained robust best response against the frozen distribution.
Table~\ref{tab:planner_objective_mini} isolates the objective of Stage-B. Expected-risk and no-CVaR variants preserve competitive averages, but they lose adversarial Tail-CVaR and Test14-hard NR score, which is consistent with the need to emphasize rare severe risks. Removing normal retention strengthens adversarial reward but degrades nominal progress and score. Additional opponent sensitivity and hyperparameter results in Tables~\ref{tab:planner_internal_host_switch} and~\ref{tab:planner_hparam_mini}, with Figure~\ref{fig:hparam_sensitivity}, show that the planner is stable under moderate host changes while following the expected robustness and nominal performance trade-off.

\vspace{-10pt}
\begin{table*}[!htbp]
\vspace{0pt}
\caption{Planner objective ablation with a robustness and nominal performance trade-off. Offline rewards are the planner's rollout utility; CL scores and metrics are ordinary nuPlan Test14 simulation.}
\label{tab:planner_objective_mini}
\centering
\setlength{\tabcolsep}{3pt}
\renewcommand{\arraystretch}{1.18}
\begin{small}
\resizebox{1.0\linewidth}{!}{%
\begin{tabular}{@{}lcccc|ccc|ccc@{}}
\toprule
\multirow{2}{*}{Variant} & \multicolumn{4}{c|}{Open-loop rollout} & \multicolumn{3}{c|}{CL: Test14-hard NR} & \multicolumn{3}{c}{CL: Test14-random R} \\
 & Nom. R & Nom. prog. & Adv. R & Tail & Score & Prog. & TTC & Score & Coll. & TTC \\
\midrule
\rowcolor{bestblue}Full & \textbf{0.9818} & $0.9545$ & $0.9756$ & $0.8919$ & 80.05 & 90.16 & \textbf{81.99} & \textbf{90.71} & \textbf{98.08} & \textbf{95.40} \\
Expected only & $0.9815$ & \textbf{0.9579} & $0.9681$ & $0.8640$ & $76.80$ & $86.20$ & $78.36$ & $90.60$ & $97.88$ & $95.20$ \\
w/o CVaR & $0.9816$ & $0.9530$ & $0.9670$ & $0.8501$ & $77.40$ & $87.00$ & $78.50$ & $90.55$ & $97.90$ & $95.15$ \\
w/o normal retention & $0.9790$ & $0.9390$ & \textbf{0.9762} & \textbf{0.8950} & \textbf{80.20} & \textbf{91.00} & $81.20$ & $88.60$ & $96.40$ & $92.05$ \\
\bottomrule
\end{tabular}
}
\end{small}
\end{table*}
\vspace{-5pt}

\vspace{-7pt}
\subsection{Case Studies}

\begin{figure}[!htbp]
\centering
\includegraphics[width=0.90\linewidth,trim=0 0 369.5 0,clip]{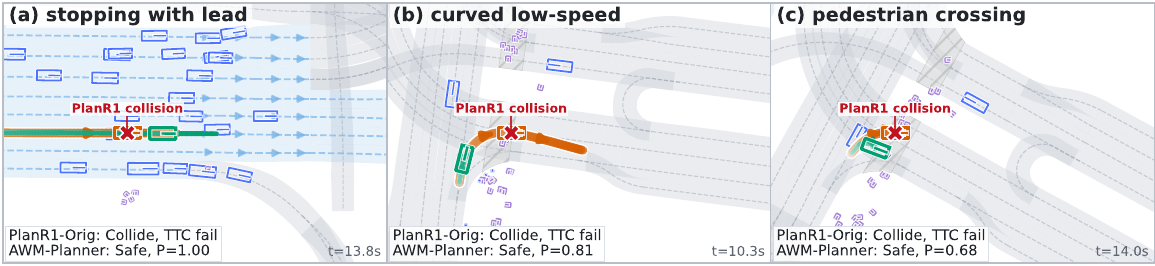}\\[3pt]
\includegraphics[width=0.6001\linewidth,trim=554.18 0 0 0,clip]{figs/main_qual_1x5.pdf}
\vspace{1pt}
\caption{Test14-hard qualitative cases. Orange denotes the original Plan-R1 planner and green denotes the AWM-trained planner. Each panel overlays the full closed-loop trajectory in the same adversarial scene; red markers identify the original planner's collision or TTC failure point.}
\label{fig:main_qual_recovery}
\end{figure}
\vspace{-10pt}
\begin{figure}[!htbp]

\centering
\includegraphics[width=0.65\linewidth]{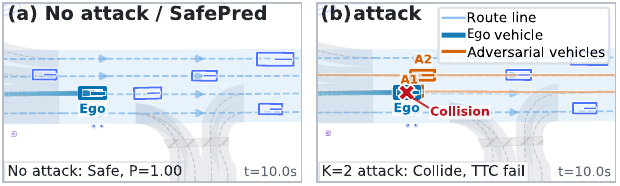}
\vspace{1pt}
\caption{K=2 cooperative attack in sim-agent evaluation. Ego is shown in blue, while A1/A2 are the two adversarial agents selected by AWM.}
\label{fig:main_k2_attack}

\end{figure}
The examples in Figures~\ref{fig:main_qual_recovery} and~\ref{fig:main_k2_attack} illustrate the two mechanisms quantified above. In the recovery cases, AWM-Planner reacts earlier than the original Plan-R1 and avoids collision or TTC failures in the same adversarial scenes. Additional recovery cases, \texttt{InterPlan-LongTail} examples, and \(K{=}2\) cooperative attacks are shown in Figures~\ref{fig:app_recovery_group_a}-\ref{fig:app_k2_group}.
The simulator-side example in Figure~\ref{fig:main_k2_attack} shows how the learned adversary creates stress without destroying the scene. Under the normal host, the ego can proceed through the local interaction. Under the \(K{=}2\) attack, two selected non-ego agents jointly constrain the ego corridor and induce a collision or TTC failure. This behavior matches the results in Table~\ref{tab:cross_planner_simagent}: AWM mainly reduces progress and TTC, while rule-compliance metrics remain nearly unchanged.

\Needspace{16\baselineskip}
\section{Conclusion}
\vspace{-10pt}
This paper introduced AWM, a multi-agent self-play fine-tuning framework leveraging predictive world models as structured adversaries for robust motion planning. To overcome the intractability of high-dimensional adversarial training, we formulated the problem as a constrained min-max game and proposed a theoretically grounded decoupled solver. By integrating scene-adaptive coalition learning with counterfactual credit assignment, AWM efficiently exposes sparse, synergistic interaction failures within the planner's inherent autoregressive interface. Subsequently, a regret-aware robust best response, with reference-anchored trust region constraints, fine-tunes the planner against the induced tail risk. Theoretical analysis and empirical evaluations demonstrate that our framework significantly improves planner robustness in rare, complex interactions while maintaining nominal performance. 
Cross-planner sim-agent evaluation and ablations further support the transferability of the learned adversarial traffic and the necessity of fine-grained credit assignment, learnable adversarial coalition, and tail-focused learning. 
However, AWM remains limited by the fidelity of the tokenized world model and lacks exhaustive scenario coverage or formal closed-loop safety guarantees.

\paragraph{Broader Impact.} AWM supports robustness evaluation and training for autonomous driving in simulation. Adversarial scenarios may be misused or create false confidence beyond the evaluated settings. We constrain generation through feasibility and realism checks and distinguish attack strength from planner performance.

\section*{Acknowledgments and Disclosures}
This work was supported by the National Natural Science Foundation of China (52125208, 524B2164, and 52232015).

The authors declare no competing interests.

\clearpage

\bibliographystyle{plainnat}
\bibliography{references_cr,cr_references}

\newpage

\appendix
\section*{Appendix}
\startcontents[appendix]
\begingroup
\setcounter{tocdepth}{3}

\small                 
\linespread{0.88}\selectfont  
\setlength{\parskip}{3pt}     

\printcontents[appendix]{}{1}{\section*{Appendix Contents}\vskip-10pt\hrulefill\vskip-15pt}
\endgroup

\newpage
\section{Related Work}

\paragraph{Autoregressive Motion Generation.}
Autoregressive modeling has recently emerged as a dominant paradigm for both multi-agent scenario simulation and ego-motion planning in autonomous driving.
Early sequence modeling methods primarily operate in continuous spaces and leverage Transformer-based architectures \cite{feng2025artemis,sun2023large} or diffusion models \cite{zhang2025epona,zheng2025diffusion} to refine continuous waypoints iteratively. However, inspired by the success of Large Language Models (LLMs), discrete next-token prediction modeling has gained significant traction due to its scalability and temporal causality.
Agent simulation models such as Trajeglish \cite{philion2023trajeglish}, SMART \cite{wu2024smart}, and BehaviorGPT \cite{zhou2024behaviorgpt} tokenize maps and trajectories to progressively decode joint multi-agent futures, achieving realistic and interactive rollouts.
Building upon this motion-as-language concept \cite{seff2023motionlm}, recent efforts have extended this discrete representation to closed-loop ego planning \cite{hu2024solving,ye2025dap,tang2025plan}, casting ego decision-making as motion token generation.
Most relevant to our work are architectures that explicitly decouple predictive scene evolution (the world model) from ego planning (the policy), such as the planning-prediction dual-branch design in \cite{tang2025plan,ye2025dap} and the policy-transition decomposition in \cite{zhang2025carplanner}. This predictive branch allows the planner to efficiently interact with reactive non-ego responses during rollouts without relying on a simulator.
However, prior methods treat the world model as a frozen or passive behavioral prior that merely replicates logged data, lacking a structured mechanism to generate customized non-ego behaviors within the same tokenized framework.

\paragraph{Reinforcement Learning Fine-Tuning (RLFT).}
While large-scale behavior cloning establishes a strong traffic prior, it invariably suffers from covariate shift in closed-loop evaluation and fails to enforce safety-critical driving principles \cite{karkus2025beyond}. To bridge this gap, RLFT is increasingly adopted for aligning driving policies with downstream objectives.
In the continuous domain, frameworks such as \cite{huang2025gen,li2026plannerrft} employ learned reward models and policy-guided diffusion to fine-tune generative planners for sample efficiency and alignment. Within discrete architectures, Plan-R1~\cite{tang2025plan} applies GRPO-style optimization with rule-based rewards while maintaining a frozen reactive world model for surrounding agents. Similarly, \citet{ye2025dap} incorporates offline SAC-BC to mitigate imitation ambiguity, and \citet{zhang2025carplanner} suggests that autoregressive structures naturally facilitate large-scale RLFT.
Furthermore, RLFT has also been applied to simulation agents for better collision realism and behavior fidelity \cite{peng2024improving,cao2024reinforcement,chen2025rift}. Despite these advances, existing RLFT pipelines primarily optimize the ego policy against the expected distribution of a naturalistic traffic prior.
Due to the rarity of long-tail events, the resulting planners can learn average behaviors and remain vulnerable to rare but safety-critical risks.

\paragraph{Adversarial Training and Testing.}
Evaluating and ensuring the robustness of autonomous driving needs moving beyond naturalistic traffic scenarios to adversarial edge cases \cite{liu2024curse}.
Early adversarial generation methods \cite{wang2021advsim,rempe2022generating,chang2024safe,mei2025llm,xu2025diffscene} often rely on heuristic perturbations, optimization-based attacks on continuous trajectories, or training specific adversarial agents to induce collisions for stress testing.
To further improve policy robustness, closed-loop adversarial training frameworks \cite{zhang2023cat,nie2025steerable,liu2025adv,stoler2025seal,nie2026adv} alternate between safety-critical scenario generation and ego policy updates.
Recent literature emphasizes that effective adversaries must be structured and reactive rather than purely collision-seeking perturbations to ensure the generated scenarios remain realistic and useful for training. Several techniques such as skill learning \cite{stoler2025seal}, preference alignment \cite{nie2025steerable}, reverse-time token reconstruction \cite{liu2025adv} are shown to be effective for preserving realism or controllability.
Despite these advances, existing methods are largely treated as external environment augmentations, heuristic resampling against a static planner, computationally heavy post-hoc optimization or test-time biasing.
More importantly, existing closed-loop adversarial training frameworks usually assume model-free RL and thus rely on CPU-intensive policy rollouts on physical simulators, significantly limiting their scalability and efficiency \cite{nie2026adv}. Therefore, they often operate in a small set of scenarios (e.g., 500), making them less effective for large-scale motion planners.
\citet{agrawal2025closing} train world models with adversarial perturbations of latent states and actions to improve gradient-based planning. Our AWM learns traffic policies conditioned on agent roles, assigns counterfactual credit to sparse coalitions, and uses the resulting interactions to fine-tune the ego planner.

\newpage
\section{Extended Architectural and Algorithmic Details}
\label{app:extended_details}

This appendix provides technical details to supplement the method presented in Section~\ref{sec:method}. Algorithm~\ref{alg:awm_self_play} first summarizes the full decoupled self-play procedure. Section~\ref{app:architecture} details the shared discrete autoregressive architecture and the role-injection mechanism. Section~\ref{app:awm_details} elaborates on the mathematical formulations of the scene-adaptive coalition learning pipeline, including probe statistics, feature representations, and the full training objective. Section~\ref{app:planner_details} provides the exact analytical derivations and loss formulations for the regret-aware constrained planner optimization. Finally, Section~\ref{app:awm_sim_agent} explains how the learned AWM is instantiated as a closed-loop sim-agent evaluator for transferable adversarial testing.

\begin{algorithm}[!htbp]
\caption{Decoupled Self-Play Adversarial Fine-Tuning}
\label{alg:awm_self_play}
\centering
\fcolorbox{black!35}{gray!3}{%
\begin{minipage}{0.96\linewidth}
\footnotesize
\begin{algorithmic}[1]
\Require Scenario distribution \(\mathcal D\); nominal planner \(\pi_{\mathrm{plan}}^{\mathrm{nom}}\); benign world model \(\pi_{\mathrm{pred}}\); coalition budget \(K_{\max}\); CVaR tail mass \(\tau\); safety thresholds \(\kappa\).
\Ensure Frozen adversarial host \(\pi_{\mathrm{awm}}^{\star}\), calibrated opponent distribution \(\mathcal P_{\mathrm{adv}}^{\star}\), robust planner \(\widehat{\pi}_{\mathrm{plan}}\).
\State Initialize planner and world-model branches with the shared autoregressive factorization in Eq.~\ref{eq:factorization}; set \(\pi_{\mathrm{ref}}\leftarrow\pi_{\mathrm{plan}}^{\mathrm{nom}}\).

\Statex \textbf{Stage A: learn the adversarial world model for Eq.~\ref{eq:stage_a_min}.}
\For{AWM minibatches \(\mathcal B\subset\mathcal D\)}
    \For{each scene \(s\in\mathcal B\)}
        \State Background probe: assign \(\rho_i=\text{\texttt{bg}}\) for all non-ego agents and roll out the role-conditioned generator in Eq.~\ref{eq:role_generation}.
        \State Build a sparse proposal pool \(P_s\) using realism-penalized threat and residual coverage (Appendix~\ref{app:awm_probe}).
        \If{\(P_s=\emptyset\)}
            \State Set \(A_s^{\mathrm{top1}}=A_s^{\mathrm{pair}}=\emptyset\) and \(p_s=0\).
        \Else
            \State Select \(i_s^\star=\arg\max_{i\in P_s} f_{\mathrm{pri}}(x_i)\); condition on \(\rho_{i_s^\star}=\text{\texttt{atk}}\) and choose \(\hat j_s=\arg\max_{j\in P_s\setminus\{i_s^\star\}} f_{\mathrm{pair}}(\phi(i_s^\star,j))\).
            \State Define \(A_s^{\mathrm{top1}}=\{i_s^\star\}\), \(A_s^{\mathrm{pair}}=\{i_s^\star,\hat j_s\}\), and pair-admission probability \(p_s=\sigma(\alpha_s^{\mathrm{cal}}/\tau_{\mathrm{risk}}))\).
            \State Update \(\pi_{\mathrm{awm}}\), \(f_{\mathrm{pri}}\), and \(f_{\mathrm{pair}}\) with role-switched counterfactual credit (Eq.~\ref{eq:hybrid_awm_reward}); update the calibrator \(q_\eta\) using Appendix~\ref{app:awm_training}.
        \EndIf
    \EndFor
\EndFor
\State Freeze \(\pi_{\mathrm{awm}}^{\star}\leftarrow\pi_{\mathrm{awm}}\) and \(q_\eta^{\star}\leftarrow q_\eta\).
\State Define \(\mathcal P_{\mathrm{adv}}^{\star}(B_s=\text{\texttt{top1}}\mid s)=1-p_s\) and \(\mathcal P_{\mathrm{adv}}^{\star}(B_s=\text{\texttt{pair}}\mid s)=p_s\) as in Eq.~\ref{eq:adv_branch_dist}.

\Statex \textbf{Stage B: train the constrained planner best response in Eq.~\ref{eq:stage_b_max}.}
\For{planner minibatches \(\mathcal B\subset\mathcal D\)}
    \For{each scene \(s\in\mathcal B\)}
        \State Roll out \(\pi_\theta\) on the benign branch \(\pi_{\mathrm{pred}}\) and adversarial branches \(B_s\in\{\text{\texttt{top1}},\text{\texttt{pair}}\}\) induced by \(\pi_{\mathrm{awm}}^{\star}\).
        \State Compute the nominal-retention loss \(\mathcal L_{\mathrm{nom}}\) in Eq.~\ref{eq:nominal_trust_region}.
        \State Compute adversarial regrets \(\Delta_s^b(\theta)=U_s^b(\pi_{\mathrm{ref}})-U_s^b(\pi_\theta)\) and the two-point regret distribution in Eqs.~\ref{eq:adv_regret}--\ref{eq:regret_dist}.
        \State Obtain \(w_s^{\mathrm{top1}},w_s^{\mathrm{pair}}\) from the closed-form regret-CVaR allocation in Appendix~\ref{app:cvar_weights}.
        \State Accumulate the adversarial tail-risk loss \(\mathcal L_{\mathrm{rob}}\) in Eq.~\ref{eq:rob_loss} and safety statistics $\mathcal L_{\mathrm{dual}}$.
    \EndFor
    \State Update \(\theta\) with \(\mathcal L_{\mathrm{nom}}+\mathcal L_{\mathrm{rob}}\); update dual variables using thresholds \(\kappa\) (Appendix~\ref{app:dual_update}).
\EndFor
\State Select \(\widehat{\pi}_{\mathrm{plan}}\) by validation robustness under \(\mathcal P_{\mathrm{adv}}^{\star}\) subject to nominal-retention criteria.
\State \Return \(\pi_{\mathrm{awm}}^{\star}\), \(\mathcal P_{\mathrm{adv}}^{\star}\), and \(\widehat{\pi}_{\mathrm{plan}}\).
\end{algorithmic}
\end{minipage}}
\end{algorithm}

\subsection{Discrete Autoregressive Planning Architecture and Role Conditioning}
\label{app:architecture}

This section details the architectural choices required to implement the unified planning-prediction decomposition introduced in Section~\ref{sec:preliminary}, and expands upon the role-conditioned generation mechanism. By casting both planning and environment simulation as discrete next-token prediction tasks \cite{tang2025plan}, we ensure that adversarial interventions modify only the behavioral distribution of the agents, without violating the underlying physical transition dynamics.

\subsubsection{Unified Autoregressive Motion Generation Interface}
\label{app:ar_interface}
A cornerstone of our framework is that both the ego planner and the environment evolution can be formulated as discrete next-token prediction tasks. Let $x_t$ denote the interactivep scene state at step $t$, which includes static map elements, historical agent trajectories, and the generated future prefix up to $t$. For each agent type (e.g., vehicles, pedestrians), we maintain a discrete motion vocabulary $\mathcal{T}$ whose entries correspond to local motion primitives (e.g., discretized displacement and heading increments $(\Delta x, \Delta y, \Delta \psi)$) over a fixed autoregressive interval following previous practices \cite{wu2024smart,seff2023motionlm}.

During the rollout, the standard two-branch architecture factorizes the joint one-step generation into a planning branch $\pi_{\mathrm{plan}}$ for the ego agent and an environment branch $\pi_{\mathrm{pred}}$ for all non-ego agents:
\begin{equation}
    p(\vz_t \mid x_t) = \pi_{\mathrm{plan}}(z_{e,t} \mid x_t) \prod_{i \in \mathcal{V}_s} \pi_{\mathrm{pred}}(z_{i,t} \mid x_t).
\end{equation}
Once the step-wise token set $\vz_t = \{z_{i,t}\}_{i \in \{e\} \cup \mathcal{V}_s}$ is sampled, it is decoded into continuous kinematic updates. A deterministic transition operator $\Gamma$ then updates the global coordinates and advances the scene state: $x_{t+1} = \Gamma(x_t, \vz_t)$.
Crucially, this rollout interface is strictly preserved throughout our framework. The AWM alters \textit{who} generates the non-ego tokens and \textit{how} they behave, but it does not modify the underlying simulator, the token space, or the physical transition dynamics.
Consequently, the normal, top-1, and pair adversarial branches share the identical state representation, transition operator, and physical motion manifold, ensuring that the planner is evaluated fairly across different attack complexities and are strictly comparable.

\subsubsection{Model Instantiation and Role Conditioning Mechanism}
\label{app:role_injection_details}
To instantiate the policies, both the planner and the environment branches share a unified high-level network template but maintain separate parameters. Specifically, for a branch $k \in \{\texttt{pred}, \texttt{plan}\}$, the architecture comprises a map encoder $E_{\mathrm{map}}^k$, an interaction backbone $B^k$, and a token decoder $D^k$:
\begin{equation}
    g^k = E_{\mathrm{map}}^k(x_t), \quad h_{i,t}^k = B^k(x_t, g^k), \quad \ell_{i,t}^k = D^k(h_{i,t}^k),
\end{equation}
where $h_{i,t}^k$ is the hidden state for agent $i$ and $\ell_{i,t}^k$ are the output logits over the vocabulary $\mathcal{T}$. The backbone $B^k$ typically employs stacked transformer blocks over relation graphs (e.g., temporal self-attention, map-to-agent cross-attention, and agent-agent interaction) \cite{tang2025plan,wu2024smart,zhou2024behaviorgpt}.

To construct the AWM from the pretrained predictive model $\pi_{\mathrm{pred}}$ without training a separate adversarial generator from scratch, we introduce a latent role conditioning $\rho_i \in \{\texttt{ego}, \texttt{bg}, \texttt{atk}\}$ for each agent.
The role identity is injected into the network at two levels to alter the agent's interaction behavior and output distribution. First, an agent-side role embedding $E_{\mathrm{agent}}(\rho_i)$ is added to the agent's initial hidden state before the spatial-temporal interaction blocks, allowing the role to influence how the agent attends to the map and other vehicles. Second, a decoder-side embedding $E_{\mathrm{dec}}(\rho_i)$ is added immediately before the final token projection. The role-conditioned generation is given by:

\begin{equation}
    \tilde{h}_{i,t}^{\mathrm{awm}} = h_{i,t}^{\mathrm{awm}} + E_{\mathrm{agent}}(\rho_i), \quad \tilde{\ell}_{i,t}^{\mathrm{awm}} = D^{\mathrm{awm}}\left(\tilde{h}_{i,t}^{\mathrm{awm}} + E_{\mathrm{dec}}(\rho_i)\right).
\end{equation}
The background role embedding $E(\texttt{bg})$ is initialized as a zero vector. Consequently, when no attacker is activated ($\rho_i = \texttt{bg}, \forall i \in \mathcal{V}_s$), the AWM perfectly recovers the naturalistic behavior of the pretrained $\pi_{\mathrm{pred}}$. Furthermore, switching between factual and counterfactual coalitions (as required in Section~\ref{sec:awm}) only requires modifying the role table $\rho_i$ without altering the network parameters or the simulator. This guarantees that the calculated counterfactual utility gaps isolate the true marginal contribution of the coalition rather than architectural discrepancies, which is the cornerstone of the counterfactual credit assignment (Section~\ref{sec:awm}).

\subsection{Algorithmic Details of Adversarial Coalition Learning}
\label{app:awm_details}

This section expands upon the scene-adaptive coalition learning pipeline, detailing the mathematical formulations for the probe signals, the neural selector inputs, and the multi-stage training objectives.

\subsubsection{Probe-Time Threat Signals and Residual Proposal Construction}
\label{app:awm_probe}
Directly evaluating the neural primary selector on all background agents is computationally prohibitive. The objective of the proposal stage is to reduce the combinatorial multi-agent search space into a sparse potential candidate.
Therefore, we execute a lightweight background probe rollout (where $\rho_i = \texttt{bg}, \forall i \in \mathcal{V}_s$) to extract heuristic threat signals and construct a sparse proposal set $P_s$.

For each non-ego agent $i$, we first compute a dense, heuristic ego-centric threat signal $r_{i,\mathrm{dense}}(t)$ at each step $t$, which aggregates proximity ($r_{\mathrm{prox}}$), time-to-collision (TTC) approximations ($r_{\mathrm{cpa}}$), blocking behaviors ($r_{\mathrm{block}}$), and cut-in maneuvers ($r_{\mathrm{cutin}}$):
\begin{equation}
    r_{i,\mathrm{dense}}(t) = w_{\mathrm{prox}} r_{i,\mathrm{prox}}(t) + w_{\mathrm{cpa}} r_{i,\mathrm{cpa}}(t) + w_{\mathrm{block}} r_{i,\mathrm{block}}(t) + w_{\mathrm{cutin}} r_{i,\mathrm{cutin}}(t).
\end{equation}
To capture severe safety violations, we define the raw probe signal by incorporating hard collision indicators:
\begin{equation}
    r_i^{\mathrm{probe}}(t) = \max \left( r_{i,\mathrm{dense}}(t), \lambda_{\mathrm{ttc}} \mathbb{I}_{i,\mathrm{ttc}}(t), \lambda_{\mathrm{col}} \mathbb{I}_{i,\mathrm{col}}(t) \right).
\end{equation}

To discourage the selection of agents that might execute physically invalid or trivial attacks (e.g., driving off-road to hit the ego), we define a comprehensive penalized signal:
\begin{equation}
    \bar{r}_i^{\mathrm{probe}}(t) = r_i^{\mathrm{probe}}(t) - w_{\mathrm{kin}} c_i^{\mathrm{kin}}(t) - w_{\mathrm{road}} c_i^{\mathrm{road}}(t) - w_{\mathrm{obs}} c_i^{\mathrm{obs}}(t),
\end{equation}
where $c_i^{(\cdot)}$ are penalty indicators for kinematic violations, off-road driving, and static obstacle collisions, respectively, weighted by constants $w_{(\cdot)}$.
Using the discounted return operator $G(v)_\tau = \sum_{t=\tau}^{H} \gamma^{t-\tau} v(t)$, we obtain the scalarized pure threat score $q_i = \max_{\tau} G(r_i^{\mathrm{probe}})_\tau$ and the comprehensive score $\bar{q}_i = \max_{\tau} G(\bar{r}_i^{\mathrm{probe}})_\tau$.
The first proposal candidate is greedily selected based on the highest $\bar{q}_i$. To ensure the proposal set $P_s$ is diverse, subsequent candidates are added based on their \textit{residual threat coverage}. Given a partially constructed proposal set $S$, the residual threat of a new agent $i$ is $\tilde{r}_i^S(t) = \max\left(r_i^{\mathrm{probe}}(t) - \max_{j \in S} r_j^{\mathrm{probe}}(t), 0\right)$. The residual score is $\psi_i^S = \max_{\tau} G(\tilde{r}_i^S)_\tau$. We only admit agent $i$ if its relative residual ratio $\eta_i^S = \psi_i^S / (q_i + \epsilon)$ exceeds a predefined threshold. The rationale behind this design is to ensure that newly proposed agents explain complementary spatial-temporal adversarial potential rather than merely duplicating the interaction patterns of already selected candidates, thereby preserving diversity for the pair-ranking stage.

\subsubsection{Feature Representations for Primary Selector and Pair Ranker}
\label{app:awm_features}
While scalar probe scores provide a heuristic prior for efficient filtering, they lack the higher-order interaction context necessary to identify the true optimal attacker. The learnable primary selector $f_{\mathrm{pri}}$ and the conditional pair ranker $f_{\mathrm{pair}}$ bridge this gap by fusing neural hidden states with proposal statistics.

\paragraph{Primary Selector.} For each candidate $i \in P_s$, the primary selector $f_{\mathrm{pri}}$ operates on a fused feature vector $x_i$ that concatenates the high-dimensional hidden state from the background probe with the heuristic proposal statistics:
\begin{equation}
    x_i = \left[ h_i^{\mathrm{probe}}, \; \mathbb{I}(i \in P_s), \; \frac{\mathrm{rank}(i)}{B_{\mathrm{prop}}}, \; q_i, \; \bar{q}_i, \; \psi_i^S, \; m_i \right],
\end{equation}
where $h_i^{\mathrm{probe}}$ is the final-layer hidden state of the AWM extracted from the background probe, $B_{\mathrm{prop}}$ is the proposal budget, and $m_i$ is a binary stage marker. The selector outputs a confidence score $a_i = f_{\mathrm{pri}}(x_i)$, and the primary attacker is determined as $i_s^* = \arg\max_{i \in P_s} a_i$.
After selecting the primary attacker, we retain a small runner-up candidate set to serve as potential partners:
\begin{equation}
R_s =
\operatorname{Top}_{B_{\mathrm{pair}}}
\left(
\left\{ i \in P_s \setminus \{i_s^*\} \right\};
a_i
\right),
\end{equation}
where $B_{\mathrm{pair}}$ is the maximum partner-candidate budget and
$\operatorname{Top}_{B_{\mathrm{pair}}}(\cdot; a_i)$ returns the candidates with the largest primary-selector scores. Thus, $P_s$ is the proposal pool produced by the residual threat-coverage procedure, while $R_s$ is the smaller selector-filtered candidate pool used for conditional pair ranking.

\paragraph{Conditional Pair Ranker and Calibrator.}
To evaluate the synergistic value of a secondary attacker, we execute a second conditional probe where $i_s^*$ is assigned $\rho = \texttt{atk}$. This yields conditional hidden states $h^{\mathrm{cond}}$ that encode the scene's reaction to the primary attack. For any candidate partner $j \in P_s \setminus \{i_s^*\}$, the pair ranker $f_{\mathrm{pair}}$ takes a relational feature $\phi(i_s^*, j)$ that captures the conditional interaction dynamics:
\begin{equation}
    \phi(i_s^*, j) = \left[ h_{i_s^*}^{\mathrm{cond}}, \; h_j^{\mathrm{cond}}, \; \left|h_{i_s^*}^{\mathrm{cond}} - h_j^{\mathrm{cond}}\right|, \; h_{i_s^*}^{\mathrm{cond}} \odot h_j^{\mathrm{cond}}, \; \bar{q}_{i_s^*}, \; \bar{q}_j, \; \psi_j^{\{i_s^*\}}, \; \Delta a_{i_s^*,j} \right],
\end{equation}
where $\odot$ denotes element-wise multiplication, and $\Delta a_{i_s^*,j} = f_{\mathrm{pri}}(x_{i_s^*}) - f_{\mathrm{pri}}(x_j)$ represents the confidence gap from the primary selector. The pair ranker predicts the marginal coalition gain $\hat{g}_s(i_s^*, j) = f_{\mathrm{pair}}(\phi(i_s^*, j))$, and the optimal partner is $\hat{j}_s = \arg\max_j \hat{g}_s(i_s^*, j)$.

Subsequently, the scene-adaptive calibrator $q_\eta$ determines whether admitting $\hat{j}_s$ is globally beneficial (should actually be admitted). The scene-level feature aggregates the selected pair representation, the predicted pair gain, the primary-selector confidence, and two cardinality-based uncertainty indicators:
\begin{equation}
    \psi_s = \left[ \phi(i_s^*, \hat{j}_s), \; \hat{g}_s(i_s^*, \hat{j}_s), \; a_{i_s^*}, \; \sigma(a_{i_s^*}), \; |P_s|, \; |R_s| \right].
\end{equation}
Here, $|P_s|$ denotes the number of plausible attackers admitted by the proposal stage, and $|R_s|$ denotes the number of selector-retained runner-up candidates considered for pair ranking. These two scalars provide the calibrator with scene-level ambiguity information: a larger $|P_s|$ indicates a denser or more ambiguous adversarial proposal pool, while a larger $|R_s|$ indicates more plausible secondary partners after primary selection.
The calibrator is trained via Binary Cross-Entropy (BCE) to predict the probability $p_s$ that the true marginal improvement $G_s^{\mathrm{pair}} = U_s(\{i_s^*\}) - U_s(\{i_s^*, \hat{j}_s\})$ is strictly positive.

\subsubsection{Multi-Stage Training Objective for AWM}
\label{app:awm_training}
To stabilize the optimization of the role-conditioned policy alongside the coalition search modules, we employ a two-stage training paradigm.

\textbf{Stage I (Host and Search Optimization):} We jointly train the role-conditioned AWM policy, the primary selector, and the conditional pair ranker. The objective is:
\begin{equation}
    \mathcal{L}_{\mathrm{stage1}} = \mathcal{L}_{\mathrm{policy}}(\Delta^{\mathrm{hyb}}) + \lambda_{\mathrm{pri}} \mathcal{L}_{\mathrm{pri}} + \lambda_{\mathrm{pair}} \mathcal{L}_{\mathrm{pair}} + \Omega_{\mathrm{stab}},
\end{equation}
Here, $\mathcal{L}_{\mathrm{policy}}$ is the policy gradient loss driven by the hybrid counterfactual reward $\Delta_i^{\mathrm{hyb}}$ (defined in Eq.~\ref{eq:hybrid_awm_reward}). $\mathcal{L}_{\mathrm{pri}}$ and $\mathcal{L}_{\mathrm{pair}}$ supervise the selector and ranker using the observed utility gaps as regression and ranking targets.
Specifically: (1) For the primary selector, the target label for candidate $i$ is its absolute adversarial impact compared to the nominal background traffic: $y_i^{\mathrm{pri}} = U_s(\emptyset) - U_s(\{i\})$, where $\emptyset$ denotes no active attackers.
(2) For the conditional pair ranker, the target label for a partner candidate $j$ is its marginal coalition gain over the primary attacker $i_s^*$: $y_{i_s^*, j}^{\mathrm{pair}} = U_s(\{i_s^*\}) - U_s(\{i_s^*, j\})$.
Finally, the stabilizer $\Omega_{\mathrm{stab}}$ anchors the AWM's background behavior to the pretrained predictive prior $\pi_{\mathrm{pred}}$ via KL divergence, preventing the model from collapsing into unrealistic adversarial behaviors.

\textbf{Stage II (Calibrator Optimization):} We freeze the components from Stage-I and train the scene-level calibrator $q_\eta$ using the BCE loss $\mathcal{L}_{\mathrm{cal}}$. This decoupling is crucial because the calibrator's objective is to act as an adaptive gate that decides \textit{when} a multi-agent attack is statistically justified, rather than discovering new attackers.
Stage-II trains the scene-level calibrator $q_\eta$ to ensure unbiased admission probabilities.

\subsection{Details of Regret-Aware Constrained Planner Optimization}
\label{app:planner_details}

This section details the planner-side robust optimization (Section~\ref{sec:planner}), including the exact formulation of the trust region constraints, the closed-form derivation of the CVaR weights, and the dual update mechanism for safety constraints. It explains how the AWM is converted into a risk distribution and how the constrained optimization is solved.

\subsubsection{Reference-Anchored Nominal Retention}
\label{app:planner_nominal}
During the training of the planner, we do not sample arbitrary attacks from a replay buffer. Instead, for each scene $s$, the frozen AWM executes the coalition pipeline to generate a top-1 attacker gate $g_s^{\mathrm{top1}}$ and a pair gate $g_s^{\mathrm{pair}}$ (if a valid partner exists). This yields three explicit rollout branches, evaluated by their respective utilities:
\begin{equation}
    J_s^{\mathrm{norm}}(\pi) = U_s(\pi, \pi_{\mathrm{pred}}), \quad J_s^{\mathrm{top1}}(\pi) = U_s(\pi, \pi_{\mathrm{awm}}; g_s^{\mathrm{top1}}), \quad J_s^{\mathrm{pair}}(\pi) = U_s(\pi, \pi_{\mathrm{awm}}; g_s^{\mathrm{pair}}).
\end{equation}

To prevent the planner from sacrificing naturalistic driving capabilities in pursuit of worst-case robustness (Section~\ref{sec:planner}), we enforce a reference-anchored nominal retention objective on the normal branch. Let $\pi_{\mathrm{ref}}$ be the frozen baseline planner. We define the normal utility margin as $m_s^{\mathrm{norm}}(\theta) = J_s^{\mathrm{norm}}(\pi_\theta) - J_s^{\mathrm{norm}}(\pi_{\mathrm{ref}})$.
To penalize degradation without forcing the planner to over-optimize already safe scenes, we introduce an asymmetric margin penalty that activates only when the planner's performance drops below an acceptable margin $\epsilon_{\mathrm{norm}} \ge 0$:
\begin{equation}
    \psi_s^{\mathrm{norm}}(\theta) = \left[ -\epsilon_{\mathrm{norm}} - m_s^{\mathrm{norm}}(\theta) \right]_+^2,
\end{equation}
where $[\cdot]_+ = \max(\cdot, 0)$. The total nominal retention loss combines the standard policy gradient surrogate $\ell_s^{\mathrm{norm}}(\theta)$, a trajectory-level KL divergence constraint, and the margin penalty:
\begin{equation}
    \mathcal{L}_{\mathrm{nom}}(\theta) = \mathbb{E}_{s \sim \mathcal{D}} \left[ \lambda_{\mathrm{pg}} \ell_s^{\mathrm{norm}}(\theta) + \lambda_{\mathrm{kl}} D_{\mathrm{KL}}(\pi_\theta \parallel \pi_{\mathrm{ref}} \mid x_{s, 1:H}^{\mathrm{norm}}) + \lambda_{\mathrm{margin}} \psi_s^{\mathrm{norm}}(\theta) \right],
\end{equation}
where $\ell_s^{\mathrm{norm}}(\theta)$ is the standard GRPO \cite{shao2024deepseekmath} surrogate loss maximizing nominal utility.
This asymmetric design penalizes performance degradation but does not force the planner to over-optimize on already safe scenes.

\subsubsection{Analytical Regret-CVaR Weighting}
\label{app:cvar_weights}

The core of robust adaptation is to optimize the risk objective with respect to the AWM-induced regret distribution.
To optimize the adversarial tail-risk, we compute the reference-relative regret $\Delta_s^{b}(\theta) = J_s^{b}(\pi_{\mathrm{ref}}) - J_s^{b}(\pi_\theta)$ for $b \in \{\mathrm{top1}, \mathrm{pair}\}$. The AWM calibrator induces a discrete two-point risk distribution:
\begin{equation}
    \mathcal{D}_s^\theta = \left\{ \left(\Delta_s^{\mathrm{top1}}(\theta), 1-p_s\right), \left(\Delta_s^{\mathrm{pair}}(\theta), p_s\right) \right\},
\end{equation}
where $p_s$ is the calibrated probability of admitting the pair attack.
We apply the Conditional Value-at-Risk (CVaR) operator \cite{rockafellar2000optimization} at tail level $\tau \in (0, 1]$ over this specific distribution.
Because AWM naturally induces a discrete two-point distribution over the attack types $b \in \{\texttt{top1}, \texttt{pair}\}$, the CVaR allocation admits an elegant \textit{closed-form analytical solution}, avoiding the high variance typically associated with sampling-based CVaR estimation.
Let $b_{(1)}$ and $b_{(2)}$ denote the two branches sorted by their regret in descending order (i.e., $\Delta_s^{b_{(1)}} \ge \Delta_s^{b_{(2)}}$), and let $\mu_s^{b_{(1)}}$ be the probability mass of the worst-case branch.
For a specified tail level $\tau$, the CVaR operator focuses strictly on the worst $\tau$-proportion of the distribution. The exact dynamic weight $w_s^{b_{(1)}}$ for the worse branch is given by:
\begin{equation}
    w_s^{b_{(1)}}(\tau) =
    \begin{cases}
        1, & \text{if } \mu_s^{b_{(1)}} \ge \tau, \\
        \frac{\mu_s^{b_{(1)}}}{\tau}, & \text{if } \mu_s^{b_{(1)}} < \tau.
    \end{cases}
\end{equation}
The weight for the better branch is simply the residual tail mass: $w_s^{b_{(2)}}(\tau) = 1 - w_s^{b_{(1)}}(\tau)$.
Intuitively, these weights dynamically route the gradient updates. If the probability of the worst-case attack exceeds $\tau$, the planner focuses entirely on surviving that specific attack ($w=1$) and absorbs as much of the tail mass $\tau$ as its probability allows. If the worst-case attack is extremely rare (probability $< \tau$), the planner distributes its learning capacity proportionally between both branches.

In our implementation, these weights are treated as detached constants during backpropagation. The final robust objective dynamically scales the branch-wise policy gradient losses:
\begin{equation}
    \mathcal{L}_{\mathrm{rob}}(\theta) = \mathbb{E}_{s \sim \mathcal{D}} \left[ \text{sg}\left(w_s^{\mathrm{top1}}(\tau)\right) \ell_s^{\mathrm{top1}}(\theta) + \text{sg}\left(w_s^{\mathrm{pair}}(\tau)\right) \ell_s^{\mathrm{pair}}(\theta) \right],
\end{equation}
where $\text{sg}(\cdot)$ denotes the stop-gradient operator, making the planner optimize its policy rather than attempting to manipulate the risk distribution itself.
This dynamic weighting reconciles the sparsity of the AWM's attacks with the planner's safety demands, ensuring the planner's gradients are guided by the scene-adaptive vulnerabilities discovered by the AWM.
Because pair attacks are scene-dependent, their calibrated probability $p_s$ is often small. Under a standard expected loss, a rare but fatal pair attack would have its gradient severely diluted by $p_s$, causing the planner to under-optimize for it. By applying CVaR, if the low-probability pair branch currently induces the highest regret, the operator aggressively upweights it (e.g., scaling its weight from $p_s$ to $p_s/\tau$ or $1$). This encourages the planner to defend against its most severe weakness rather than just average-case perturbations.

\subsubsection{Safety-Constrained Dual Update Parameterization}
\label{app:dual_update}
To prevent the planner from exploiting the regret objective by violating hard traffic rules (e.g., evading an attacker by driving off-road), we enforce explicit safety constraints via a dual update mechanism (Section~\ref{subsubsec:dual_update}).
Let $v_{s,m}^b \in [0, 1]$ denote the violation indicator for the $m$-th safety cost under branch $b$. The expected branch-mixed violation rate is $\bar{v}_{s,m} = (1-p_s) v_{s,m}^{\texttt{top1}} + p_s v_{s,m}^{\texttt{pair}}$. To ensure optimization stability and guarantee that the utility multipliers remain strictly positive and bounded, we parameterize the progress reward multiplier $\lambda_R$ and the $M$ safety cost multipliers $\lambda_m$ jointly using a softmax projection over unconstrained dual variables $\zeta \in \mathbb{R}^M$ and a fixed progress logit $\alpha_R$:
\begin{equation}
    [\lambda_R, \lambda_1, \ldots, \lambda_M] = \operatorname{softmax}([\alpha_R, \zeta_1, \ldots, \zeta_M]).
\end{equation}
This parameterization ensures $\lambda_m \in (0, 1)$ and $\sum \lambda = 1$. Given the admissible violation thresholds $\kappa \in \mathbb{R}^M$, the dual objective to be maximized with respect to $\zeta$ is:
\begin{equation}
    \mathcal{L}_{\mathrm{dual}}(\zeta) = \sum_{m=1}^{M} \lambda_m(\zeta) \left( \mathbb{E}_{s \sim \mathcal{D}}[\bar{v}_{s,m}] - \kappa_m \right).
\end{equation}
During training, the planner parameters $\theta$ and the dual variables $\zeta$ are updated via alternating gradient descent-ascent. This separation allows the planner to aggressively minimize adversarial regret while treating hard traffic rules as strict boundaries rather than easily exploitable reward shaping.

\subsection{AWM as a Closed-Loop Sim-Agent}
\label{app:awm_sim_agent}

This section describes how the learned AWM is used as a simulator-side traffic world model for the transferable adversarial evaluation in Table~\ref{tab:cross_planner_simagent}. The goal is to expose a planner to adversarial but feasible non-ego reactions within the closed-loop reactive simulation. This replaces the rule-based or replay agents with our learned AWM policies.
We therefore instantiate AWM as a drop-in non-ego policy: the evaluated planner still controls only the ego vehicle, the simulator still advances the world state with its standard transition and scoring logic, and AWM only changes the motion proposals of selected surrounding agents.

\paragraph{State synchronization and policy roles.}
At each simulation step \(t\), the simulator provides the current closed-loop state \(x_t\), including the map context, the ego state produced by the evaluated planner, the generated non-ego histories, and the generated future prefix. We encode \(x_t\) using the same tokenized autoregressive representation used in the AWM training. The normal sim-agent samples each non-ego token from the background predictive policy \(\pi_{\mathrm{pred}}\). The adversarial sim-agent instead runs the AWM coalition pipeline on the current state: the proposal module constructs a sparse candidate set, the primary selector chooses the main adversarial agent, the pair ranker proposes a complementary partner, and the calibrator decides whether the pair branch is admitted. The selected coalition receives the adversarial role \(\rho_i=\texttt{atk}\), while all other agents retain the background role \(\rho_i=\texttt{bg}\).

\paragraph{Closed-loop execution.}
The sim-agent executes a mixed non-ego policy rather than replacing the entire traffic scene with adversarial behavior. For each non-ego agent \(i\), we first generate a normal token and, if \(i\) belongs to the admitted coalition \(A_t\), an adversarial token from the role-conditioned AWM. The token finally submitted to the simulator is
\begin{equation}
z_{i,t}^{\mathrm{sim}} =
\begin{cases}
z_{i,t}^{\mathrm{awm}} \sim \pi_{\mathrm{awm}}(\cdot \mid x_t,\rho_i=\texttt{atk}), & i \in A_t \ \text{and}\ g_{i,t}=1,\\
z_{i,t}^{\mathrm{norm}} \sim \pi_{\mathrm{pred}}(\cdot \mid x_t,\rho_i=\texttt{bg}), & \text{otherwise},
\end{cases}
\end{equation}
where \(g_{i,t}\) is a feasibility gate defined below. The accepted token is decoded into a short-horizon displacement and heading update, and the simulator applies the same state transition used for the normal learned sim-agent. This design keeps the evaluation controlled: the AWM changes the conditional behavior distribution of selected agents, but it does not alter the ego planner, the map, the vehicle model, the collision checker, or the closed-loop score.

\paragraph{Feasibility-gated adversarial replacement.}
Before an adversarial token is accepted, the sim-agent checks whether the decoded motion remains within the shared motion vocabulary's plausible region and the simulator's local validity constraints. The gate rejects proposals that introduce implausible kinematic jumps, immediate map invalidity, static-object conflicts, or degenerate ego-local overlaps that would be better interpreted as simulator artifacts than as adversarial traffic. Formally, the gate can be viewed as
\begin{equation}
g_{i,t}
= \mathbb{I}\!\left[
c_{i,t}^{\mathrm{kin}}\le \kappa_{\mathrm{kin}},\
c_{i,t}^{\mathrm{map}}\le \kappa_{\mathrm{map}},\
c_{i,t}^{\mathrm{static}}\le \kappa_{\mathrm{static}},\
c_{i,t}^{\mathrm{ego}}\le \kappa_{\mathrm{ego}}
\right],
\end{equation}
where the four checks summarize kinematic continuity, map consistency, static-object validity, and ego-local sanity constraints. If the gate fails, the agent falls back to the normal sim-agent token for the same state and step. The fallback makes the AWM host an adversarial simulator under feasibility constraints, rather than a mechanism for injecting arbitrary invalid trajectories.

\paragraph{Temporal use of the coalition.}
The coalition is selected from the current closed-loop state, not from a fixed logged trajectory. This matters because the evaluated planner can change the future interaction geometry, and the AWM must respond to the planner's generated ego behavior. Recomputing the role assignment online lets the adversarial host adapt when the ego yields, accelerates, or changes its path. At the same time, the attack budget \(K_{\max}=2\) and the calibrator prevent the sim-agent from turning every nearby participant into an attacker. Therefore, the closed-loop host remains sparse: it stresses the planner through a small number of behaviorally meaningful agents while leaving the rest of the traffic under the normal predictive policy.

\paragraph{Interpretation of AWM-sim scores.}
The AWM sim-agent score is a closed-loop evaluation score under an adversarial non-ego traffic host, and it is distinct from both open-loop AWM rollout rewards and ordinary planner closed-loop scores reported in the main benchmark (Table~\ref{tab:main_benchmark}). A lower score under the AWM host indicates that the planner is more vulnerable when nearby traffic follows the learned adversarial policy. To verify that score drops are caused by controlled adversarial interaction rather than simulator collapse, we report simulator-side diagnostics in Table~\ref{tab:app_simagent_diagnostics}.

\newpage

\section{Theoretical Analysis}
\label{app:theory_final}

This section provides a theoretical justification for the decoupled solver introduced in Section~\ref{sec:method}. We do not claim an exact neural solution to the full non-convex simultaneous game in Eq.~\ref{eq:ideal_game}; such games are generally intractable and unstable even in simpler smooth two-player settings~\cite{mescheder2017numerics,daskalakis2017training}. Instead, we analyze the implemented two-stage procedure as a local constrained Stackelberg surrogate of the ideal robust game.
In this surrogate, Section~\ref{sec:awm} learns a structured adversarial world model around a frozen reference planner, and Section~\ref{sec:planner} computes a regret-sensitive constrained best response inside a reference-anchored trust region.
The analysis formalizes four claims: (1) freezing the AWM induces a controlled local approximation error determined by Stage-A optimization error, structured-class coverage, and the planner trust-region radius; (2) role-switched counterfactual rewards are valid difference-reward baselines for attacker-side score-function optimization when they are evaluated as action-independent stop-gradient counterfactuals; (3) regret-CVaR is an appropriate planner-side risk signal because regret removes intrinsic scene difficulty while CVaR emphasizes rare but severe adversarial branches; and (4) nominal retention and the projected safety dual provide conservative local safeguards against obtaining robustness by degrading ordinary driving behavior.

\subsection{Notation and Autoregressive Local Stability}
\label{app:theory_setup_final}

For a scene $s$, let $\Omega_s^\star$ denote the set of all admissible adversarial executions.
An element $\omega\in\Omega_s^\star$ specifies both an active coalition $A_s\subseteq \mathcal V_s$ satisfying $|A_s|\le K_{\max}$ and the corresponding role-conditioned non-ego rollout law induced by the AWM under the shared tokenized transition operator.
Let $\nu_s\in\Delta(\Omega_s^\star)$ be a scene-conditional adversarial environment distribution.
For a planner $\pi$, define the expected finite-horizon utility
\begin{equation}
    J_s(\pi,\nu_s)
    :=
    \mathbb E_{\omega\sim \nu_s}
    \mathbb E_{\tau\sim(\pi,\omega)}
    \left[
        \sum_{t=1}^{H}\gamma^{t-1}u_t(x_t,z_t)
    \right],
\end{equation}
where the rollout follows the autoregressive dynamics $x_{t+1}=\Gamma(x_t,\vz_t)$ from Section~\ref{subsec:autoregressive_generation}.
Throughout the analysis, we assume the step utility is bounded:
\begin{equation}
    |u_t(x_t,z_t)|\le u_{\max},
    \qquad
    \forall t\in\{1,\ldots,H\}.
\end{equation}
Define the finite-horizon sensitivity constant
\begin{equation}
    \Lambda_{H,\gamma}
    :=
    2u_{\max}\sum_{t=1}^{H}t\gamma^{t-1}
    =
    \begin{cases}
        \displaystyle
        2u_{\max}
        \frac{1-(H+1)\gamma^H+H\gamma^{H+1}}{(1-\gamma)^2},
        & \gamma\in(0,1),\\[8pt]
        u_{\max}H(H+1),
        & \gamma=1.
    \end{cases}
    \label{eq:final_lambda_H}
\end{equation}
Let $\pi_{\mathrm{ref}}$ be the frozen reference planner.
For the theoretical analysis, we use the stronger statewise trust region
\begin{equation}
    \Pi_\epsilon
    :=
    \left\{
        \pi:
        \sup_x
        D_{\mathrm{KL}}
        \big(
            \pi(\cdot\mid x)
            \,\|\,
            \pi_{\mathrm{ref}}(\cdot\mid x)
        \big)
        \le \epsilon
    \right\}.
    \label{eq:final_uniform_trust_region}
\end{equation}
The actual implementation uses a trajectory-level KL penalty, so Eq.~\ref{eq:final_uniform_trust_region} should be read as an idealized sufficient condition. We also implicitly restrict attention to planners that are absolutely continuous with respect to $\pi_{\mathrm{ref}}$ on the discrete token support, so that the KL divergence is finite.
By Pinsker's inequality, any $\pi\in\Pi_\epsilon$ satisfies
\begin{equation}
    \delta(\pi,\pi_{\mathrm{ref}})
    :=
    \sup_x
    D_{\mathrm{TV}}
    \big(
        \pi(\cdot\mid x),
        \pi_{\mathrm{ref}}(\cdot\mid x)
    \big)
    \le
    \sqrt{\epsilon/2}
    =:
    \delta_\epsilon .
    \label{eq:final_delta_epsilon}
\end{equation}

\begin{lemma}[Autoregressive planner sensitivity]
\label{lem:final_ar_sensitivity}
For any fixed scene $s$, any fixed adversarial environment distribution $\nu_s\in\Delta(\Omega_s^\star)$, and any two planners $\pi,\pi'$, the finite-horizon utility satisfies
\begin{equation}
    \big|
        J_s(\pi,\nu_s)-J_s(\pi',\nu_s)
    \big|
    \le
    \Lambda_{H,\gamma}\,
    \delta(\pi,\pi'),
    \label{eq:final_ar_sensitivity}
\end{equation}
where
\begin{equation}
    \delta(\pi,\pi')
    :=
    \sup_x
    D_{\mathrm{TV}}
    \big(
        \pi(\cdot\mid x),
        \pi'(\cdot\mid x)
    \big).
\end{equation}
Consequently, for every $\pi\in\Pi_\epsilon$,
\begin{equation}
    \big|
        J_s(\pi,\nu_s)-J_s(\pi_{\mathrm{ref}},\nu_s)
    \big|
    \le
    \Lambda_{H,\gamma}\delta_\epsilon .
    \label{eq:final_ar_sensitivity_trust}
\end{equation}
\end{lemma}

\emph{Proof.}
Couple the two rollouts of $\pi$ and $\pi'$ against the same sampled environment law $\omega\sim\nu_s$ and the same exogenous randomness. Conditional on identical prefixes up to a given step, the two rollouts have the same state and therefore the same non-ego conditional laws; we couple the non-ego samples identically until the ego token first differs.
As long as the generated token prefixes are identical, the scene states are identical because the transition operator $\Gamma$ is deterministic given the sampled token set.
At any step $t$, conditional on identical prefixes up to $t-1$, a maximal coupling of the ego token distributions gives
\begin{equation}
    \mathbb P(z_{e,t}\neq z'_{e,t})
    \le
    D_{\mathrm{TV}}
    \big(
        \pi(\cdot\mid x_t),
        \pi'(\cdot\mid x_t)
    \big)
    \le
    \delta(\pi,\pi').
\end{equation}
Let $E_t$ be the event that the two coupled rollouts have diverged at or before step $t$.
By the union bound,
\begin{equation}
    \mathbb P(E_t)
    \le
    t\,\delta(\pi,\pi').
\end{equation}
On $E_t^c$, the two-step utilities are identical.
On $E_t$, their difference is at most $2u_{\max}$.
Therefore,
\begin{equation}
\begin{aligned}
    \big|
        J_s(\pi,\nu_s)-J_s(\pi',\nu_s)
    \big|
    &\le
    \sum_{t=1}^{H}
    \gamma^{t-1}
    \cdot
    2u_{\max}
    \mathbb P(E_t) \\
    &\le
    2u_{\max}
    \sum_{t=1}^{H}
    t\gamma^{t-1}
    \delta(\pi,\pi') =
    \Lambda_{H,\gamma}\delta(\pi,\pi').
\end{aligned}
\end{equation}
Eq.~\ref{eq:final_ar_sensitivity_trust} follows from Eq.~\ref{eq:final_delta_epsilon}.
\hfill $\square$

\subsection{Local Stackelberg Justification of the Decoupled Stage-A/B Solver}
\label{app:theory_stackelberg_final}
We now analyze the error introduced by freezing Stage~A before updating the planner.
The ideal adversary class $\Delta(\Omega_s^\star)$ is too large to optimize directly.
Stage~A restricts it to the structured class induced by the AWM pipeline:
$\texttt{proposal}
    \;\rightarrow\;
    \texttt{primary selector}
    \;\rightarrow\;
    \texttt{conditional pair ranker}
    \;\rightarrow\;
    \texttt{scene calibrator}$.

Let $\widehat{\Omega}_s\subseteq\Omega_s^\star$ denote the support covered by this structured pipeline, and define $\mathcal{Q}_s^\star$ as the full set of admissible adversarial environment distributions satisfying the attack budget and realism constraints. Let $\widehat{\mathcal{Q}}_s\subseteq\mathcal{Q}_s^\star$ denote the structured class achieved by our AWM pipeline:

\begin{equation}
    \widehat{\mathcal Q}_s
    :=
    \Delta(\widehat{\Omega}_s)
    \subseteq
    \mathcal Q_s^\star
    :=
    \Delta(\Omega_s^\star).
\end{equation}
For a fixed planner $\pi$, define the full and structured worst-case utilities:
\begin{equation}
    g_s(\pi)
    :=
    \inf_{\nu_s\in\mathcal Q_s^\star}
    J_s(\pi,\nu_s),
    \qquad
    \widehat g_s(\pi)
    :=
    \inf_{\nu_s\in\widehat{\mathcal Q}_s}
    J_s(\pi,\nu_s).
\end{equation}
Because $\widehat{\mathcal Q}_s\subseteq\mathcal Q_s^\star$, we have
\begin{equation}
    \widehat g_s(\pi)\ge g_s(\pi).
\end{equation}
The structured class may not cover the full adversarial class. We define the scene-level coverage gap as
\begin{equation}
    \delta_{\mathrm{cov},s}(\pi)
    :=
    \widehat g_s(\pi)-g_s(\pi)
    \ge 0.
    \label{eq:final_coverage_gap_scene}
\end{equation}

\begin{proposition}[Local error of freezing the Stage-A adversary]
\label{prop:final_frozen_adv_gap}
Assume Stage~A returns a frozen structured adversary $\widehat\nu_s\in\widehat{\mathcal Q}_s$ satisfying the reference-planner suboptimality condition
\begin{equation}
    J_s(\pi_{\mathrm{ref}},\widehat\nu_s)
    \le
    \widehat g_s(\pi_{\mathrm{ref}})
    +
    \varepsilon_{A,s}.
    \label{eq:final_stageA_subopt}
\end{equation}
Then, for every planner $\pi\in\Pi_\epsilon$,
\begin{equation}
    0
    \le
    J_s(\pi,\widehat\nu_s)-g_s(\pi)
    \le
    \varepsilon_{A,s}
    +
    \delta_{\mathrm{cov},s}(\pi_{\mathrm{ref}})
    +
    2\Lambda_{H,\gamma}\delta_\epsilon .
    \label{eq:final_frozen_adv_gap_scene}
\end{equation}
\end{proposition}

\emph{Proof.}
The lower bound follows immediately from the definition of $g_s(\pi)$:
\begin{equation}
    g_s(\pi)
    =
    \inf_{\nu_s\in\mathcal Q_s^\star}
    J_s(\pi,\nu_s)
    \le
    J_s(\pi,\widehat\nu_s).
\end{equation}
For the upper bound, add and subtract the reference-planner quantities:
\begin{align}
    J_s(\pi,\widehat\nu_s)-g_s(\pi)
    &=
    \underbrace{
    J_s(\pi,\widehat\nu_s)
    -
    J_s(\pi_{\mathrm{ref}},\widehat\nu_s)
    }_{(I)}
    +
    \underbrace{
    J_s(\pi_{\mathrm{ref}},\widehat\nu_s)
    -
    g_s(\pi_{\mathrm{ref}})
    }_{(II)}
    \nonumber
    +
    \underbrace{
    g_s(\pi_{\mathrm{ref}})
    -
    g_s(\pi)
    }_{(III)} .
    \label{eq:final_stackelberg_decomp}
\end{align}
By Lemma~\ref{lem:final_ar_sensitivity},
\begin{equation}
    (I)
    \le
    \Lambda_{H,\gamma}\delta_\epsilon .
\end{equation}
For term $(II)$, Eq.~\ref{eq:final_stageA_subopt} and Eq.~\ref{eq:final_coverage_gap_scene} give
\begin{equation}
\begin{aligned}
    (II)
    &=
    J_s(\pi_{\mathrm{ref}},\widehat\nu_s)
    -
    g_s(\pi_{\mathrm{ref}})\\
    &\le
    \varepsilon_{A,s}
    +
    \widehat g_s(\pi_{\mathrm{ref}})
    -
    g_s(\pi_{\mathrm{ref}})\\
    &=
    \varepsilon_{A,s}
    +
    \delta_{\mathrm{cov},s}(\pi_{\mathrm{ref}}).
\end{aligned}
\end{equation}
For term $(III)$, observe that $g_s(\pi)$ is the pointwise infimum of functions $J_s(\pi,\nu_s)$ that are all $\Lambda_{H,\gamma}$-Lipschitz in $\pi$ by Lemma~\ref{lem:final_ar_sensitivity}.
The pointwise infimum of uniformly Lipschitz functions is also Lipschitz with the same constant.
Thus
\begin{equation}
    (III)
    \le
    |g_s(\pi_{\mathrm{ref}})-g_s(\pi)|
    \le
    \Lambda_{H,\gamma}\delta_\epsilon .
\end{equation}
Combining the bounds on $(I)$, $(II)$, and $(III)$ proves Eq.~\ref{eq:final_frozen_adv_gap_scene}.
\hfill $\square$

\begin{corollary}[Approximate local robust value guarantee]
\label{cor:final_local_value_guarantee}
Define the ideal local robust value
\begin{equation}
    V_\epsilon^\star
    :=
    \sup_{\pi\in\Pi_\epsilon}
    \mathbb E_{s\sim\mathcal D}
    \big[
        g_s(\pi)
    \big].
\end{equation}
Let Stage~B return a planner $\widehat\pi\in\Pi_\epsilon$ that is an $\varepsilon_B$-approximate best response to the frozen Stage-A adversary:
\begin{equation}
    \mathbb E_s[J_s(\widehat\pi,\widehat\nu_s)]
    \ge
    \sup_{\pi\in\Pi_\epsilon}
    \mathbb E_s[J_s(\pi,\widehat\nu_s)]
    -
    \varepsilon_B .
    \label{eq:final_stageB_subopt}
\end{equation}
Let
\begin{equation}
    \bar\varepsilon_A
    :=
    \mathbb E_s[\varepsilon_{A,s}],
    \qquad
    \bar\delta_{\mathrm{cov}}
    :=
    \mathbb E_s[
        \delta_{\mathrm{cov},s}(\pi_{\mathrm{ref}})
    ].
\end{equation}
Then
\begin{equation}
    \mathbb E_s[g_s(\widehat\pi)]
    \ge
    V_\epsilon^\star
    -
    \varepsilon_B
    -
    \bar\varepsilon_A
    -
    \bar\delta_{\mathrm{cov}}
    -
    2\Lambda_{H,\gamma}\delta_\epsilon .
    \label{eq:final_local_value_guarantee}
\end{equation}
\end{corollary}

\emph{Proof.}
Taking expectation in Proposition~\ref{prop:final_frozen_adv_gap} yields, for every $\pi\in\Pi_\epsilon$,
\begin{equation}
    \mathbb E_s[g_s(\pi)]
    \ge
    \mathbb E_s[J_s(\pi,\widehat\nu_s)]
    -
    \bar\varepsilon_A
    -
    \bar\delta_{\mathrm{cov}}
    -
    2\Lambda_{H,\gamma}\delta_\epsilon .
    \label{eq:final_expected_frozen_bound}
\end{equation}
Applying Eq.~\ref{eq:final_expected_frozen_bound} at $\pi=\widehat\pi$ and using Eq.~\ref{eq:final_stageB_subopt},
\begin{equation}
\begin{aligned}
    \mathbb E_s[g_s(\widehat\pi)]
    &\ge
    \mathbb E_s[J_s(\widehat\pi,\widehat\nu_s)]
    -
    \bar\varepsilon_A
    -
    \bar\delta_{\mathrm{cov}}
    -
    2\Lambda_{H,\gamma}\delta_\epsilon\\
    &\ge
    \sup_{\pi\in\Pi_\epsilon}
    \mathbb E_s[J_s(\pi,\widehat\nu_s)]
    -
    \varepsilon_B
    -
    \bar\varepsilon_A
    -
    \bar\delta_{\mathrm{cov}}
    -
    2\Lambda_{H,\gamma}\delta_\epsilon\\
    &\ge
    \sup_{\pi\in\Pi_\epsilon}
    \mathbb E_s[g_s(\pi)]
    -
    \varepsilon_B
    -
    \bar\varepsilon_A
    -
    \bar\delta_{\mathrm{cov}}
    -
    2\Lambda_{H,\gamma}\delta_\epsilon\\
    &=
    V_\epsilon^\star
    -
    \varepsilon_B
    -
    \bar\varepsilon_A
    -
    \bar\delta_{\mathrm{cov}}
    -
    2\Lambda_{H,\gamma}\delta_\epsilon .
\end{aligned}
\end{equation}
\hfill $\square$

\begin{corollary}[Support-mass interpretation of the coverage gap]
\label{cor:final_support_mass_gap}
Let
\begin{equation}
    U_H
    :=
    u_{\max}\sum_{t=1}^{H}\gamma^{t-1}
\end{equation}
so that $|J_s(\pi,\omega)|\le U_H$ for every rollout execution $\omega$.
Assume that a full adversarial best response at the reference planner exists, and let $\nu_{s,\mathrm{full}}^\star\in\arg\min_{\nu_s\in\mathcal Q_s^\star}J_s(\pi_{\mathrm{ref}},\nu_s)$. If the infimum is not attained, the same statement holds with an additional arbitrarily small optimality slack by taking an approximate minimizer. Define the omitted support mass
\begin{equation}
    \eta_s
    :=
    \nu_{s,\mathrm{full}}^\star
    \big(
        \Omega_s^\star\setminus\widehat{\Omega}_s
    \big).
\end{equation}
If $\eta_s<1$, then
\begin{equation}
    \delta_{\mathrm{cov},s}(\pi_{\mathrm{ref}})
    \le
    2U_H\eta_s .
    \label{eq:final_support_gap_scene}
\end{equation}
Consequently,
\begin{equation}
    \bar\delta_{\mathrm{cov}}
    \le
    2U_H
    \mathbb E_{s\sim\mathcal D}[\eta_s].
    \label{eq:final_support_gap_expected}
\end{equation}
\end{corollary}

\emph{Proof.}
For each scene $s$, let $\widetilde\nu_s$ be the renormalization of $\nu_{s,\mathrm{full}}^\star$ onto $\widehat{\Omega}_s$.
Then $\widetilde\nu_s\in\widehat{\mathcal Q}_s$ and
\begin{equation}
    D_{\mathrm{TV}}
    (
        \widetilde\nu_s,
        \nu_{s,\mathrm{full}}^\star
    )
    =
    \eta_s .
\end{equation}
By the boundedness of the utility,
\begin{equation}
    \left|
        J_s(\pi_{\mathrm{ref}},\widetilde\nu_s)
        -
        J_s(\pi_{\mathrm{ref}},\nu_{s,\mathrm{full}}^\star)
    \right|
    \le
    2U_H\eta_s .
\end{equation}
Because $\widehat g_s(\pi_{\mathrm{ref}})\le J_s(\pi_{\mathrm{ref}},\widetilde\nu_s)$ and $g_s(\pi_{\mathrm{ref}})=J_s(\pi_{\mathrm{ref}},\nu_{s,\mathrm{full}}^\star)$,
\begin{equation}
\begin{aligned}
    \delta_{\mathrm{cov},s}(\pi_{\mathrm{ref}})
    &=
    \widehat g_s(\pi_{\mathrm{ref}})
    -
    g_s(\pi_{\mathrm{ref}})\\
    &\le
    J_s(\pi_{\mathrm{ref}},\widetilde\nu_s)
    -
    J_s(\pi_{\mathrm{ref}},\nu_{s,\mathrm{full}}^\star)
    \le
    2U_H\eta_s .
\end{aligned}
\end{equation}
Averaging over $s$ proves Eq.~\ref{eq:final_support_gap_expected}.
\hfill $\square$

\paragraph{Interpretation.}
Proposition~\ref{prop:final_frozen_adv_gap} and Corollaries~\ref{cor:final_local_value_guarantee}--\ref{cor:final_support_mass_gap} formalize the role of the decoupled Stage-A/B solver.
The frozen AWM objective is optimistic relative to the full worst-case game because it evaluates the planner against one learned structured adversary rather than the entire adversary class.
However, freezing the AWM is a local Stackelberg approximation whose error is controlled by four interpretable quantities: the Stage-A optimization error $\bar\varepsilon_A$, the Stage-B best-response error $\varepsilon_B$, the structured-class coverage gap $\bar\delta_{\mathrm{cov}}$, and the planner trust-region radius $\epsilon$ through the $\mathcal O(\sqrt{\epsilon})$ term $2\Lambda_{H,\gamma}\delta_\epsilon$.
The support-mass bound further explains why the proposal--top1--pair--calibrator design matters: it should cover the low-utility adversarial executions that carry mass under the true worst-case adversary, while avoiding an intractable dense subset search.

\subsection{Scene-Adaptive Pair Admission as a Bayes Gate}
\label{app:theory_calibrator_final}

The selector and pair ranker decide \emph{who} should attack.
The calibrator decides \emph{whether} the second attacker should be admitted.
This distinction can be formalized as a Bayes decision problem.

For a scene $s$, let the candidate pair gain be
\begin{equation}
    G_s^{\mathrm{pair}}
    :=
    U_s(\{i_s^\star\})
    -
    U_s(\{i_s^\star,\widehat j_s\}).
\end{equation}
Positive $G_s^{\mathrm{pair}}$ means that the pair attack reduces planner utility more than the top-1 attack.
Let $\psi_s$ be the scene-level calibrator feature.

\begin{proposition}[Optimal pair-admission rule]
\label{prop:final_pair_admission_rule}
Among all binary admission policies $a(\psi_s)\in\{0,1\}$, the policy maximizing expected admitted pair gain
\begin{equation}
    \mathbb E
    \big[
        a(\psi_s)G_s^{\mathrm{pair}}
    \big]
\end{equation}
is
\begin{equation}
    a^\star(\psi_s)
    =
    \mathbb I
    \left\{
        \mathbb E[
            G_s^{\mathrm{pair}}
            \mid
            \psi_s
        ]
        >0
    \right\}.
    \label{eq:final_bayes_pair_gate}
\end{equation}
If the calibrator is calibrated for the binary label
\begin{equation}
    Y_s
    :=
    \mathbb I\{
        G_s^{\mathrm{pair}}>0
    \},
    \qquad
    p_s
    =
    \mathbb P(Y_s=1\mid\psi_s),
\end{equation}
and if
\begin{equation}
    m_+(\psi_s)
    :=
    \mathbb E[
        G_s^{\mathrm{pair}}
        \mid
        \psi_s,Y_s=1
    ],
    \qquad
    m_-(\psi_s)
    :=
    \mathbb E[
        -G_s^{\mathrm{pair}}
        \mid
        \psi_s,Y_s=0
    ],
\end{equation}
then, on feature values where the conditional quantities are well defined and $m_+(\psi_s)+m_-(\psi_s)>0$, Eq.~\ref{eq:final_bayes_pair_gate} is equivalent to
\begin{equation}
    p_s
    >
    \frac{
        m_-(\psi_s)
    }{
        m_+(\psi_s)+m_-(\psi_s)
    }.
    \label{eq:final_pair_threshold_general}
\end{equation}
When the conditional magnitudes are approximately stable over the deployment distribution, this reduces to a scalar threshold rule $p_s>\tau_{\mathrm{cal}}^\star$, which is the tractable rule used by our scene-adaptive calibrator.
\end{proposition}

\emph{Proof.}
For any measurable admission policy $a(\psi_s)$,
\begin{align}
    \mathbb E[
        a(\psi_s)G_s^{\mathrm{pair}}
    ]
    &=
    \mathbb E
    \left[
        a(\psi_s)
        \mathbb E[
            G_s^{\mathrm{pair}}
            \mid
            \psi_s
        ]
    \right].
\end{align}
Since $a(\psi_s)$ is binary, the pointwise maximizer admits the pair exactly when the conditional expected gain is positive, proving Eq.~\ref{eq:final_bayes_pair_gate}.
Next,
\begin{equation}
\begin{aligned}
    \mathbb E[
        G_s^{\mathrm{pair}}
        \mid
        \psi_s
    ]
    &=
    p_s
    \mathbb E[
        G_s^{\mathrm{pair}}
        \mid
        \psi_s,Y_s=1
    ]
    +
    (1-p_s)
    \mathbb E[
        G_s^{\mathrm{pair}}
        \mid
        \psi_s,Y_s=0
    ]\\
    &=
    p_s m_+(\psi_s)
    -
    (1-p_s)m_-(\psi_s).
\end{aligned}
\end{equation}
This quantity is positive if and only if Eq.~\ref{eq:final_pair_threshold_general} holds.
\hfill $\square$

\paragraph{Interpretation.}
The calibrator is not merely a heuristic confidence score.
It is a learned approximation to the Bayes gate that decides whether the additional attack dimension has positive expected value in the current scene.
This explains why adaptive $K>1$ is preferable to fixed $K=2$: pair attacks should be admitted only when their conditional expected marginal gain is positive.

\subsection{Role-Switched Counterfactual Credit as a Difference-Reward Estimator}
\label{app:theory_credit_final}

We next justify the role-switched counterfactual rewards used to train the AWM.
For this subsection, fix a scene $s$ and a factual coalition $A=A_s^{\mathrm{fact}}$.
Let $U_s(A)$ denote the planner's utility under the factual coalition, and define the adversarial team reward
\begin{equation}
    R_s(A)
    :=
    -U_s(A).
\end{equation}
A larger $R_s(A)$ means a stronger attack.

For an active attacker $i\in A$, let $Y_i=(z_{i,1},\ldots,z_{i,H})$ denote its sampled factual attack-token trajectory, and let
\begin{equation}
    \Psi_i
    :=
    \nabla_\phi^{(i)}
    \log p_\phi(Y_i\mid \mathcal H_i,\rho_i=\texttt{atk})
    =
    \sum_{t=1}^{H}
    \nabla_\phi
    \log
    \pi_{\mathrm{awm},\phi}
    \big(
        z_{i,t}
        \mid
        x_t,\rho_i=\texttt{atk}
    \big)
    \label{eq:final_attacker_score}
\end{equation}
be the score-function term associated with attacker $i$; $\mathcal H_i$ denotes the factual pre-sampling history needed to define the conditional trajectory law. When parameters are shared across attackers, the full attacker-side policy-gradient estimator sums Eq.~\ref{eq:final_attacker_score} over active attackers.

Let $\mathrm{bg}(A)$ denote the counterfactual rollout obtained by switching all active attackers in $A$ to the background role without removing the agents from the scene. Define the admissible baselines
\begin{equation}
    B_i^{\mathrm{loo}}
    :=
    -U_s(A\setminus\{i\}),
    \qquad
    B^{\mathrm{team}}
    :=
    -U_s(\mathrm{bg}(A)),
\end{equation}
where both quantities are evaluated by the action-independent counterfactual replay described above. Then
\begin{equation}
    \Delta_i^{\mathrm{loo}}
    =
    U_s(A\setminus\{i\})-U_s(A)
    =
    R_s(A)-B_i^{\mathrm{loo}},
\end{equation}
and
\begin{equation}
    \Delta_s^{\mathrm{team}}
    =
    U_s(\mathrm{bg}(A))-U_s(A)
    =
    R_s(A)-B^{\mathrm{team}}.
\end{equation}

\begin{proposition}[Hybrid counterfactual reward is gradient-aligned]
\label{prop:final_hybrid_credit}
Assume the counterfactual utilities used to form $B_i^{\mathrm{loo}}$ and $B^{\mathrm{team}}$ are evaluated with stop-gradient, are $\mathcal F_i^{\mathrm{cf}}$-measurable, and are conditionally independent of the factual attacker trajectory $Y_i$ given $\mathcal F_i^{\mathrm{cf}}$.
For weights $\omega_{\mathrm{loo}},\omega_{\mathrm{team}}\ge 0$, define
\begin{equation}
    \Delta_i^{\mathrm{hyb}}
    :=
    \omega_{\mathrm{loo}}\Delta_i^{\mathrm{loo}}
    +
    \omega_{\mathrm{team}}\Delta_s^{\mathrm{team}},
    \qquad
    W:=\omega_{\mathrm{loo}}+\omega_{\mathrm{team}}.
\end{equation}
Then
\begin{equation}
    \mathbb E[
        \Psi_i
        \Delta_i^{\mathrm{hyb}}
    ]
    =
    W\,
    \nabla_\phi^{(i)}
    \mathbb E[
        R_s(A)
    ],
    \label{eq:final_hybrid_credit_unbiased}
\end{equation}
where $\nabla_\phi^{(i)}$ denotes the policy-gradient contribution associated with attacker $i$.
In particular, if $W>0$, the hybrid counterfactual reward has the same ascent direction as the adversarial team objective up to the positive scalar $W$; if $W=1$, it is an unbiased estimator of that component.
\end{proposition}

\emph{Proof.}
By the score-function identity applied to the factual attack trajectory,
\begin{equation}
    \nabla_\phi^{(i)}
    \mathbb E[
        R_s(A)
    ]
    =
    \mathbb E[
        \Psi_i R_s(A)
    ].
    \label{eq:final_score_identity_credit}
\end{equation}
It remains to show that the counterfactual baselines vanish after multiplication by $\Psi_i$.
Because $B_i^{\mathrm{loo}}$ is $\mathcal F_i^{\mathrm{cf}}$-measurable and conditionally independent of $Y_i$ given $\mathcal F_i^{\mathrm{cf}}$,
\begin{equation}
\begin{aligned}
    \mathbb E[
        \Psi_i B_i^{\mathrm{loo}}
    ]
    &=
    \mathbb E
    \big[
        B_i^{\mathrm{loo}}
        \mathbb E[
            \Psi_i
            \mid
            \mathcal F_i^{\mathrm{cf}}
        ]
    \big].
\end{aligned}
\end{equation}
For the conditional score,
\begin{equation}
\begin{aligned}
    \mathbb E[
        \Psi_i
        \mid
        \mathcal F_i^{\mathrm{cf}}
    ]
    &=
    \sum_{y_i}
    p_\phi(y_i\mid \mathcal F_i^{\mathrm{cf}})
    \nabla_\phi
    \log p_\phi(y_i\mid \mathcal F_i^{\mathrm{cf}})\\
    &=
    \sum_{y_i}
    \nabla_\phi
    p_\phi(y_i\mid \mathcal F_i^{\mathrm{cf}})
    =
    \nabla_\phi
    \sum_{y_i}
    p_\phi(y_i\mid \mathcal F_i^{\mathrm{cf}})
    =0.
\end{aligned}
\end{equation}
Thus $\mathbb E[\Psi_i B_i^{\mathrm{loo}}]=0$. The same argument gives $\mathbb E[\Psi_i B^{\mathrm{team}}]=0$. Therefore,
\begin{equation}
\begin{aligned}
    \mathbb E[
        \Psi_i \Delta_i^{\mathrm{hyb}}
    ]
    &=
    \omega_{\mathrm{loo}}
    \mathbb E[
        \Psi_i(R_s(A)-B_i^{\mathrm{loo}})
    ]
    +
    \omega_{\mathrm{team}}
    \mathbb E[
        \Psi_i(R_s(A)-B^{\mathrm{team}})
    ]\\
    &=
    (\omega_{\mathrm{loo}}+\omega_{\mathrm{team}})
    \mathbb E[
        \Psi_i R_s(A)
    ]
    =
    W\,
    \nabla_\phi^{(i)}
    \mathbb E[
        R_s(A)
    ].
\end{aligned}
\end{equation}
\hfill $\square$

\begin{lemma}[Variance-optimal action-independent baseline]
\label{lem:final_variance_optimal_baseline}
Let $\mathcal F_{-i}$ be any sigma-field that excludes the factual attacker sample $Y_i$ and satisfies $\mathbb E[\Psi_i\mid\mathcal F_{-i}]=0$. Consider any $\mathcal F_{-i}$-measurable baseline $b_i(\mathcal F_{-i})$.
Define the score-weighted second moment
\begin{equation}
    \mathcal V(b_i)
    :=
    \mathbb E
    \left[
        \|\Psi_i\|_2^2
        \big(
            R_s(A)-b_i(\mathcal F_{-i})
        \big)^2
    \right].
\end{equation}
On events where $\mathbb E[\|\Psi_i\|_2^2\mid\mathcal F_{-i}]>0$, the unique minimizer among all $\mathcal F_{-i}$-measurable baselines is
\begin{equation}
    b_i^\star(\mathcal F_{-i})
    =
    \frac{
        \mathbb E[
            \|\Psi_i\|_2^2 R_s(A)
            \mid
            \mathcal F_{-i}
        ]
    }{
        \mathbb E[
            \|\Psi_i\|_2^2
            \mid
            \mathcal F_{-i}
        ]
    }.
    \label{eq:final_optimal_baseline_weighted}
\end{equation}
On events where this conditional second moment is zero, the choice of baseline is immaterial. Moreover, for any other baseline $b_i$,
\begin{equation}
    \mathcal V(b_i)
    =
    \mathcal V(b_i^\star)
    +
    \mathbb E
    \left[
        \mathbb E[
            \|\Psi_i\|_2^2
            \mid
            \mathcal F_{-i}
        ]
        \big(
            b_i(\mathcal F_{-i})
            -
            b_i^\star(\mathcal F_{-i})
        \big)^2
    \right].
    \label{eq:final_baseline_decomposition}
\end{equation}
If $\mathbb E[\|\Psi_i\|_2^2\mid\mathcal F_{-i},R_s(A)]
=
\mathbb E[\|\Psi_i\|_2^2\mid\mathcal F_{-i}]$, then
\begin{equation}
    b_i^\star(\mathcal F_{-i})
    =
    \mathbb E[
        R_s(A)
        \mid
        \mathcal F_{-i}
    ].
    \label{eq:final_conditional_expectation_baseline}
\end{equation}
\end{lemma}

\emph{Proof.}
Condition on $\mathcal F_{-i}$.
For any scalar value $b_i=b_i(\mathcal F_{-i})$, the conditional objective is
\begin{equation}
\begin{aligned}
    \mathcal V(b_i\mid\mathcal F_{-i})
    &=
    \mathbb E[
        \|\Psi_i\|_2^2 R_s(A)^2
        \mid
        \mathcal F_{-i}
    ]
    -
    2b_i
    \mathbb E[
        \|\Psi_i\|_2^2 R_s(A)
        \mid
        \mathcal F_{-i}
    ]\\
    &\quad
    +
    b_i^2
    \mathbb E[
        \|\Psi_i\|_2^2
        \mid
        \mathcal F_{-i}
    ].
\end{aligned}
\end{equation}
This is a quadratic in $b_i$.
Differentiating and setting the derivative to zero gives Eq.~\ref{eq:final_optimal_baseline_weighted}.
Completing the square yields Eq.~\ref{eq:final_baseline_decomposition}.
If the score magnitude is conditionally independent of the reward given $\mathcal F_{-i}$, the numerator in Eq.~\ref{eq:final_optimal_baseline_weighted} factorizes, giving Eq.~\ref{eq:final_conditional_expectation_baseline}.
\hfill $\square$

\begin{remark}[Pairwise Shapley reconstruction]
\label{rem:final_pairwise_shapley}
When the admitted coalition has size two, $A=\{i,j\}$, define the cooperative game
\begin{equation}
    v(B)
    :=
    U_s(\mathrm{bg}(A))
    -
    U_s(B),
    \qquad
    B\subseteq A,
\end{equation}
so that $v(\emptyset)=0$ and $v(A)=\Delta_s^{\mathrm{team}}$.
The Shapley values of the two attackers are
\begin{equation}
\begin{aligned}
    \phi_i
    &=
    \frac12
    \big(
        v(\{i\})-v(\emptyset)
    \big)
    +
    \frac12
    \big(
        v(\{i,j\})-v(\{j\})
    \big),\\
    \phi_j
    &=
    \frac12
    \big(
        v(\{j\})-v(\emptyset)
    \big)
    +
    \frac12
    \big(
        v(\{i,j\})-v(\{i\})
    \big).
\end{aligned}
\end{equation}
Using
\begin{equation}
    \Delta_i^{\mathrm{loo}}
    =
    v(\{i,j\})-v(\{j\}),
    \qquad
    \Delta_j^{\mathrm{loo}}
    =
    v(\{i,j\})-v(\{i\}),
\end{equation}
we obtain
\begin{equation}
    \phi_i
    =
    \frac12
    \big(
        \Delta_s^{\mathrm{team}}
        +
        \Delta_i^{\mathrm{loo}}
        -
        \Delta_j^{\mathrm{loo}}
    \big),
    \qquad
    \phi_j
    =
    \frac12
    \big(
        \Delta_s^{\mathrm{team}}
        +
        \Delta_j^{\mathrm{loo}}
        -
        \Delta_i^{\mathrm{loo}}
    \big).
    \label{eq:final_pair_shapley}
\end{equation}
Thus, in the $K_{\max}=2$ regime used by our method, the team and leave-one-out counterfactuals span the exact fair-credit decomposition of the corresponding pairwise cooperative game.
\end{remark}

\paragraph{Interpretation.}
Proposition~\ref{prop:final_hybrid_credit} and Lemma~\ref{lem:final_variance_optimal_baseline} show that role-switched counterfactual rewards are not arbitrary heuristics.
They are difference-reward control variates for score-function optimization \cite{greensmith2004variance,foerster2018counterfactual} under the action-independence condition stated above.
Because the factual and counterfactual rollouts share the same model, token vocabulary, transition operator, and scene context, the baselines remove nuisance scene variation; because the counterfactual utility is detached and independent of the factual attacker sample, the attacker-side score-function direction is preserved.
The team term further preserves non-additive coalition information, which is essential for learning cooperative pair attacks.

\subsection{Regret as a Reference-Anchored Control Variate}
\label{app:theory_regret_final}
We now analyze the planner-side robust objective.
Stage~B freezes the learned AWM and optimizes the planner against the induced top1/pair branch distribution. Throughout this subsection, the branch identity rule and the calibrated probability $p_s$ are treated as frozen with respect to the planner parameters.
Let
\begin{equation}
    J_s^b(\pi),
    \qquad
    b\in\{\mathrm{top1},\mathrm{pair}\},
\end{equation}
denote the planner utility in scene $s$ under a frozen adversarial branch $b$.
We define the reference-relative adversarial regret as
\begin{equation}
    \Delta_s^b(\theta)
    :=
    J_s^b(\pi_{\mathrm{ref}})
    -
    J_s^b(\pi_\theta).
    \label{eq:final_branch_regret}
\end{equation}
A positive value means that the current planner underperforms the reference planner under the same attack.

\begin{proposition}\textbf{\upshape(Regret preserves branchwise gradient direction and reduces scene difficulty)}\par\noindent
\label{prop:final_regret_control_variate}
For any branch $b$,
\begin{equation}
    \nabla_\theta
    \mathbb E_s[
        \Delta_s^b(\theta)
    ]
    =
    -
    \nabla_\theta
    \mathbb E_s[
        J_s^b(\pi_\theta)
    ].
    \label{eq:final_regret_gradient_equiv}
\end{equation}
Thus, minimizing regret is branchwise equivalent to maximizing utility.
Moreover,
\begin{align}
    \mathrm{Var}_s[
        \Delta_s^b(\theta)
    ]
    &=
    \mathrm{Var}_s[
        J_s^b(\pi_\theta)
    ]
    +
    \mathrm{Var}_s[
        J_s^b(\pi_{\mathrm{ref}})
    ]
    \nonumber\\
    &\quad
    -
    2\mathrm{Cov}_s
    \big(
        J_s^b(\pi_\theta),
        J_s^b(\pi_{\mathrm{ref}})
    \big).
    \label{eq:final_regret_variance_identity}
\end{align}
Consequently,
\begin{equation}
    \mathrm{Var}_s[
        \Delta_s^b(\theta)
    ]
    \le
    \mathrm{Var}_s[
        J_s^b(\pi_\theta)
    ]
\end{equation}
whenever
\begin{equation}
    2\mathrm{Cov}_s
    \big(
        J_s^b(\pi_\theta),
        J_s^b(\pi_{\mathrm{ref}})
    \big)
    \ge
    \mathrm{Var}_s[
        J_s^b(\pi_{\mathrm{ref}})
    ].
    \label{eq:final_regret_variance_condition}
\end{equation}
Finally, if $\pi_\theta\in\Pi_\epsilon$, then
\begin{equation}
    |\Delta_s^b(\theta)|
    \le
    \Lambda_{H,\gamma}\delta_\epsilon,
    \qquad
    \mathrm{Var}_s[
        \Delta_s^b(\theta)
    ]
    \le
    \Lambda_{H,\gamma}^2\delta_\epsilon^2 .
    \label{eq:final_regret_trust_bound}
\end{equation}
\end{proposition}

\emph{Proof.}
Since $\pi_{\mathrm{ref}}$ is frozen,
\begin{equation}
    \nabla_\theta
    \Delta_s^b(\theta)
    =
    -
    \nabla_\theta
    J_s^b(\pi_\theta),
\end{equation}
which proves Eq.~\ref{eq:final_regret_gradient_equiv} after taking expectation over scenes.
Eq.~\ref{eq:final_regret_variance_identity} follows from
\begin{equation}
    \mathrm{Var}(Y-X)
    =
    \mathrm{Var}(Y)
    +
    \mathrm{Var}(X)
    -
    2\mathrm{Cov}(X,Y),
\end{equation}
with $Y=J_s^b(\pi_{\mathrm{ref}})$ and $X=J_s^b(\pi_\theta)$.
The sufficient condition in Eq.~\ref{eq:final_regret_variance_condition} follows immediately.

For Eq.~\ref{eq:final_regret_trust_bound}, apply Lemma~\ref{lem:final_ar_sensitivity} to the fixed branch environment:
\begin{equation}
    |\Delta_s^b(\theta)|
    =
    |
        J_s^b(\pi_{\mathrm{ref}})
        -
        J_s^b(\pi_\theta)
    |
    \le
    \Lambda_{H,\gamma}
    \delta(\pi_\theta,\pi_{\mathrm{ref}})
    \le
    \Lambda_{H,\gamma}\delta_\epsilon .
\end{equation}
A random variable bounded in absolute value by $c$ has variance at most $c^2$, yielding the variance bound.
\hfill $\square$

\begin{remark}[Generalized control-variate coefficient]
\label{rem:final_generalized_beta}
More generally, define
\begin{equation}
    \Delta_{s,\beta}^b(\theta)
    :=
    \beta J_s^b(\pi_{\mathrm{ref}})
    -
    J_s^b(\pi_\theta).
\end{equation}
For every $\beta\in\mathbb R$,
\begin{equation}
    \nabla_\theta
    \mathbb E_s[
        \Delta_{s,\beta}^b(\theta)
    ]
    =
    -
    \nabla_\theta
    \mathbb E_s[
        J_s^b(\pi_\theta)
    ].
\end{equation}
When $\mathrm{Var}_s(J_s^b(\pi_{\mathrm{ref}}))>0$, the variance-minimizing coefficient is
\begin{equation}
    \beta^\star
    =
    \frac{
        \mathrm{Cov}_s
        (
            J_s^b(\pi_\theta),
            J_s^b(\pi_{\mathrm{ref}})
        )
    }{
        \mathrm{Var}_s
        (
            J_s^b(\pi_{\mathrm{ref}})
        )
    }.
    \label{eq:final_beta_star}
\end{equation}
Our implementation uses the simple unit coefficient $\beta=1$, which is natural when $\pi_\theta$ remains close to $\pi_{\mathrm{ref}}$ under the trust region and the two branch utilities are highly correlated.
\end{remark}

\paragraph{Interpretation.}
Regret does not change the branchwise optimization direction.
It changes the cross-scene random variable to which the tail-risk operator is applied.
In autonomous driving, absolute utility can be dominated by scene-intrinsic difficulty.
If
\begin{equation}
    J_s^b(\pi)
    =
    d_s^b
    +
    \widetilde J_s^b(\pi),
    \label{eq:final_additive_difficulty}
\end{equation}
where $d_s^b$ is independent of the planner, then
\begin{equation}
    \Delta_s^b(\theta)
    =
    \widetilde J_s^b(\pi_{\mathrm{ref}})
    -
    \widetilde J_s^b(\pi_\theta),
\end{equation}
so the intrinsic difficulty term cancels exactly.
This explains why CVaR should be applied to regret rather than absolute utility: absolute-utility CVaR can overemphasize intrinsically hard scenes, whereas regret-CVaR focuses on planner-induced degradation relative to the same reference policy.

\subsection{Closed-Form CVaR on the AWM-Induced Two-Point Risk Distribution}
\label{app:theory_cvar_final}
For each scene, the frozen AWM induces a two-point adversarial branch distribution:
\begin{equation}
    \mathcal D_s^\theta
    =
    \left\{
        \big(
            \Delta_s^{\mathrm{top1}}(\theta),
            1-p_s
        \big),
        \big(
            \Delta_s^{\mathrm{pair}}(\theta),
            p_s
        \big)
    \right\},
    \label{eq:final_two_point_regret_distribution}
\end{equation}
where $p_s$ is the calibrated pair-admission probability and is frozen during the planner update.
We use upper-tail CVaR with tail mass $\tau\in(0,1]$:
\begin{equation}
    \mathrm{CVaR}_\tau(Z)
    :=
    \min_{\eta\in\mathbb R}
    \left\{
        \eta
        +
        \frac{1}{\tau}
        \mathbb E[
            (Z-\eta)_+
        ]
    \right\}.
    \label{eq:final_cvar_definition}
\end{equation}
This is the standard Rockafellar--Uryasev variational form for upper-tail risk \cite{rockafellar2000optimization}, written with tail mass $\tau$.

\begin{proposition}[Closed-form CVaR weights and rare-risk amplification]
\label{prop:final_cvar_weights}
Let the two branch regrets be sorted as
\begin{equation}
    \Delta_s^{b_{(1)}}(\theta)
    \ge
    \Delta_s^{b_{(2)}}(\theta),
\end{equation}
and let $\mu_s^{b_{(1)}}\in\{p_s,1-p_s\}$ be the probability mass of the worse branch.
Then
\begin{equation}
    \mathrm{CVaR}_\tau(\mathcal D_s^\theta)
    =
    w_s^{b_{(1)}}(\tau)
    \Delta_s^{b_{(1)}}(\theta)
    +
    w_s^{b_{(2)}}(\tau)
    \Delta_s^{b_{(2)}}(\theta),
    \label{eq:final_cvar_closed_form}
\end{equation}
where
\begin{equation}
    w_s^{b_{(1)}}(\tau)
    =
    \begin{cases}
        1,
        &
        \mu_s^{b_{(1)}}\ge \tau,\\[4pt]
        \dfrac{\mu_s^{b_{(1)}}}{\tau},
        &
        \mu_s^{b_{(1)}}<\tau,
    \end{cases}
    \qquad
    w_s^{b_{(2)}}(\tau)
    =
    1-w_s^{b_{(1)}}(\tau).
    \label{eq:final_cvar_weights}
\end{equation}
Away from the tie set
$\Delta_s^{\mathrm{top1}}=\Delta_s^{\mathrm{pair}}$, the subgradient is
\begin{equation}
    \nabla_\theta
    \mathrm{CVaR}_\tau(\mathcal D_s^\theta)
    =
    w_s^{\mathrm{top1}}(\tau)
    \nabla_\theta
    \Delta_s^{\mathrm{top1}}(\theta)
    +
    w_s^{\mathrm{pair}}(\tau)
    \nabla_\theta
    \Delta_s^{\mathrm{pair}}(\theta).
    \label{eq:final_cvar_gradient}
\end{equation}
If the pair branch is the worse branch and $p_s<\tau$, then its CVaR coefficient is $p_s/\tau$.
Compared with ordinary expected regret, where the coefficient is $p_s$, the pair branch receives the amplification factor
\begin{equation}
    \frac{p_s/\tau}{p_s}
    =
    \frac{1}{\tau}.
    \label{eq:final_pair_amplification}
\end{equation}
\end{proposition}

\emph{Proof.}
Let $z_{(1)}\ge z_{(2)}$ denote the two sorted regret values and let $\mu$ be the probability mass of $z_{(1)}$.
For $\eta\in\mathbb R$, define
\begin{equation}
    f(\eta)
    :=
    \eta
    +
    \frac{1}{\tau}
    \left[
        \mu(z_{(1)}-\eta)_+
        +
        (1-\mu)(z_{(2)}-\eta)_+
    \right].
\end{equation}
We analyze three regions.

If $\eta\ge z_{(1)}$, both hinge terms vanish and $f(\eta)=\eta$, minimized at $\eta=z_{(1)}$ with value $z_{(1)}$.

If $z_{(2)}\le\eta\le z_{(1)}$, only the first hinge is active:
\begin{equation}
    f(\eta)
    =
    \eta
    +
    \frac{\mu}{\tau}
    (z_{(1)}-\eta)
    =
    \frac{\mu}{\tau}z_{(1)}
    +
    \left(
        1-\frac{\mu}{\tau}
    \right)\eta .
\end{equation}
If $\mu\ge\tau$, the slope is non-positive, so the minimum over this interval is attained at $\eta=z_{(1)}$, giving value $z_{(1)}$.
If $\mu<\tau$, the slope is positive, so the minimum is attained at $\eta=z_{(2)}$, giving
\begin{equation}
    \frac{\mu}{\tau}z_{(1)}
    +
    \left(
        1-\frac{\mu}{\tau}
    \right)z_{(2)}.
\end{equation}

If $\eta\le z_{(2)}$, both hinge terms are active:
\begin{equation}
\begin{aligned}
    f(\eta)
    &=
    \eta
    +
    \frac{1}{\tau}
    \left[
        \mu(z_{(1)}-\eta)
        +
        (1-\mu)(z_{(2)}-\eta)
    \right]\\
    &=
    \frac{\mu}{\tau}z_{(1)}
    +
    \frac{1-\mu}{\tau}z_{(2)}
    +
    \left(
        1-\frac{1}{\tau}
    \right)\eta .
\end{aligned}
\end{equation}
Since $\tau\in(0,1]$, the slope is non-positive, so the minimum over this region is attained at the largest feasible $\eta$, namely $\eta=z_{(2)}$.
This gives the same value as the previous case when $\mu<\tau$.
Combining the cases proves Eq.~\ref{eq:final_cvar_closed_form} and Eq.~\ref{eq:final_cvar_weights}.
The subgradient expression follows by differentiating the piecewise-linear closed form away from branch ties.
The amplification claim follows by setting $\mu=p_s$ when the pair branch is worse and $p_s<\tau$.
\hfill $\square$

\paragraph{Interpretation.}
The calibrated Pair attacks are intentionally sparse: in most scenes, the calibrated probability $p_s$ may be small.
Expected regret would multiply the pair-branch gradient by $p_s$, causing rare but severe cooperative failures to be under-optimized.
CVaR prevents this dilution.
If the rare pair branch is currently the worst branch, it is amplified by exactly $1/\tau$ until it either no longer dominates the tail or its probability mass saturates the entire tail.
This is why the planner-side objective is regret-sensitive and CVaR-oriented rather than a simple expectation over AWM samples.

In the implementation, the CVaR weights are treated with stop-gradient. On every region where the branch ordering is fixed, this coincides with one valid subgradient of the piecewise-linear CVaR objective in Eq.~\ref{eq:final_cvar_gradient}. At branch ties, the CVaR subdifferential is set-valued; any convex combination of the tied-branch gradients is valid.

\subsection{Nominal Retention as a Local Performance Floor}
\label{app:theory_nominal_final}

We finally analyze the nominal-retention and safety-constraint mechanisms.
Let
\begin{equation}
    J_s^{\mathrm{norm}}(\pi)
    :=
    U_s(\pi,\pi_{\mathrm{pred}})
\end{equation}
be the scene-level nominal utility, and define
\begin{equation}
    J^{\mathrm{norm}}(\pi)
    :=
    \mathbb E_{s\sim\mathcal D}
    [
        J_s^{\mathrm{norm}}(\pi)
    ].
\end{equation}
Recall the normal utility margin
\begin{equation}
    m_s^{\mathrm{norm}}(\theta)
    :=
    J_s^{\mathrm{norm}}(\pi_\theta)
    -
    J_s^{\mathrm{norm}}(\pi_{\mathrm{ref}})
\end{equation}
and the asymmetric penalty
\begin{equation}
    \psi_s^{\mathrm{norm}}(\theta)
    :=
    \left[
        -\epsilon_{\mathrm{norm}}
        -
        m_s^{\mathrm{norm}}(\theta)
    \right]_+^2 .
    \label{eq:final_nominal_margin_penalty}
\end{equation}

\begin{proposition}[Nominal performance floor]
\label{prop:final_nominal_floor}
For every $\pi_\theta\in\Pi_\epsilon$ and every scene $s$,
\begin{equation}
    \big|
        J_s^{\mathrm{norm}}(\pi_\theta)
        -
        J_s^{\mathrm{norm}}(\pi_{\mathrm{ref}})
    \big|
    \le
    \Lambda_{H,\gamma}\delta_\epsilon .
    \label{eq:final_nominal_trust_bound}
\end{equation}
Moreover, the asymmetric margin penalty implies the pointwise bound
\begin{equation}
    m_s^{\mathrm{norm}}(\theta)
    \ge
    -\epsilon_{\mathrm{norm}}
    -
    \sqrt{
        \psi_s^{\mathrm{norm}}(\theta)
    }.
    \label{eq:final_margin_pointwise}
\end{equation}
Taking expectation over scenes gives
\begin{equation}
    J^{\mathrm{norm}}(\pi_\theta)
    \ge
    J^{\mathrm{norm}}(\pi_{\mathrm{ref}})
    -
    \epsilon_{\mathrm{norm}}
    -
    \sqrt{
        \mathbb E_s[
            \psi_s^{\mathrm{norm}}(\theta)
        ]
    }.
    \label{eq:final_nominal_margin_bound}
\end{equation}
Therefore,
\begin{equation}
    J^{\mathrm{norm}}(\pi_\theta)
    \ge
    J^{\mathrm{norm}}(\pi_{\mathrm{ref}})
    -
    \min
    \left\{
        \Lambda_{H,\gamma}\delta_\epsilon,
        \;
        \epsilon_{\mathrm{norm}}
        +
        \sqrt{
            \mathbb E_s[
                \psi_s^{\mathrm{norm}}(\theta)
            ]
        }
    \right\}.
    \label{eq:final_nominal_combined_bound}
\end{equation}
In particular, if $\psi_s^{\mathrm{norm}}(\theta)=0$ almost surely, then
\begin{equation}
    J_s^{\mathrm{norm}}(\pi_\theta)
    \ge
    J_s^{\mathrm{norm}}(\pi_{\mathrm{ref}})
    -
    \epsilon_{\mathrm{norm}}
    \qquad
    \text{for almost every scene }s.
    \label{eq:final_zero_margin_bound}
\end{equation}
\end{proposition}

\emph{Proof.}
Eq.~\ref{eq:final_nominal_trust_bound} is Lemma~\ref{lem:final_ar_sensitivity} applied to the nominal environment $\pi_{\mathrm{pred}}$.
For the margin bound, define
\begin{equation}
    a_s(\theta)
    :=
    \left[
        -\epsilon_{\mathrm{norm}}
        -
        m_s^{\mathrm{norm}}(\theta)
    \right]_+
    =
    \sqrt{
        \psi_s^{\mathrm{norm}}(\theta)
    }.
\end{equation}
Since $x\le[x]_+$ for all real $x$,
\begin{equation}
    -\epsilon_{\mathrm{norm}}
    -
    m_s^{\mathrm{norm}}(\theta)
    \le
    a_s(\theta),
\end{equation}
which rearranges to Eq.~\ref{eq:final_margin_pointwise}.
Taking expectations,
\begin{align}
    J^{\mathrm{norm}}(\pi_\theta)
    -
    J^{\mathrm{norm}}(\pi_{\mathrm{ref}})
    &=
    \mathbb E_s[
        m_s^{\mathrm{norm}}(\theta)
    ]\\
    &\ge
    -\epsilon_{\mathrm{norm}}
    -
    \mathbb E_s[
        a_s(\theta)
    ].
\end{align}
By Jensen's inequality,
\begin{equation}
    \mathbb E_s[
        a_s(\theta)
    ]
    \le
    \sqrt{
        \mathbb E_s[
            \psi_s^{\mathrm{norm}}(\theta)
        ]
    },
\end{equation}
which proves Eq.~\ref{eq:final_nominal_margin_bound}.
Both Eq.~\ref{eq:final_nominal_trust_bound} and Eq.~\ref{eq:final_nominal_margin_bound} are valid lower bounds on expected nominal utility, so taking the tighter one gives Eq.~\ref{eq:final_nominal_combined_bound}.
If $\psi_s^{\mathrm{norm}}(\theta)=0$, Eq.~\ref{eq:final_margin_pointwise} reduces to Eq.~\ref{eq:final_zero_margin_bound}.
\hfill $\square$

\paragraph{Interpretation.}
Proposition~\ref{prop:final_nominal_floor} is the finite-horizon autoregressive analogue of trust-region and conservative policy-improvement arguments \cite{kakade2002approximately,schulman2015trust}.
The KL anchor limits how far the planner can move from the reference behavior, while the asymmetric margin penalty directly controls degradation in normal traffic.
Thus robustness is not obtained by unconstrained overfitting to adversarial samples; it is obtained by a controlled local update with an explicit nominal-performance floor.

\subsection{Safety Dual Variables as a Projected Max-Violation Surrogate}
\label{app:theory_dual_final}

Finally, we analyze the safety dual update in Section~\ref{subsubsec:dual_update}.
For each safety channel $m$, define the branch-mixed violation gap
\begin{equation}
    h_m(\theta)
    :=
    \mathbb E_{s\sim\mathcal D}
    [
        \bar v_{s,m}
    ]
    -
    \kappa_m,
    \qquad
    \bar v_{s,m}
    :=
    (1-p_s)v_{s,m}^{\mathrm{top1}}
    +
    p_s v_{s,m}^{\mathrm{pair}}.
    \label{eq:final_violation_gap}
\end{equation}
In a classical Lagrangian relaxation, one would use nonnegative multipliers $\lambda_m\ge0$ and maximize $\sum_m\lambda_m h_m(\theta)$.
Our implementation instead uses a bounded softmax parameterization
\begin{equation}
    [
        \lambda_R,
        \lambda_1,\ldots,\lambda_M
    ]
    =
    \operatorname{softmax}
    (
        [
            \alpha_R,
            \zeta_1,\ldots,\zeta_M
        ]
    ),
\end{equation}
and maximizes
\begin{equation}
    \mathcal L_{\mathrm{dual}}(\zeta)
    =
    \sum_{m=1}^{M}
    \lambda_m(\zeta)h_m(\theta).
    \label{eq:final_softmax_dual}
\end{equation}
For analysis, define $h_R:=0$ for the progress channel.

\begin{proposition}[Softmax safety multipliers as a bounded max-violation surrogate]
\label{prop:final_softmax_dual}
For fixed planner parameters $\theta$, the gradient of Eq.~\ref{eq:final_softmax_dual} with respect to the safety logit $\zeta_m$ is
\begin{equation}
    \frac{\partial \mathcal L_{\mathrm{dual}}}{\partial \zeta_m}
    =
    \lambda_m
    \left(
        h_m
        -
        \sum_{\ell=1}^{M}
        \lambda_\ell h_\ell
    \right).
    \label{eq:final_dual_gradient}
\end{equation}
Thus dual ascent increases the relative weight of safety channels whose violation gap exceeds the current softmax-weighted average, and decreases the relative weight of channels below that average.
Moreover,
\begin{equation}
    \sup_{\zeta\in\mathbb R^M}
    \mathcal L_{\mathrm{dual}}(\zeta)
    =
    \max
    \left\{
        0,
        \max_{m\in\{1,\ldots,M\}}
        h_m(\theta)
    \right\}.
    \label{eq:final_dual_supremum}
\end{equation}
\end{proposition}

\emph{Proof.}
The softmax derivative gives
\begin{equation}
    \frac{\partial \lambda_\ell}{\partial \zeta_m}
    =
    \lambda_\ell
    (
        \mathbb I\{\ell=m\}
        -
        \lambda_m
    ),
    \qquad
    \ell\in\{1,\ldots,M\}.
\end{equation}
Therefore,
\begin{equation}
\begin{aligned}
    \frac{\partial \mathcal L_{\mathrm{dual}}}{\partial \zeta_m}
    &=
    \sum_{\ell=1}^{M}
    h_\ell
    \frac{\partial \lambda_\ell}{\partial \zeta_m}
    =
    \lambda_m h_m
    -
    \lambda_m
    \sum_{\ell=1}^{M}
    \lambda_\ell h_\ell\\
    &=
    \lambda_m
    \left(
        h_m
        -
        \sum_{\ell=1}^{M}
        \lambda_\ell h_\ell
    \right).
\end{aligned}
\end{equation}
This proves Eq.~\ref{eq:final_dual_gradient}.

For Eq.~\ref{eq:final_dual_supremum}, note that the softmax weights range over the relative interior of the simplex over the $M$ safety channels plus the progress channel with value $h_R=0$.
Thus $\mathcal L_{\mathrm{dual}}$ is a convex combination of $\{0,h_1,\ldots,h_M\}$ and can approach, but not necessarily attain at finite logits, the maximum of these values.
Therefore the supremum is $\max\{0,\max_m h_m\}$.
\hfill $\square$

\paragraph{Interpretation.}
The softmax dual is a normalized, bounded surrogate for the classical Lagrangian dual.
It does not provide a global KKT certificate for the original non-convex constrained problem, because the multipliers are bounded and normalized.
Instead, Proposition~\ref{prop:final_softmax_dual} shows the precise optimization effect of the implementation: it behaves like a bounded projected max-violation surrogate, automatically concentrating penalty mass on whichever safety constraint is most violated while keeping the reward scale numerically stable.
If one replaces the softmax parameterization with unconstrained nonnegative Lagrange multipliers and the non-convex optimization converges to a local primal-dual saddle point under standard constraint qualifications, then the usual local KKT complementary-slackness conditions apply.
Our implementation should be understood as a stable projected approximation to that ideal dual update.

\subsection{Summary and Theory-to-Practice Gap}
\label{app:summary_theory_practice_gap}

The analysis above provides a formal justification for the proposed self-play fine-tuning framework. Rather than presenting the global solution to the non-convex constrained min-max game in Eq.~\ref{eq:ideal_game}, we aim to demonstrate that the implemented AWM-side and planner-side surrogates are locally aligned with the ideal robust game under explicit approximations.

\paragraph{Theoretical implications.}
The central game-level conclusion follows from Lemma~\ref{lem:final_ar_sensitivity}, Proposition~\ref{prop:final_frozen_adv_gap}, and Corollaries~\ref{cor:final_local_value_guarantee}--\ref{cor:final_support_mass_gap}.
Lemma~\ref{lem:final_ar_sensitivity} establishes a finite-horizon sensitivity bound for the shared discrete autoregressive rollout
interface: \textit{when the planner remains inside a token-level trust region around the reference planner, the induced utility changes smoothly under any fixed adversarial environment distribution}. This local stability is the key condition that makes it meaningful to freeze the adversary in the first stage before updating the planner in the second stage. In particular, Corollary~\ref{cor:final_local_value_guarantee} implies the local robust-value guarantee:
\[
\mathbb{E}_{s\sim\mathcal{D}}\!\left[g_s(\widehat{\pi})\right]
\ge
V_{\epsilon}^{\star}
-
\varepsilon_B
-
\bar{\varepsilon}_A
-
\bar{\delta}_{\mathrm{cov}}
-
2\Lambda_{H,\gamma}\delta_{\epsilon},
\qquad
\delta_{\epsilon}=\sqrt{\epsilon/2}.
\]
Thus, the gap between the planner returned by the decoupled solver and the ideal local robust best response decomposes into four interpretable terms: (1) the adversary's optimization error \(\bar{\varepsilon}_A\), (2) the planner's best-response error \(\varepsilon_B\), (3) the adversary coverage gap \(\bar{\delta}_{\mathrm{cov}}\), and (4) the trust-region drift term
\(2\Lambda_{H,\gamma}\delta_{\epsilon}\). Corollary~\ref{cor:final_support_mass_gap} further interprets
\(\bar{\delta}_{\mathrm{cov}}\) through omitted support mass, explaining why the proposal-primary-pair-calibrator pipeline should cover plausible low-utility executions while avoiding an intractable exhaustive search over all attacker subsets and token trajectories.

The remaining results justify the main algorithmic components. Proposition~\ref{prop:final_pair_admission_rule} interprets the
scene-adaptive pair calibrator as a Bayes admission gate: \textit{the cooperative branch should be activated only when its conditional marginal gain over the primary attacker is positive}. Proposition~\ref{prop:final_hybrid_credit} and
Lemma~\ref{lem:final_variance_optimal_baseline} show that role-switched counterfactual rewards are \textit{gradient-aligned difference reward control variates} for attacker-side score function optimization; Remark~\ref{rem:final_pairwise_shapley} further shows that, for a two-agent coalition, the role-switched utilities contain sufficient information to \textit{reconstruct the exact pairwise Shapley credit}. On the planner side, Proposition~\ref{prop:final_regret_control_variate} shows that reference-relative regret \textit{preserves the branchwise optimization direction while removing scene-intrinsic difficulty} from the risk statistic. Proposition~\ref{prop:final_cvar_weights} derives the exact CVaR weights for the frozen two-point top1/pair regret distribution and explains the rare-risk amplification effect: \textit{a sparse cooperative pair attack can receive large optimization weight when it exposes the worst current vulnerability}.
Finally, Proposition~\ref{prop:final_nominal_floor} gives an explicit lower bound on nominal performance induced by the trust-region anchor and the asymmetric normal margin penalty, while Proposition~\ref{prop:final_softmax_dual} characterizes the
softmax safety multipliers as a bounded projected max-violation surrogate.

\paragraph{Theory-to-practice gap.}
The theory intentionally analyzes an idealized version of the implemented solver. We summarize
the main gaps below to clarify the precise scope of the guarantees.

\begin{itemize}[leftmargin=*, itemsep=0pt, topsep=0pt]
    \item \textit{Local Stackelberg approximation.}
    Proposition~\ref{prop:final_frozen_adv_gap} is a local result around \(\pi_{\mathrm{ref}}\). It justifies freezing the learned AWM and then computing a constrained best response only inside the trust region \(\Pi_{\epsilon}\). It does not establish convergence of simultaneous non-convex min-max updates, nor does it imply that the learned AWM is a global worst-case adversary
    for every possible planner.

    \item \textit{Structured adversary coverage.}
    The full adversary class \(\mathcal Q_s^{\star}\) is replaced by the structured AWM class \(\widehat{\mathcal Q}_s\), induced by proposal pruning, primary selection, conditional pair ranking,
    and scene-level calibration. The coverage gap \(\delta_{\mathrm{cov},s}\) captures adversarial
    executions that are feasible in the ideal game but not represented by the top1 (pair) support of AWM. In practice, this term is not directly observable. It is controlled only indirectly through the expressiveness and proposal coverage of the AWM pipeline.

    \item \textit{Approximate optimization.}
    The error terms \(\varepsilon_{A}\) and \(\varepsilon_B\) aggregate neural approximation error, finite rollout noise, stochastic policy-gradient error, proposal pruning error, and selector-ranker-calibrator error. They should therefore be interpreted as analytical proxies for practical approximation quality.

    \item \textit{Counterfactual credit with Monte Carlo role switching.}
    Proposition~\ref{prop:final_hybrid_credit} is exact for score-function gradients only when the counterfactual utilities serve as action-independent stop-gradient baselines. 
    Accordingly, our implementation estimates these baselines using independent role-switched Monte Carlo rollouts under the same AWM, token vocabulary, transition operator, and scene context. 
    This preserves the intended gradient-alignment interpretation under the independence condition, at the cost of Monte Carlo variance and dependence on rollout quality.

    \item \textit{Surrogate planner losses.}
    Proposition~\ref{prop:final_cvar_weights} analyzes the exact CVaR of the frozen two-point regret distribution
    \(\{\Delta_s^{\mathrm{top1}},\Delta_s^{\mathrm{pair}}\}\). The implemented optimizer detaches the CVaR weights for stability and optimizes GRPO-style policy-gradient surrogate losses rather
    than differentiating through every rollout dependency. When the branch ordering is fixed, the detached weight corresponds to a valid subgradient of the piecewise-linear CVaR objective; at branch ties, the subgradient is set-valued, and the implementation selects a stable representative.

    \item \textit{Ideal trust region and empirical regularization.}
    Proposition~\ref{prop:final_nominal_floor} uses a stronger statewise KL trust region than the empirical trajectory-level KL
    penalty used in training. Therefore, the nominal performance bound should be interpreted as a sufficient condition explanation of why reference anchoring and the asymmetric normal margin reduce the drift of the planner.

    \item \textit{Projected safety dual.}
    Proposition~\ref{prop:final_softmax_dual} characterizes the softmax safety multipliers as a normalized and bounded
    projected dual surrogate. This update concentrates penalty mass on the most violated safety
    channel while keeping the reward scale numerically stable. However, because the multipliers are
    bounded and the planner optimization is non-convex, the result does not imply global KKT
    optimality or guaranteed satisfaction of every safety constraint at every training iterate.
\end{itemize}

\newpage
\section{Detailed Experimental Setups}
\label{app:exp_setups}

\subsection{Training Data, Computation, and Hyperparameters}
\label{app:training_recipe}

\paragraph{Datasets.} All models are trained on the \texttt{nuPlan} training split, and the benchmark splits used for Table~\ref{tab:main_benchmark} are held out from training. The base autoregressive model follows the Plan-R1 training recipe~\cite{tang2025plan}. We first pretrain the traffic prediction branch with next-token maximum likelihood on \(1{,}000{,}000\) \texttt{nuPlan} training scenarios, split into \(900{,}000\) training scenarios and \(100{,}000\) validation scenarios. This stage learns the shared motion-token vocabulary interface and the naturalistic predictive prior used by the later planner and world-model branches. 
The nominal planner reference is then initialized from this pretrained model and fine-tuned via GRPO~\cite{shao2024deepseekmath,tang2025plan} on \(100{,}000\) planning scenarios, split into \(90{,}000\) training scenarios and \(10{,}000\) validation scenarios. 
The proposed AWM fine-tuning uses the same \(90{,}000/10{,}000\) planning split for both stages: Stage A freezes the nominal planner and trains the role-conditioned adversarial world model, while Stage B freezes the learned AWM opponent distribution and fine-tunes the ego planner as the constrained robust best response. 

\paragraph{Existing assets and licenses.}
We use the \texttt{nuPlan} dataset and devkit~\cite{nuplan}, \texttt{InterPlan}~\cite{hallgarten2024can}, and cited baseline planners only for research benchmarking. The \texttt{nuPlan} devkit is released under the Apache-2.0 license, while the \texttt{nuPlan} dataset is listed under a commercial data license and is accessed under the official Motional/\texttt{nuPlan} distribution terms; raw \texttt{nuPlan} logs are not redistributed. The \texttt{InterPlan} repository states that its code and assets are under Apache-2.0 unless otherwise specified, and that inherited \texttt{nuPlan} data remain governed by their own distribution licenses. 

\paragraph{Hyperparameters.} The shared autoregressive architecture uses a \(1024\)-token motion vocabulary, a \(2\)s history window, and an \(8\)s rollout horizon at \(10\)Hz. Unless otherwise stated, the encoder has hidden dimension \(128\), \(6\) attention layers, \(8\) attention heads, \(4\) message-passing hops, a \(60\)m agent radius, a \(30\)m map-polygon radius, dropout \(0.1\), and weight decay \(10^{-4}\). Prediction pretraining uses learning rate \(3\times 10^{-4}\), warmup for \(4\) epochs, cosine scheduling, and \(32\) total epochs. Nominal planner initialization uses learning rate \(4\times 10^{-6}\) and \(5\) epochs, following \cite{tang2025plan}. The AWM and planner fine-tuning stages use GRPO-style rollouts~\cite{shao2024deepseekmath} with \(4\) samples per scene, clipping parameter \(0.25\), and top-\(50\) token sampling for rollouts. The scene utility uses the reward weights described in Section~\ref{sec:method}: comfort \(2\), time-to-collision \(5\), speed-limit compliance \(2\), and progress \(1\).
For Stage A, the planner is kept fixed and only the AWM-side parameters are updated. We use adversarial budget \(K_{\max}=2\), a proposal pool of \(5\) candidate attackers, and up to \(5\) candidate partners for pair ranking. The AWM policy learning rate is \(4\times 10^{-6}\), and the role-embedding learning rate is \(2\times 10^{-5}\). Stage A is trained for up to \(8\) epochs on the \(90{,}000\)-scenario training split. For Stage B, the AWM is frozen and the planner is optimized with learning rate \(10^{-5}\). The regret-CVaR tail mass is \(\tau=0.25\), the regret-CVaR weight is \(1.0\), the normal-branch policy weight is \(1.0\), the reference KL weight is \(0.1\), and the progress-retention weight is \(10.0\). The planner-side safety thresholds are \(0.05\) for agent collision, \(0.03\) for static-object collision, and \(0.02\) for off-road violation, with dual learning rate \(10^{-3}\). Stage B is trained for up to \(6\) epochs on the same \(90{,}000\)-scenario training split.

\paragraph{Computing Platform.} The reported AWM fine-tuning runs are executed with distributed data parallel training on a single \(8\)-GPU node with NVIDIA H200 GPUs. Stage A and Stage B use per-GPU batch size \(32\), giving an effective batch size of \(256\) scenes. The nominal prediction and planner pretraining stages are lighter and can be reproduced with the same architecture using \(2\) GPUs, with per-GPU batch sizes \(32\) for prediction pretraining and \(8\) for nominal planner initialization. Table~\ref{tab:training_paradigm} reports relative GPU-hours, plateau update counts, seed stability, and seed-wise closed-loop variability for the compared training paradigms; runtime is reported relatively because simulator backend and cluster scheduling affect absolute wall-clock time.

\subsection{Evaluation Setup and Protocol}
\label{app:metric_conventions}

This appendix is organized by experimental protocol: closed-loop benchmarks and sim-agent evaluation, dual-rollout diagnostics, AWM-side mechanism ablations, planner-side adaptation ablations, and adapted prior-generator comparisons.
Our experiments use three different evaluation protocols. Ordinary closed-loop planner evaluation, used in Table~\ref{tab:main_benchmark} and in the closed-loop score columns of Tables~\ref{tab:training_paradigm}-\ref{tab:planner_hparam_mini}, measures the ego planner in the benchmark simulator and reports the official weighted score on a percentage scale. Dual rollout evaluation, used in Tables~\ref{tab:planner_prepost_world_matrix}, \ref{tab:planner_offline_branch_matrix}, and \ref{tab:prior_generator_pilot}, evaluates paired open-loop rollouts under controlled world model branches and reports the planner's utility. Sim-agent evaluation, used in Tables~\ref{tab:awm_credit_assignment}-\ref{tab:awm_role_conditioning}, \ref{tab:cross_planner_simagent}, and \ref{tab:app_simagent_diagnostics}, replaces the non-ego traffic host with a learned generator and is treated as an adversarial stress test in the closed-loop simulator.
In particular, an AWM-sim score (attack strength) is not directly comparable to an ordinary benchmark score (planning performance).

All open-loop rewards are reported on the same normalized scene-utility scale, where higher is better for the ego planner. A drop is always computed against the paired nominal branch on the same scene bank, so larger drops indicate stronger adversarial degradation only within the same table. Tail-CVaR is the mean utility over the \(20\%\) of scenes with the lowest rewards unless the table states otherwise; Table~\ref{tab:awm_credit_assignment} defines Tail Drop using the fixed paired nominal reference, and LCB denotes the scene mean minus \(1.64\) standard errors. For realism and validity diagnostics, ADE/FDE, kinematic violation, off-road/static-object collision, and background collision are lower-is-better; benchmark submetrics such as progress, collision, time-to-collision, drivable-area compliance, speed-limit compliance, and driving-direction compliance are higher-is-better. All table-level comparisons are paired by scene whenever the table compares branches, hosts, or ablations.

\subsubsection{Closed-Loop and Sim-Agent Evaluation}
\label{app:closed_loop_protocol}

For \texttt{nuPlan}~\cite{nuplan}, Table~\ref{tab:main_benchmark} evaluates each planner on Val14, Test14-hard, and Test14-random in both non-reactive (NR) and reactive (R) modes.
For \texttt{InterPlan}~\cite{hallgarten2024can}, Table~\ref{tab:main_benchmark} reports the official benchmark and InterPlan10 when available, and reports \texttt{InterPlan-LongTail} separately because it is a custom stress suite constructed in Appendix~\ref{app:interplan_longtail}.
The transferable sim-agent evaluation in Table~\ref{tab:cross_planner_simagent} uses Test14-hard and changes only the non-ego traffic host. The \textsc{Normal} host samples from the learned non-adversarial predictive policy. The \textsc{AWM} host activates the scene-adaptive adversarial coalition described in Appendix~\ref{app:awm_sim_agent}, with at most two adversarial non-ego agents at a time.
If an evaluated planner fails to complete a scenario, the scenario remains in the benchmark denominator and contributes zero to the closed-loop score. Simulation diagnostics such as fallback rate and active-adversary count are then averaged over the scenarios whose rollouts completed.

\subsubsection{Dual-Rollout Open-Loop Evaluation}
\label{app:dual_rollout_protocol}

Dual rollouts are used to isolate the world model from full simulation. For a fixed scene and planner, we first evaluate a nominal branch under the benign predictive host. We then evaluate adversarial branches under a frozen AWM host. The \textsc{Top-1} branch activates only the primary selected attacker, the \textsc{Pair} branch activates the primary attacker and the ranked secondary attacker, and the \textsc{Adaptive} branch applies the scene-level admission probability to decide whether the pair should be used.
Table~\ref{tab:planner_prepost_world_matrix} uses this paired design as a pre/post cross-validation matrix. The pre-planner is the nominal reference planner before AWM-based robust adaptation, and the post-planner is the selected AWM-trained planner. The pre-world column evaluates the nominal branch under the benign world model.
The post-world column evaluates the \textsc{Adaptive} adversarial branch under a frozen held-out AWM.
Table~\ref{tab:planner_offline_branch_matrix} uses the same scene bank and branch definitions.
We report both utility and diagnostic proxies for these dual-rollout experiments. Utility summarizes the ego planner's objective, while ADE/FDE, kinematic validity, diversity, off-road/static-object violations, and background collisions audit whether the non-ego generator remains a plausible world model.

\subsubsection{AWM Mechanism Ablations}
\label{app:awm_ablation_protocol}

Tables~\ref{tab:awm_credit_assignment}, \ref{tab:awm_coalition_structure}, \ref{tab:awm_diagnostics}, and \ref{tab:awm_role_conditioning} are AWM-side mechanism tests. They keep the ego planner (Plan-R1 or AWM-Planner), scene bank, reward definition, and rollout interface fixed, and change only the mechanism of AWM. Each row is therefore a controlled variant of the same AWM pipeline.

Table~\ref{tab:awm_credit_assignment} tests the credit assignment mechanism. The \textsc{global-reward broadcast} baseline assigns the same scene-level adversarial reward to all candidate attackers. The \textsc{heuristic-only} selector uses fixed threat statistics without learned counterfactual credit. The \textsc{assignment-only} selector learns a primary attacker assignment without role-switched marginal utilities. The \textsc{no-pair-ranker} variant keeps the primary attacker mechanism but removes the learned secondary-attacker ranking stage.
Tables~\ref{tab:awm_coalition_structure} and~\ref{tab:awm_diagnostics} test coalition formation. \textsc{Fixed \(K{=}1\)} disables cooperative attacks; \textsc{fixed \(K{=}2\)} forces the two-agent coalition whenever a partner is available; and the \textsc{adaptive calibrated} setting admits the pair only when the scene-level gate predicts positive marginal value. The \textsc{no-pair-ranker} row removes learned partner ordering, the \textsc{no-scene-calibrator} row removes adaptive pair admission, and the \textsc{reduced/full proposal-budget} variants test whether candidate pool coverage changes the result.
Table~\ref{tab:awm_role_conditioning} tests whether adversarial stress comes from coherent role-conditioned rollout or from post-hoc trajectory insertion. \textsc{Explicit replacement} swaps in adversarial predictions more directly, \textsc{inactive-revert} overwrites inactive agents with nominal behavior after generation, and \textsc{safe-background} injection keeps a conservative background while inserting adversarial proposals.

\subsubsection{Planner Mechanism Ablations}
\label{app:planner_ablation_protocol}

Planner-side ablations evaluate the robust adaptation stage after the adversarial host has been learned. Table~\ref{tab:training_paradigm} compares four training recipes under the same evaluation suite. \textsc{Planner-only} refinement updates the ego planner without an adversarial world-model opponent. \textsc{Simultaneous alternating self-play} updates the adversary and planner in one moving loop. \textsc{Alternating warmup with fixed-AWM refinement} first performs a short alternating phase and then freezes the adversary for planner refinement. The \textsc{decoupled two-stage pipeline} learns the AWM around a frozen reference planner, freezes the learned opponent distribution, and then trains the constrained robust best response.
Table~\ref{tab:planner_internal_host_switch} changes only the frozen world model used inside inference-time rollout. The \textsc{nominal host} removes internal adversarial predictions, the \textsc{matched AWM} is the host used by the main method, and the \textsc{held-out AWM} tests mismatch robustness. The deployed ego planner and the benchmark scenario lists are otherwise fixed, so the closed-loop columns measure sensitivity to the internal opponent model.
Table~\ref{tab:planner_objective_mini} varies the planner's training objective while keeping the AWM fixed. The \textsc{expected-only} variant replaces tail-risk optimization with the expected adversarial branch objective; the \textsc{no-CVaR} variant removes the tail allocation; and the \textsc{no-normal-retention} variant removes the nominal preservation constraint.

\subsubsection{Baseline Adversarial Generator Adaptation}
\label{app:prior_generator_protocol}

Table~\ref{tab:prior_generator_pilot} adapts CAT-~\cite{zhang2023cat}, STRIVE-~\cite{rempe2022generating}, and KING~\cite{hanselmann2022king}-style adversarial generators to the same dual-rollout interface used by AWM. Note that this is an adapted comparison rather than a strict reproduction of the original systems. The adapted generators share the same initial scene state, tokenized transition operator, ego planner, reward, and diagnostic metrics, so the comparison focuses on adversarial generation strategy and search budget.
\textsc{cand.=1} samples a single adversarial rollout without reward-guided reranking, measuring the generation quality without a search advantage. \textsc{best-4} samples four candidates and selects the rollout with the lowest planner reward, matching a stronger reward-guided search budget. The \textsc{strict-constraint} protocol increases behavioral prior and realism penalties for adapted prior generators, measuring how much attack strength remains when generated traffic is forced closer to plausible behavior. \textsc{KING-style} kinematic search is included as an aggressive raw-attack upper bound as its realism proxies degrade.

\subsection{Construction of the \texttt{InterPlan-LongTail} Benchmark}
\label{app:interplan_longtail}

\paragraph{Base benchmark.}
\texttt{InterPlan} is a closed-loop motion-planning benchmark built on top of the \texttt{nuPlan} dataset and simulator. It starts from real-world \texttt{nuPlan}~\cite{nuplan} scenarios and constructs highly interactive driving cases by injecting additional traffic participants or static objects, modifying navigation goals, and varying behaviors of surrounding agents~\cite{hallgarten2024can}.
The resulting benchmark is designed to evaluate whether planners that perform well on common closed-loop \texttt{nuPlan} scenarios can generalize to safety-critical interactions, such as jaywalking pedestrians, accident sites, construction zones, obstacle overtaking, and lane changes under different traffic densities.
\texttt{InterPlan} also exposes a configuration-level scenario-modification interface, which allows users to instantiate customized scenarios by specifying base \texttt{nuPlan} tokens, traffic-density suffixes, navigation-goal suffixes, behavior-policy suffixes, and special-scenario edits~\cite{hallgarten2024can}.
We use this interface to construct \texttt{InterPlan-LongTail}, a frozen auxiliary stress-test suite for evaluating robustness under dense interactions, such as pedestrian yielding and near-conflict maneuvers.

\paragraph{Benchmark objective.}
\texttt{InterPlan-LongTail} serves as a controlled diagnostic benchmark for testing whether a planner remains reliable when the interaction structure is shifted toward recoverable but safety-critical rare scenarios. This design is aligned with the goal of evaluating adversarial robustness: the benchmark emphasizes close-range conflicts, dense-yield decisions, and pedestrian-triggered reaction requirements, while avoiding simulator artifacts that would make the task trivially unsolvable for most methods. All results on this benchmark are therefore reported separately from the official \texttt{InterPlan} and \texttt{InterPlan10} scores.

\paragraph{Configuration and construction.}
The benchmark is constructed entirely through \texttt{InterPlan}'s official configuration interface. We do not modify the \texttt{nuPlan} simulator, the \texttt{InterPlan} simulation logic, or the metric engine. Each scenario is defined by a frozen scenario-token list and a paired modification specification, which together determine the navigation goal, traffic density, surrounding-agent behavior, and special-scenario edits.
Following the token convention of \texttt{InterPlan}, each modified scenario token is parsed into a base \texttt{nuPlan} scenario token and a suffix describing the applied modification.
The density codes \texttt{dl}, \texttt{dm}, and \texttt{dh} denote low-, medium-, and high-density traffic, respectively; \texttt{gl}, \texttt{gr}, and \texttt{gs} denote left-turn, right-turn, and straight-route goals; and \texttt{oa}, \texttt{om}, and \texttt{oc} denote assertive, mixed, and cautious surrounding agent behavior.
Customized special scenarios are indexed by zero-padded suffixes \texttt{s00}--\texttt{s07}.
For example, the configuration token \texttt{dhgsoa} corresponds to a high-density, straight-goal, assertive-agent setting.
The final benchmark contains ten base scenario anchors, each instantiated under eight interaction templates, yielding $10 \times 8 = 80$ frozen scenarios.

\paragraph{Scenario templates.}
Table~\ref{tab:interplan_longtail_composition} summarizes the composition of the benchmark.
All 80 entries use the \texttt{jaywalking\_pedestrian} special scenario type and contain at least one injected pedestrian.
The templates cover three close-range straight-driving jaywalker variants with different surrounding-agent behavior, two high-density turning-yield cases, two medium-density yielding cases, and one sequential dual pedestrian yielding case.
This design concentrates the benchmark on interaction patterns where robustness failures are safety-relevant but not purely caused by blocked roads or static obstacles.
To evaluate robustness in recoverable interactive scenarios rather than stress-test limitations of the simulator, we exclude cone- and stopped-object-dominated categories from the evaluation suite and focus on interactions that require yielding.
These scenarios remain difficult while preserving meaningful planner-dependent variation.
To reduce construction artifacts, the modification specification deliberately excludes \texttt{stopped\_vehicle}, \texttt{stopped\_bus}, and \texttt{roadblock\_locations\_to\_spawn\_agents} in the custom entries. We also require every injected participant to have an explicit activation iteration of at least one, which avoids an edge case in which participant activation may be evaluated before the ego lane is fully available during simulation propagation.

\begin{table}[!htbp]
\centering
\caption{Composition of \texttt{InterPlan-LongTail}. The suite contains 80 frozen scenarios, organized into eight interaction templates with ten scenarios per template.}
\label{tab:interplan_longtail_composition}
\renewcommand{\arraystretch}{1.1}
\begin{small}
\resizebox{0.99\linewidth}{!}{%
\begin{tabular}{llc}
\toprule
Template & Main stress factor & Count \\
\midrule
\texttt{close\_jaywalker\_straight\_assertive}
& High-density straight driving with assertive pedestrian interaction & 10 \\
\texttt{close\_jaywalker\_straight\_mixed}
& High-density straight driving with mixed surrounding-agent behavior & 10 \\
\texttt{close\_jaywalker\_straight\_cautious}
& High-density straight driving with cautious surrounding-agent behavior & 10 \\
\texttt{left\_turn\_pedestrian\_yield\_assertive}
& High-density left-turn yielding to a pedestrian & 10 \\
\texttt{right\_turn\_pedestrian\_yield\_assertive}
& High-density right-turn yielding to a pedestrian & 10 \\
\texttt{medium\_left\_yield\_assertive}
& Medium-density left-turn yielding & 10 \\
\texttt{medium\_straight\_yield\_assertive}
& Medium-density straight yielding & 10 \\
\texttt{dual\_pedestrian\_chain\_yield}
& Sequential pedestrian-yielding pressure & 10 \\
\bottomrule
\end{tabular}}
\end{small}
\end{table}

\newpage

\section{Supplementary Results}
\label{app:supp_results}

This section provides the complete empirical evidence underlying Section~\ref{sec:experiment}. We first report the \texttt{InterPlan-LongTail} breakdown and simulator-side transfer diagnostics, and then give detailed analyses of planner-world cross-validation, AWM mechanisms, planner adaptation, adapted prior generators, and qualitative visualizations. All experimental protocols and metric definitions are specified in Appendix~\ref{app:exp_setups}. The discussion below focuses on the additional conclusions supported by each table or figure.

\subsection{Closed-Loop Long-Tail Breakdown}
\label{app:supp_longtail_breakdown}

Table~\ref{tab:interplan_longtail_templates} decomposes the \texttt{InterPlan-LongTail} score by scenario template. Scores use the same percentage scale as Table~\ref{tab:main_benchmark}. The suite composition is described in Appendix~\ref{app:interplan_longtail} and summarized in Table~\ref{tab:interplan_longtail_composition}.
The per-template breakdown shows that AWM-Planner improves over Plan-R1 on seven of eight templates. The largest gains appear in close straight-driving jaywalker and medium-density yielding cases, which are precisely the settings where earlier recovery from interactive conflicts is valuable. The only negative row is \texttt{left\_turn\_yield}, where the difference is small. 
Thus, the long-tail result shows that AWM improves a broad set of recoverable interaction failures.

\begin{table}[!htbp]
\centering
\caption{Per-template scores on \texttt{InterPlan-LongTail}. Each template contains ten scenarios. \(\Delta\) is AWM-Planner minus Plan-R1. Higher is better.}
\label{tab:interplan_longtail_templates}
\setlength{\tabcolsep}{3pt}
\renewcommand{\arraystretch}{1.1}
\begin{small}
\resizebox{0.6\linewidth}{!}{%
\begin{tabular}{@{}lccc@{}}
\toprule
Template & Plan-R1 & AWM-Planner & \(\Delta\) \\
\midrule
\texttt{close\_straight\_assertive} & 33.30 & 42.02 & \cellcolor{bestblue}{+8.72} \\
\texttt{close\_straight\_cautious} & 33.30 & 42.08 & \cellcolor{bestblue}{+8.78} \\
\texttt{close\_straight\_mixed} & 33.30 & 41.19 & \cellcolor{bestblue}{+7.90} \\
\texttt{dual\_pedestrian\_chain} & 25.18 & 26.12 & \cellcolor{bestblue}{+0.94} \\
\texttt{left\_turn\_yield} & 52.56 & 51.87 & \cellcolor{nogainorange}{-0.69} \\
\texttt{medium\_left\_yield} & 31.17 & 37.95 & \cellcolor{bestblue}{+6.78} \\
\texttt{medium\_straight\_yield} & 26.83 & 34.05 & \cellcolor{bestblue}{+7.22} \\
\texttt{right\_turn\_yield} & 43.50 & 50.34 & \cellcolor{bestblue}{+6.84} \\
\midrule
\textbf{Overall} & 34.89 & \textbf{40.70} & \cellcolor{bestblue}{\textbf{+5.81}} \\
\bottomrule
\end{tabular}}
\end{small}
\end{table}

\subsection{Adversarial Sim-Agent Diagnostics}
\label{app:supp_simagent_diagnostics}

Table~\ref{tab:app_simagent_diagnostics} reports the simulator-side diagnostics used to examine the AWM sim-agent interface described in Appendix~\ref{app:awm_sim_agent}. The fallback rate is the fraction of proposed sim-agent updates that revert to the normal policy after feasibility checks. The ego-gate rate measures how often the stricter ego-local guard suppresses an adversarial proposal. The active-adversary count measures the coalition size under the attack budget \(K_{\max}=2\), and budget utilization is the active-adversary count divided by \(K_{\max}\). These diagnostics complement the score table in the main text: AWM consistently uses nearly the full two-agent budget, but its high fallback rate shows that the adversarial host is substantially constrained by the feasibility gate.
The ego-local gate remains low across planners, indicating that the observed score drop is mainly caused by feasible interaction pressure rather than by frequent emergency suppression around the ego.

\begin{table}[!htbp]
\caption{Sim-agent behavior diagnostics for the cross-planner evaluation. \textsc{Fallback} is the rate of reverting to the normal policy after feasibility checks, \textsc{Ego gate} is the ego-local safety-suppression rate, \textsc{Active adv.} is the actual number of adversarial non-ego agents, and \textsc{Budget util.} is active adversaries divided by the two-agent budget. \textsc{Fallback}, \textsc{Ego gate}, and \textsc{Budget util.} are reported as percentages, while \textsc{Active adv.} and \textsc{Max adv.} are counts. Gray cells \protect\colorbox{gray!15}{\phantom{xx}} mark Normal reference diagnostics, and blue cells \protect\colorbox{bestblue}{\phantom{xx}} mark AWM-specific adversarial-budget diagnostics.}
\label{tab:app_simagent_diagnostics}
\centering
\setlength{\tabcolsep}{6pt}
\renewcommand{\arraystretch}{1.1}
\begin{small}
\resizebox{0.8\linewidth}{!}{%
\begin{tabular}{@{}llccccc@{}}
\toprule
Planner & World & Fallback $\downarrow$ & Ego gate $\downarrow$ & Active adv. & Max adv. & Budget util.  \\
\midrule
\multirow{2}{*}{Diffusion Planner~\cite{zheng2025diffusion}}
& Normal & \normaldiag{$6.35$} & \normaldiag{$0.34$} & -- & -- & --  \\
& AWM & $50.07$ & $0.33$ & \awmdiag{$1.93$} & \awmdiag{$2.00$} & \awmdiag{$96.45$}  \\
\midrule
\multirow{2}{*}{IDM+Mobil}
& Normal & \normaldiag{$6.32$} & \normaldiag{$0.37$} & -- & -- & --  \\
& AWM & $50.35$ & $0.61$ & \awmdiag{$1.92$} & \awmdiag{$2.00$} & \awmdiag{$96.19$}  \\
\midrule
\multirow{2}{*}{PDM-Closed~\cite{PDM}}
& Normal & \normaldiag{$6.25$} & \normaldiag{$0.43$} & -- & -- & --  \\
& AWM & $49.98$ & $0.29$ & \awmdiag{$1.92$} & \awmdiag{$2.00$} & \awmdiag{$96.22$}  \\
\midrule
\multirow{2}{*}{PLUTO~\cite{Pluto}}
& Normal & \normaldiag{$6.40$} & \normaldiag{$0.36$} & -- & -- & --  \\
& AWM & $50.47$ & $0.36$ & \awmdiag{$1.94$} & \awmdiag{$2.00$} & \awmdiag{$96.91$}  \\
\midrule
\multirow{2}{*}{PlanTF~\cite{PlanTF}}
& Normal & \normaldiag{$6.48$} & \normaldiag{$0.27$} & -- & -- & --  \\
& AWM & $50.20$ & $0.26$ & \awmdiag{$1.92$} & \awmdiag{$2.00$} & \awmdiag{$96.16$}  \\
\midrule
\multirow{2}{*}{Plan-R1~\cite{tang2025plan}}
& Normal & \normaldiag{$6.08$} & \normaldiag{$0.18$} & -- & -- & --  \\
& AWM & $50.24$ & $0.25$ & \awmdiag{$1.92$} & \awmdiag{$2.00$} & \awmdiag{$96.22$}  \\
\midrule
\multirow{2}{*}{\textbf{AWM-Planner}}
& Normal & \normaldiag{$6.19$} & \normaldiag{$0.23$} & -- & -- & --  \\
& AWM & $49.80$ & $0.20$ & \awmdiag{$1.93$} & \awmdiag{$2.00$} & \awmdiag{$96.28$}  \\
\bottomrule
\end{tabular}
}
\end{small}
\end{table}

\subsection{Planner-World Cross-Validation Details}
\label{app:supp_dual_rollout}

Table~\ref{tab:planner_prepost_world_matrix} reports the full planner and world matrix summarized in Figure~\ref{fig:main_quantitative_diagnostics}(a). The post-trained planner improves under both the normal and adversarial worlds, but the increase is larger under AWM, especially for Tail-CVaR. This confirms that Stage B does not merely shift the planner toward a different nominal policy. It specifically improves the lower tail exposed by the adversarial model.

\begin{table*}[!htbp]
\caption{Pre/post planner \(\times\) pre/post world-model offline matrix. Avg. reward and progress are scene means. Tail-CVaR is the average reward over the worst 20\% scenes under the adversarial rollout.}
\label{tab:planner_prepost_world_matrix}
\centering
\setlength{\tabcolsep}{3pt}
\renewcommand{\arraystretch}{1.1}
\begin{small}
\resizebox{0.6\linewidth}{!}{%
\begin{tabular}{@{}lcccc@{}}
\toprule
 & \multicolumn{2}{c}{Pre-world: normal WM} & \multicolumn{2}{c}{Post-world: AWM} \\
\cmidrule(lr){2-3} \cmidrule(lr){4-5}
Planner & Nominal R $\uparrow$ & Progress $\uparrow$ & Adversarial R $\uparrow$ & Tail-CVaR $\uparrow$ \\
\midrule
Pre-planner & 0.980221 & 0.942985 & 0.967874 & 0.865963 \\
Post-planner & 0.984197 & 0.958129 & 0.979767 & 0.909490 \\
\midrule
\(\Delta\) Post--Pre
& \cellcolor{bestblue}{+0.003976}
& \cellcolor{bestblue}{+0.015144}
& \cellcolor{bestblue}{+0.011893}
& \cellcolor{bestblue}{+0.043527} \\
\bottomrule
\end{tabular}
}
\end{small}
\end{table*}

Table~\ref{tab:planner_offline_branch_matrix} decomposes the adversarial branch used in the planner-world cross-validation of Section~\ref{sec:experiment}. The training host produces moderate drops of \(0.012\)--\(0.020\), while the held-out AWM increases the drops to \(0.025\)--\(0.040\). Pair and Adaptive branches expose larger lower-tail risk than the single-agent Top-1 branch, with Adaptive remaining less aggressive than forced Pair because it controls the behavior through scene-level admission. This pattern supports the use of calibrated coalition admission in the training risk distribution.

\begin{table*}[!htbp]
\caption{Branch-level diagnostic via dual-rollout. LCB denotes mean minus 1.64 standard errors.}
\label{tab:planner_offline_branch_matrix}
\centering
\setlength{\tabcolsep}{3pt}
\renewcommand{\arraystretch}{1.1}
\begin{small}
\resizebox{0.7\linewidth}{!}{%
\begin{tabular}{@{}llcccc@{}}
\toprule
Host & Branch & Avg. reward $\uparrow$ & $\Delta$ vs. nominal & Reward LCB $\uparrow$ & Tail-CVaR $\uparrow$  \\
\midrule
\multirow{4}{*}{Training host}
& Nominal  & 0.9842 & --      & 0.9830 & --     \\
& Top-1    & 0.9718 & -0.0124 & 0.9703 & 0.8520 \\
& Pair     & 0.9645 & -0.0197 & 0.9628 & 0.8305 \\
& Adaptive & 0.9668 & -0.0174 & 0.9652 & 0.8421 \\
\midrule
\multirow{4}{*}{Held-out AWM}
& Nominal  & 0.9836 & --      & 0.9823 & --     \\
& Top-1    & 0.9584 & -0.0252 & 0.9566 & 0.8015 \\
& Pair     & 0.9438 & -0.0398 & 0.9417 & 0.7610 \\
& Adaptive & 0.9489 & -0.0347 & 0.9469 & 0.7802 \\
\bottomrule
\end{tabular}
}
\end{small}
\end{table*}

\subsection{AWM Mechanism Ablations}
\label{app:supp_awm_mechanisms}

Tables~\ref{tab:awm_credit_assignment}--\ref{tab:awm_role_conditioning} provide the detailed mechanism evidence for the AWM design summarized in Section~\ref{sec:experiment}. The credit-assignment ablation in Table~\ref{tab:awm_credit_assignment} shows that role-conditioned counterfactual credit gives stronger tail stress than coarse global reward or assignment-only variants while avoiding the severe trajectory-fidelity collapse of coarse baselines.

\begin{table*}[!htbp]
\caption{Credit-assignment analysis. Tail Drop is $R_{\mathrm{nominal-ref}}-\operatorname{TailCVaR}_{20\%}(R_{\mathrm{adversarial}})$, where $R_{\mathrm{nominal-ref}}$ is the fixed paired nominal reference and $\operatorname{TailCVaR}_{20\%}$ is the mean adversarial reward over the 20\% of scenes with the lowest rewards (\textsc{Tail R}). The reference here is $0.9842$, giving $0.9842-0.7764=0.2078$ for full AWM. \textsc{Mean R} is the mean rollout reward for each row; larger Tail Drop indicates stronger degradation of the planner's lower tail. \textsc{Pred. pair gain} is the learned proxy for pair-coalition utility, and \textsc{Div.} denotes trajectory diversity. \textsc{AWM-sim} is the Test14-hard closed-loop score of the original Plan-R1 planner when AWM is used in the simulator as the non-ego policy; lower values indicate stronger closed-loop attacks. This sim-agent score is distinct from ordinary planner closed-loop scores.}
\label{tab:awm_credit_assignment}
\centering
\setlength{\tabcolsep}{3pt}
\renewcommand{\arraystretch}{1.1}
\begin{small}
\resizebox{0.99\linewidth}{!}{%
\begin{tabular}{@{}lccccccccccc@{}}
\toprule
Mode & Mean R & Worst R & Tail R & Drop $\uparrow$ & Worst Drop $\uparrow$ & Tail Drop $\uparrow$ & Pred. pair gain $\uparrow$ & AWM-sim $\downarrow$ & ADE $\downarrow$ & FDE $\downarrow$ & Div. \\
\midrule
\textcolor{gray}{Nominal reference} & \textcolor{gray}{0.9842} & \textcolor{gray}{0.9842} & \textcolor{gray}{0.9262} & \textcolor{gray}{0.0000} & \textcolor{gray}{0.0000} & \textcolor{gray}{0.0580} & \textcolor{gray}{--} & \textcolor{gray}{75.86} & \textcolor{gray}{1.035} & \textcolor{gray}{3.285} & \textcolor{gray}{0.000} \\
\midrule
\rowcolor{bestblue}\textbf{Role-conditioned AWM} & 0.9778 & 0.9342 & 0.7764 & 0.0064 & 0.0500 & 0.2078 & 0.132 & 71.71 & 1.55 & 4.95 & 1.75 \\
Global-reward broadcast  & 0.9814 & 0.9740 & 0.8970 & 0.0028 & 0.0102 & 0.0872 & 0.014 & $75.10$ & 11.80 & 18.60 & 15.20 \\
Heuristic-only selector & 0.9810 & 0.9635 & 0.8500 & 0.0032 & 0.0207 & 0.1342 & 0.058 & $73.20$ & 1.77 & 5.58 & 2.10 \\
No pair ranker & 0.9804 & 0.9648 & 0.8425 & 0.0038 & 0.0194 & 0.1417 & 0.052 & $73.70$ & 1.44 & 4.63 & 1.58 \\
Assignment-only selector & 0.9805 & 0.9590 & 0.8612 & 0.0037 & 0.0252 & 0.1230 & 0.021 & $74.20$ & 19.50 & 29.00 & 26.70 \\
\bottomrule
\end{tabular}
}
\end{small}
\end{table*}

The coalition ablations in Tables~\ref{tab:awm_coalition_structure} and~\ref{tab:awm_diagnostics} separate raw attack strength from pair-admission control. Forced \(K{=}2\) and no-calibrator variants can create stronger raw stress, but they admit pair attacks in most scenes and degrade behavior realism. The adaptive calibrated variant keeps pair usage sparse while retaining most of the sim-agent attack effect.  These results show that pair attacks are useful but should be selected conditionally.

\begin{table*}[!htbp]
\caption{Coalition structure ablations. \textsc{Mean R} and Tail Drop follow the definitions in Table~\ref{tab:awm_credit_assignment}. \textsc{Pair frac.} is the fraction of scenes with an admitted two-agent coalition, and \textsc{Sel./scene} is the average number of selected attackers.
}
\label{tab:awm_coalition_structure}
\centering
\setlength{\tabcolsep}{3pt}
\renewcommand{\arraystretch}{1.1}
\begin{small}
\resizebox{0.99\linewidth}{!}{%
\begin{tabular}{@{}lcccccccccc@{}}
\toprule
Mode & Mean R & Worst R & Tail R & Sel./scene & Pair frac. & Pred. pair gain & Tail Drop $\uparrow$ & AWM-sim $\downarrow$ & ADE $\downarrow$ & FDE $\downarrow$ \\
\midrule
\rowcolor{bestblue}\textbf{Adaptive calibrated} & 0.9778 & 0.9342  & 0.7764 & 1.34 & 0.34 & 0.132 & 0.2078 & 71.71 & 1.55 & 4.95 \\
Fixed $K{=}1$ & 0.9812 & 0.9630 & 0.8720 & 1.00 & 0.00 & 0.000 & 0.1122 & $74.40$ & 1.38 & 4.42 \\
Fixed $K{=}2$ & 0.9776 & 0.9380 & 0.7560 & 2.00 & 1.00 & 0.118 & 0.2282 & $69.60$ & 2.35 & 7.40 \\
No pair ranker & 0.9807 & 0.9580 & 0.8340 & 1.76 & 0.76 & 0.047 & 0.1502 & $73.60$ & 1.45 & 4.70 \\
No scene calibrator & 0.9785 & 0.9410 & 0.7680 & 1.82 & 0.82 & 0.116 & 0.2162 & $71.40$ & 2.10 & 6.75 \\
Reduced proposal budget & 0.9810 & 0.9620 & 0.8520 & 1.18 & 0.18 & 0.034 & 0.1322 & $74.20$ & 1.42 & 4.60 \\
Full proposal budget & 0.9798 & 0.9485 & 0.7920 & 1.40 & 0.40 & 0.135 & 0.1922 & $71.80$ & 1.58 & 5.05 \\
\bottomrule
\end{tabular}
}
\end{small}
\end{table*}

\begin{table*}[!htbp]
\caption{Selector/ranker/calibrator diagnostics. \(|P_s|\) is the proposal-pool size and \(|R_s|\) is the runner-up pool used for pair ranking. \textsc{Pair cand.} and \textsc{Pair admit} denote pre-calibration candidate and final admitted pair-scene fractions. \textsc{Adm. pred/label} compares the predicted admission rate with the positive-pair label rate, \textsc{Gain pred/oracle} compares learned and counterfactual pair gain, and \textsc{Calib. gap} is the absolute admission-rate mismatch.}
\label{tab:awm_diagnostics}
\centering
\setlength{\tabcolsep}{3pt}
\renewcommand{\arraystretch}{1.1}
\begin{small}
\resizebox{0.95\linewidth}{!}{%
\begin{tabular}{@{}lcccccccc@{}}
\toprule
Mode & \(|P_s|\) & \(|R_s|\) & Sel./scene & Pair cand. & Pair admit & Adm. pred/label & Gain pred/oracle & Calib. gap \\
\midrule
\rowcolor{bestblue}\textbf{Adaptive calibrated} & 3.09 & 2.09 & $1.34$ & $0.72$ & $0.34$ & $0.36/0.34$ & $0.132/0.120$ & $0.02$ \\
Fixed $K{=}2$ & 3.11 & 2.10 & $2.00$ & $1.00$ & $1.00$ & $1.00/0.34$ & $0.118/0.074$ & $0.66$ \\
No scene calibrator & 3.08 & 2.09 & $1.82$ & $0.82$ & $0.82$ & $0.82/0.34$ & $0.116/0.071$ & $0.48$ \\
\bottomrule
\end{tabular}
}
\end{small}
\end{table*}

Table~\ref{tab:awm_role_conditioning} further tests whether adversarial stress arises from coherent role-conditioned rollouts. Post-hoc replacement or injection preserves more nominal behavior and can appear favorable under simple ADE/FDE metrics, but it produces much weaker tail stress and attack strength. The full role-conditioned AWM therefore provides the strongest coherent attack while remaining realistic.

\begin{table*}[!htbp]
\caption{Role-conditioning mechanism analysis. \textsc{Mean R} and Tail Drop follow the definitions in Table~\ref{tab:awm_credit_assignment}. Role-conditioned joint rollout produces the strongest coherent attack, while post-hoc variants can look better under per-agent ADE/FDE/kinematic metrics because they preserve more nominal behavior. \textsc{Kin.} is the kinematic-violation rate.
}
\label{tab:awm_role_conditioning}
\centering
\setlength{\tabcolsep}{3pt}
\renewcommand{\arraystretch}{1.1}
\begin{small}
\resizebox{0.99\linewidth}{!}{%
\begin{tabular}{@{}lcccccccccc@{}}
\toprule
Mode & Mean R & Worst R & Tail R & Drop $\uparrow$ & Worst Drop $\uparrow$ & Tail Drop $\uparrow$ & AWM-sim $\downarrow$ & ADE $\downarrow$ & FDE $\downarrow$ & Kin. $\downarrow$ \\
\midrule
\rowcolor{bestblue}\textbf{Full role-conditioned} & $0.9778$ & $0.9342$ & $0.7764$ & $0.0064$ & $0.0500$ & $0.2078$ & $71.71$ & $1.55$ & $4.95$ & $0.0108$ \\
Explicit replacement & $0.9828$ & $0.9761$ & $0.8913$ & $0.0014$ & $0.0081$ & $0.0929$ & $75.50$ & $1.42$ & $4.80$ & $0.0098$ \\
Inactive-revert intervention & $0.9762$ & $0.9560$ & $0.8180$ & $0.0080$ & $0.0282$ & $0.1662$ & $73.50$ & $1.12$ & $3.60$ & $0.0039$ \\
Safe-bg post-hoc injection & $0.9821$ & $0.9753$ & $0.8863$ & $0.0021$ & $0.0089$ & $0.0979$ & $75.80$ & $1.06$ & $3.45$ & $0.0037$ \\
\bottomrule
\end{tabular}
}
\end{small}
\end{table*}

\subsection{Additional Planner Ablations}
\label{app:supp_planner_ablation}

Table~\ref{tab:planner_internal_host_switch} evaluates the sensitivity of Stage B to the frozen opponent used inside the inference-time rollout. The matched self-play AWM gives the best combined results. Training with a nominal world model or a held-out external AWM remains functional, but the resulting planner is weaker. This indicates that the learned opponent distribution is most effective when the planner is optimized against the same calibrated AWM used to define the adversarial risk distribution.

\begin{table*}[!htbp]
\caption{Inference-time opponent model sensitivity. Open-loop rewards are on the planner's rollout utility scale; closed-loop (CL) score and submetrics are ordinary nuPlan simulation. \textsc{Adv. R} is the adaptive-calibrated adversarial reward, \textsc{Tail} is its Tail-CVaR, and Prog./Coll./TTC denote progress, collision, and time-to-collision submetrics.}
\label{tab:planner_internal_host_switch}
\centering
\setlength{\tabcolsep}{4pt}
\renewcommand{\arraystretch}{1.1}
\begin{small}
\resizebox{0.95\linewidth}{!}{%
\begin{tabular}{@{}lccc|ccc|ccc@{}}
\toprule
\multirow{2}{*}{Internal opponent host} & \multicolumn{3}{c|}{Open-loop rollout} & \multicolumn{3}{c|}{CL: Test14-hard NR} & \multicolumn{3}{c}{CL: Test14-random R} \\
 & Nom. R $\uparrow$ & Adv. R $\uparrow$ & Tail $\uparrow$ & Score $\uparrow$ & Prog. $\uparrow$ & TTC $\uparrow$ & Score $\uparrow$ & Coll. $\uparrow$ & TTC $\uparrow$ \\
\midrule
Nominal WM & $0.9815$ & -- & -- & $79.41$ & $88.83$ & \textbf{82.10} & $90.36$ & $97.75$ & $94.60$ \\
\cellcolor{bestblue}{Matched self-play AWM} & $0.9818$ & $0.9756$ & $0.8919$ & \textbf{80.05} & \textbf{90.16} & 81.99 & \textbf{90.71} & \textbf{98.08} & \textbf{95.40} \\
Held-out external AWM & $0.9809$ & $0.9738$ & $0.8825$ & $79.29$ & $89.11$ & $80.94$ & $90.25$ & $97.64$ & $94.23$ \\
\bottomrule
\end{tabular}
}
\end{small}
\end{table*}

Table~\ref{tab:planner_hparam_mini} and Figure~\ref{fig:hparam_sensitivity} report hyperparameter sensitivity. The main trend is a trade-off between robustness and nominal performance: stronger tail-risk pressure or lower admission thresholds can improve adversarial tail metrics, but overly aggressive settings reduce closed-loop performance. Stronger retention preserves nominal behavior but can under-emphasize adversarial recovery. 

\begin{table}[!htbp]
\caption{Hyperparameter sensitivity around the selected planner operating point. The base row uses the default Stage-B setting: CVaR tail mass $\tau{=}0.25$, regret-CVaR weight $1.0$, progress-retention weight $10.0$, and training-time AWM threshold $0.45$. \textsc{Hard NR} and \textsc{Random R} are ordinary nuPlan closed-loop planner scores reported as percentages, not AWM sim-agent scores. \textsc{Sel./scene} and \textsc{Pair frac.} describe the AWM exposure density used by the training rollout.}
\label{tab:planner_hparam_mini}
\centering
\setlength{\tabcolsep}{3pt}
\renewcommand{\arraystretch}{1.1}
\begin{small}
\resizebox{0.95\linewidth}{!}{%
\begin{tabular}{@{}lcccccccc@{}}
\toprule
\multirow{2}{*}{Variant} & \multicolumn{4}{c}{Open-loop rollout} & \multicolumn{2}{c}{Planner CL} & \multicolumn{2}{c}{AWM exposure} \\
\cmidrule(lr){2-5}\cmidrule(lr){6-7}\cmidrule(lr){8-9}
 & Nom. R $\uparrow$ & Nom. prog. $\uparrow$ & Adv. R $\uparrow$ & Tail $\uparrow$ & Hard NR $\uparrow$ & Random R $\uparrow$ & Sel./scene & Pair frac. \\
\midrule
\rowcolor{bestblue}Base & $0.9818$ & $0.9545$ & $0.9756$ & $0.8919$ & $80.05$ & $90.71$ & $1.34$ & $0.34$  \\
\midrule
CVaR tail $\tau{=}0.10$ & 0.9808 & 0.9487 & 0.9758 & 0.8885 & 79.55 & 90.20 & 1.34 & 0.34  \\
CVaR tail $\tau{=}0.50$ & 0.9814 & 0.9538 & 0.9734 & 0.8790 & 78.60 & 90.45 & 1.34 & 0.34  \\
Expected risk $\tau{=}1.00$ & 0.9816 & 0.9552 & 0.9712 & 0.8650 & 77.85 & 90.55 & 1.34 & 0.34  \\
\midrule
CVaR weight $0.0$ & 0.9820 & 0.9560 & 0.9690 & 0.8580 & 77.55 & 90.40 & 1.34 & 0.34  \\
CVaR weight $0.5$ & 0.9821 & 0.9550 & 0.9735 & 0.8835 & 79.20 & 90.65 & 1.34 & 0.34  \\
CVaR weight $2.0$ & 0.9806 & 0.9490 & 0.9760 & 0.8940 & 79.70 & 89.90 & 1.34 & 0.34  \\
\midrule
Progress weight $5.0$ & 0.9807 & 0.9475 & 0.9760 & 0.8940 & 79.75 & 89.85 & 1.34 & 0.34  \\
Progress weight $20.0$ & 0.9819 & 0.9560 & 0.9720 & 0.8750 & 78.90 & 90.82 & 1.34 & 0.34  \\
\midrule
AWM threshold $0.10$ & 0.9809 & 0.9502 & 0.9748 & 0.8868 & 79.20 & 90.15 & 1.52 & 0.52  \\
AWM threshold $0.20$ & 0.9812 & 0.9524 & 0.9752 & 0.8895 & 79.55 & 90.38 & 1.45 & 0.45  \\
\bottomrule
\end{tabular}
}
\end{small}
\end{table}

\begin{figure}[!htbp]
\centering
\includegraphics[width=0.85\linewidth]{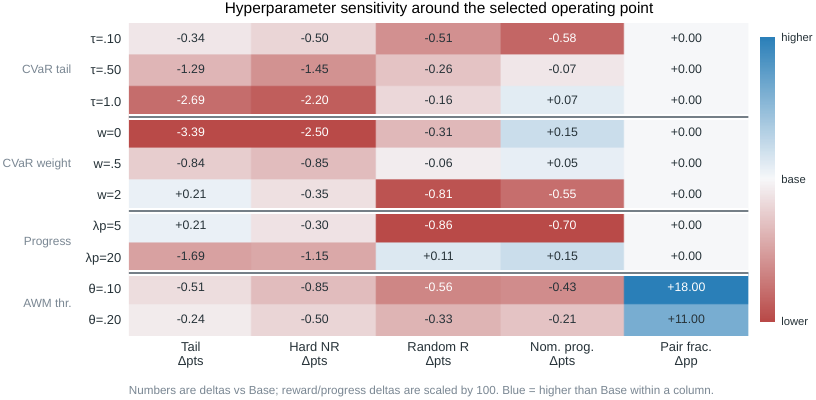}
\caption{Sensitivity around the selected planner operating point. The default configuration balances adversarial tail improvement with ordinary closed-loop performance, while more aggressive risk or weaker retention settings move along the robustness and nominal performance trade-off.}
\label{fig:hparam_sensitivity}
\end{figure}

\subsection{Adapted Prior Adversarial Generators}
\label{app:supp_prior_generators}

Table~\ref{tab:prior_generator_pilot} gives the full adapted generator comparison summarized in Figure~\ref{fig:main_quantitative_diagnostics}(b). The comparison separates adversarial strength from realism and search budget by adapting prior generators to the same autoregressive rollout interface. Single-candidate CAT$^*$ and STRIVE$^*$ produce relatively weak tail stress, while best-4 search strengthens all generators by selecting more damaging samples. The KING-style aggressive bound obtains the largest raw drop, but at a substantial cost in ADE/FDE and background-collision proxies.

\begin{table*}[!htbp]
\caption{Adapted prior generator comparison. $^*$ denotes a prior baseline adapted to the autoregressive rollout interface, so all methods share the same scene state, tokenized transition, planner reward, and metrics. \textsc{Protocol} separates search and constraints: \textsc{cand.=1} uses one generated candidate without reward-guided reranking; \textsc{best-4 search} generates four candidates and selects the lowest-reward rollout; \textsc{strict constraint} strengthens behavioral and realism penalties; and \textsc{kinematic-search} is an intentionally aggressive baseline used as a raw-attack upper bound. \textsc{Off./Stat.} reports off-road and static-object collision violation; \textsc{BG-coll.} reports the background agent collision.}
\label{tab:prior_generator_pilot}
\centering
\setlength{\tabcolsep}{3pt}
\renewcommand{\arraystretch}{1.1}
\begin{small}
\resizebox{0.99\linewidth}{!}{%
\begin{tabular}{@{}lllcccccccc@{}}
\toprule
Setting & Method & Protocol & Drop $\uparrow$ & Worst $\uparrow$ & Tail $\uparrow$ & ADE $\downarrow$ & FDE $\downarrow$ & Kin. $\downarrow$ & Off./Stat. $\downarrow$ & BG-coll. $\downarrow$ \\
\midrule
\rowcolor{bestblue}\textbf{Full AWM} & \textbf{Full AWM} & adaptive coalition & 0.0064 & 0.0500 & 0.2078 & 1.55 & 4.95 & 0.0108 & 0.280/0.009 & 0.081 \\
\midrule
Single candidate & CAT$^*$~\cite{zhang2023cat} & cand.=1 sample & 0.0049 & 0.0230 & 0.0917 & 2.46 & 7.71 & 0.0309 & 0.380/0.016 & 0.147 \\
Single candidate & STRIVE$^*$~\cite{rempe2022generating} & cand.=1 sample & 0.0054 & 0.0245 & 0.1007 & 3.47 & 7.89 & 0.0181 & 0.381/0.016 & 0.148 \\
\midrule
Matched search & AWM & best-4 search & 0.0242 & 0.0743 & 0.3853 & 1.47 & 4.73 & 0.0082 & 0.300/0.013 & 0.090 \\
Matched search & CAT$^*$~\cite{zhang2023cat} & best-4 search & 0.0258 & 0.0774 & 0.2530 & 2.65 & 8.02 & 0.0140 & 0.424/0.020 & 0.165 \\
Matched search & STRIVE$^*$~\cite{rempe2022generating} & best-4 search & 0.0204 & 0.0649 & 0.2473 & 3.88 & 8.17 & 0.0123 & 0.406/0.018 & 0.155 \\
Aggressive bound & KING$^*$~\cite{hanselmann2022king} & kinematic-search (upp. bound) & 0.1249 & 0.1542 & 0.3901 & 5.06 & 10.57 & 0.0211 & 0.579/0.046 & 0.258 \\
\midrule
Realism control & CAT$^*$~\cite{zhang2023cat} & strict constraint & 0.0141 & 0.0503 & 0.1174 & 1.45 & 4.79 & 0.0099 & 0.359/0.011 & 0.145 \\
Realism control & STRIVE$^*$~\cite{rempe2022generating} & strict constraint & 0.0113 & 0.0426 & 0.1606 & 1.32 & 4.57 & 0.0077 & 0.372/0.013 & 0.144 \\
\bottomrule
\end{tabular}
}
\end{small}
\end{table*}

\newpage

\subsection{Supplementary Visualizations}
\label{app:supp_visualization}

\subsubsection{\texttt{InterPlan-LongTail} Examples}
\label{app:interplan_visualization}

Figure~\ref{fig:app_interplan_group} shows representative \texttt{InterPlan-LongTail} cases. These examples complement Table~\ref{tab:interplan_longtail_templates} by illustrating the dense interaction patterns introduced by the long-tail scenario modifications.

\begin{figure}[!htbp]
\centering
\includegraphics[width=\linewidth]{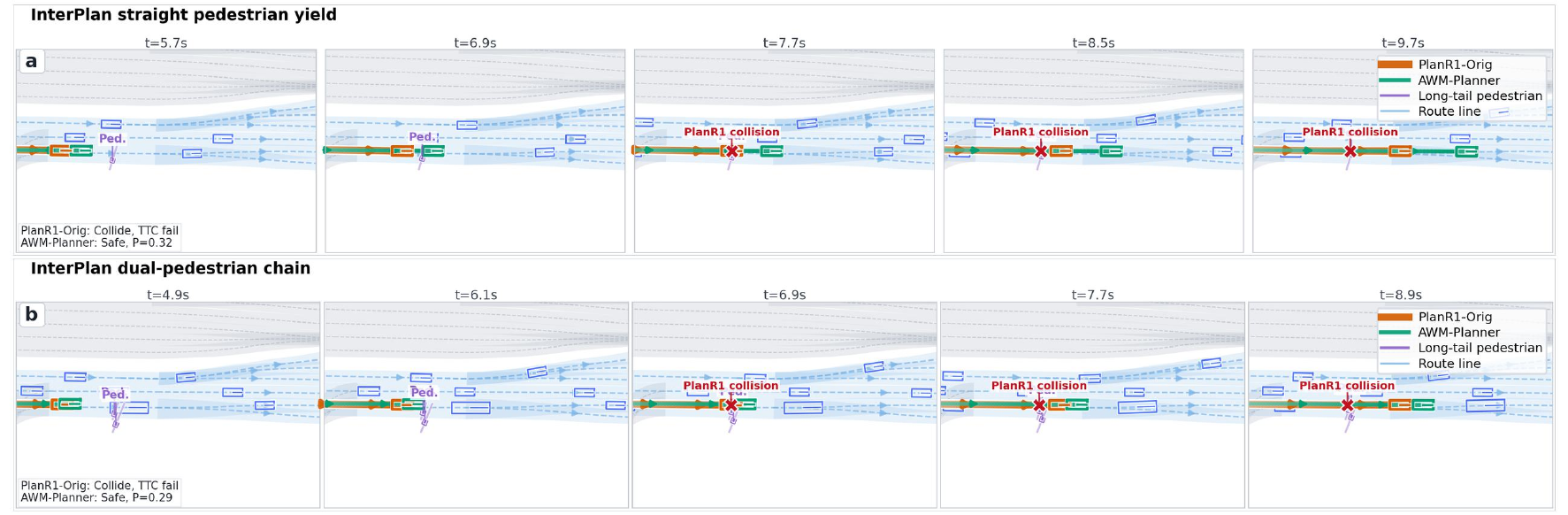}
\caption{\texttt{InterPlan-LongTail} pedestrian-yield visualizations. (a) Straight pedestrian-yield case. (b) Dual-pedestrian chain case.}
\label{fig:app_interplan_group}
\end{figure}

\subsubsection{Time-Unrolled Robust-Recovery Cases}
\label{app:recovery_visualization}

Figures~\ref{fig:app_recovery_group_a} and~\ref{fig:app_recovery_group_b} provide time-unrolled visualizations of the recovery cases in Figure~\ref{fig:main_qual_recovery}. Each row compares Plan-R1 and AWM-Planner in the same adversarial scene. The temporal sequences show that AWM-Planner typically responds earlier, maintains a safer interaction envelope, and avoids the collision or TTC failure reached by the original planner.

\begin{figure}[!htbp]
\centering
\includegraphics[width=\linewidth]{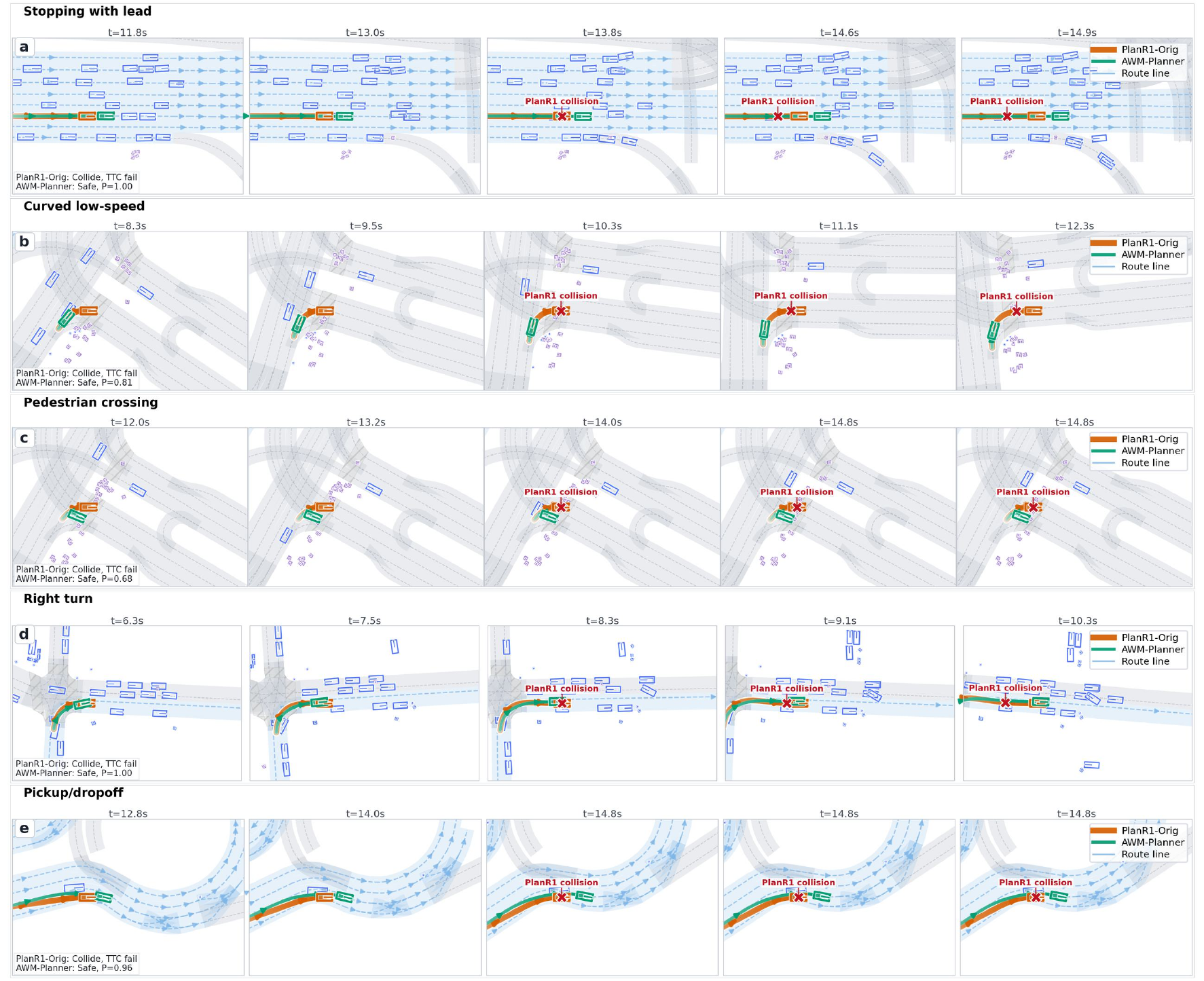}
\caption{Time-unrolled robust-recovery visualizations on \texttt{NuPlan}, group I. (a) Stopping with lead. (b) Curved low-speed interaction. (c) Pedestrian crossing. (d) Right turn. (e) Pickup/dropoff interaction.}
\label{fig:app_recovery_group_a}
\end{figure}

\begin{figure}[!htbp]
\centering
\includegraphics[width=\linewidth]{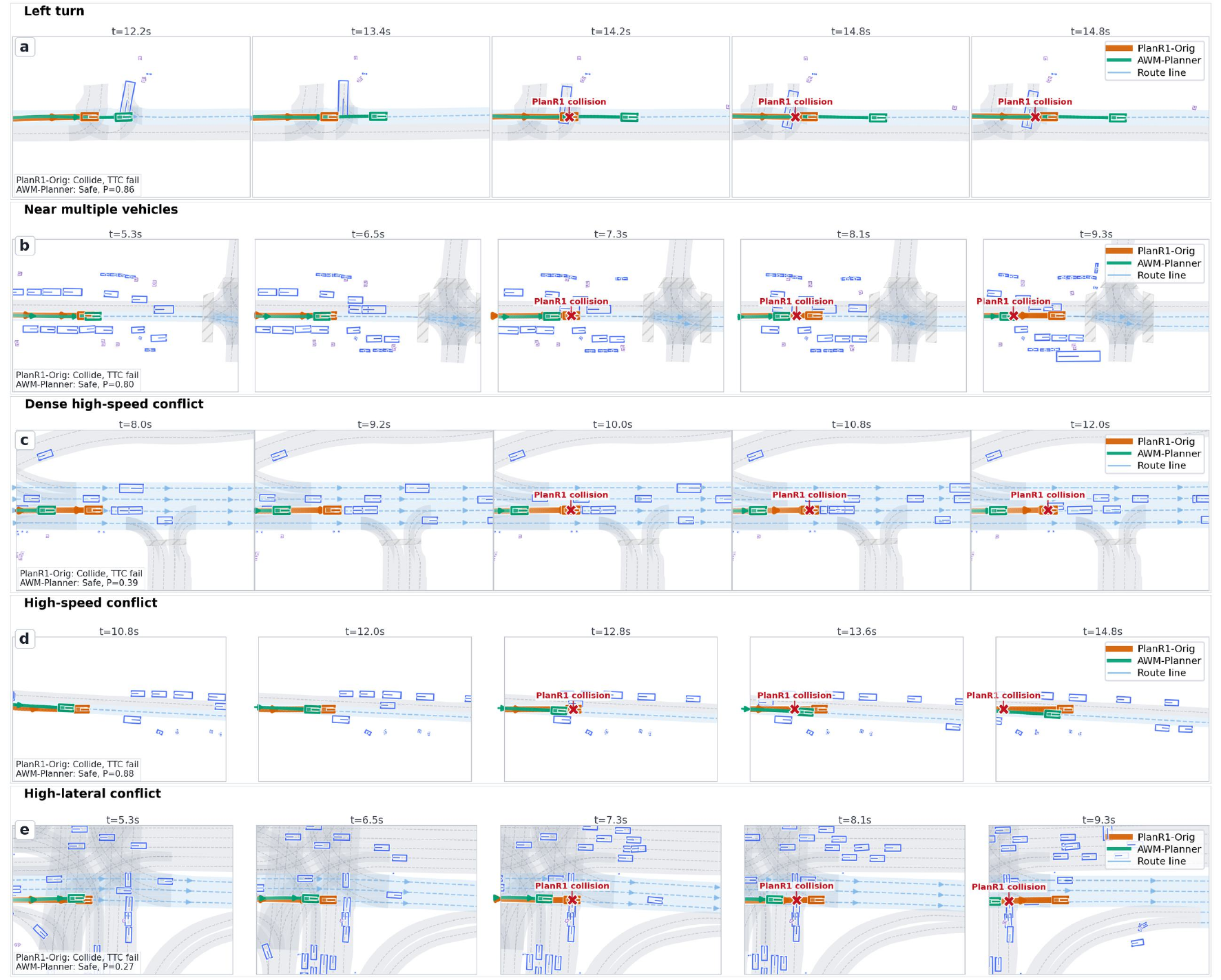}
\caption{Time-unrolled robust-recovery visualizations on \texttt{NuPlan}, group II. (a) Left turn. (b) Dense traffic near multiple vehicles. (c) Dense high-speed conflict. (d) High-speed conflict. (e) High-lateral conflict.}
\label{fig:app_recovery_group_b}
\end{figure}

\clearpage
\subsubsection{K=2 Cooperative Attack Cases}
\label{app:k2_visualization}

Figure~\ref{fig:app_k2_group} provides additional \(K{=}2\) cooperative attack visualizations. These cases complement the transfer evaluation in Table~\ref{tab:cross_planner_simagent} by showing how two adversarial vehicles can jointly restrict the ego planner while preserving a coherent traffic scene.

\begin{figure}[!htbp]
\centering
\includegraphics[width=\linewidth]{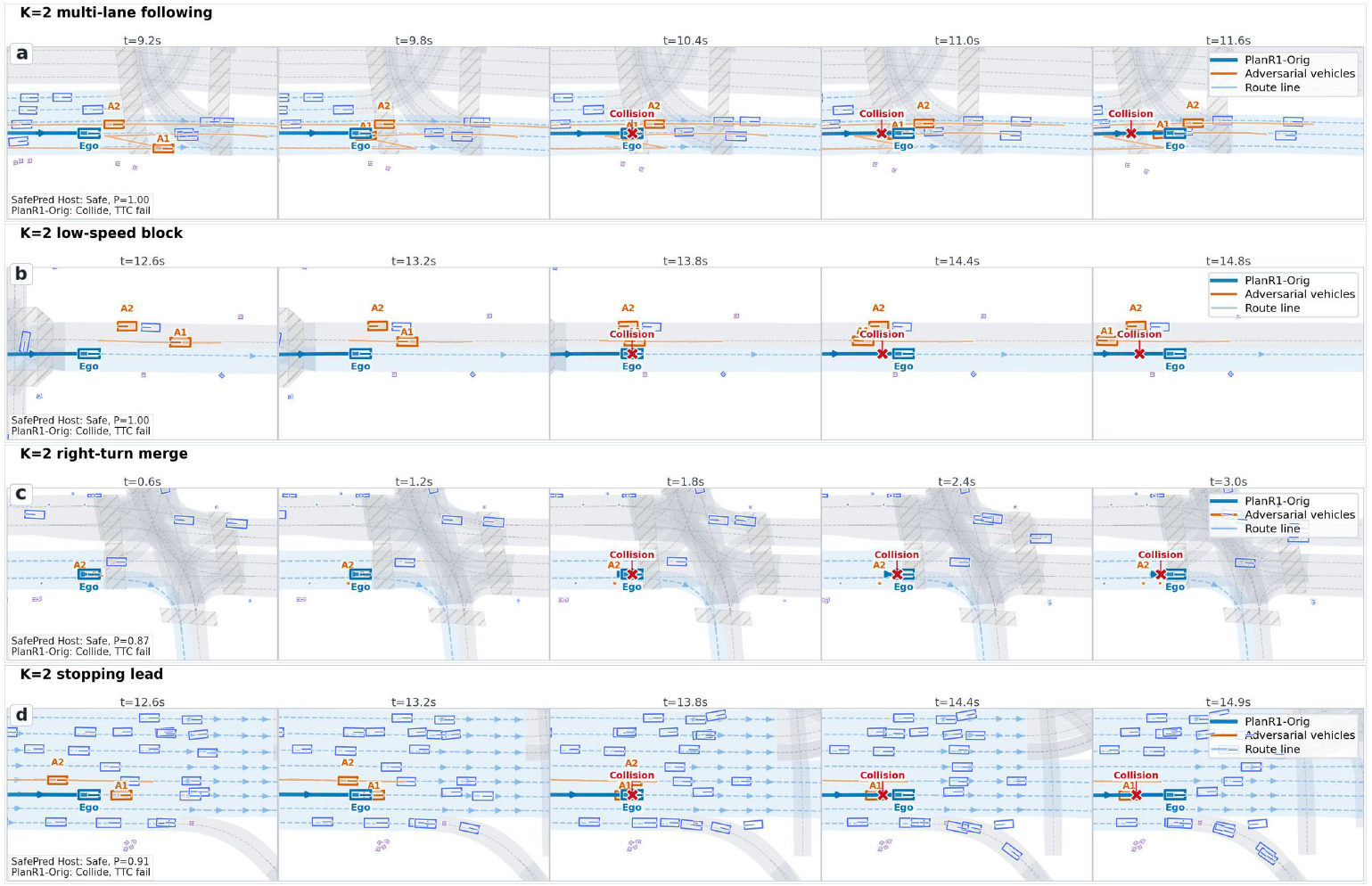}
\caption{K=2 cooperative attack visualizations. (a) Multi-lane following. (b) Low-speed blocking. (c) Right-turn merge. (d) Stopping with lead.}
\label{fig:app_k2_group}
\end{figure}

\stopcontents[appendix]
\clearpage
\input{checklist.tex}

\end{document}

%% file: math_commands.tex
\usepackage{amsmath,amsfonts,bm}

\def\eqref#1{equation~\ref{#1}}

\def\1{\bm{1}}

\def\vz{{\bm{z}}}

\DeclareMathAlphabet{\mathsfit}{\encodingdefault}{\sfdefault}{m}{sl}
\SetMathAlphabet{\mathsfit}{bold}{\encodingdefault}{\sfdefault}{bx}{n}



%% file: checklist.tex
\section*{NeurIPS Paper Checklist}

\begin{enumerate}

\item {\bf Claims}
    \item[] Question: Do the main claims made in the abstract and introduction accurately reflect the paper's contributions and scope?
    \item[] Answer: \answerYes{}
    \item[] Justification: The abstract and Section~\ref{sec:intro} state the proposed AWM framework, its two-stage scope, and the claimed empirical gains. The supporting method, theoretical analysis, and experiments are provided in Sections~\ref{sec:method}--\ref{sec:experiment} and Appendices~\ref{app:theory_final}--\ref{app:exp_setups}.

\item {\bf Limitations}
    \item[] Question: Does the paper discuss the limitations of the work performed by the authors?
    \item[] Answer: \answerYes{}
    \item[] Justification: Appendix~\ref{app:theory_final} discusses the local and surrogate nature of the theoretical guarantees, approximation errors, Monte Carlo counterfactual variance, and trust-region assumptions. Section~\ref{sec:experiment} and Appendix~\ref{app:exp_setups} further discuss attack--realism and robustness--nominal-performance trade-offs.

\item {\bf Theory assumptions and proofs}
    \item[] Question: For each theoretical result, does the paper provide the full set of assumptions and a complete (and correct) proof?
    \item[] Answer: \answerYes{}
    \item[] Justification: Appendix~\ref{app:theory_final} states the assumptions and notation used for the theoretical analysis and provides proofs for the lemma and propositions supporting the decoupled solver, credit assignment, CVaR weighting, nominal retention, and safety-dual surrogate.

    \item {\bf Experimental result reproducibility}
    \item[] Question: Does the paper fully disclose all the information needed to reproduce the main experimental results of the paper to the extent that it affects the main claims and/or conclusions of the paper (regardless of whether the code and data are provided or not)?
    \item[] Answer: \answerYes{}
    \item[] Justification: The architecture, two-stage optimization procedure, training splits, hyperparameters, evaluation protocols, metric definitions, and benchmark construction are described in Section~\ref{sec:method}, Section~\ref{sec:experiment}, and Appendices~\ref{app:exp_setups}--\ref{app:interplan_longtail}.

\item {\bf Open access to data and code}
    \item[] Question: Does the paper provide open access to the data and code, with sufficient instructions to faithfully reproduce the main experimental results, as described in supplemental material?
    \item[] Answer: \answerNo{}
    \item[] Justification: The code is not publicly available. Appendices~\ref{app:extended_details} and~\ref{app:exp_setups} describe the implementation, datasets, training procedures, and evaluation protocol.

\item {\bf Experimental setting/details}
    \item[] Question: Does the paper specify all the training and test details (e.g., data splits, hyperparameters, how they were chosen, type of optimizer) necessary to understand the results?
    \item[] Answer: \answerYes{}
    \item[] Justification: Appendix~\ref{app:training_recipe} specifies data splits, architecture settings, learning rates, epochs, optimizer-style rollout settings, reward weights, adversarial budgets, and Stage-A/Stage-B fine-tuning details. Appendix~\ref{app:metric_conventions} describes the evaluation protocols and metric conventions.

\item {\bf Experiment statistical significance}
    \item[] Question: Does the paper report error bars suitably and correctly defined or other appropriate information about the statistical significance of the experiments?
    \item[] Answer: \answerYes{}
    \item[] Justification: The experiments use fixed held-out benchmark scenario lists and paired scene-level comparisons where applicable, as described in Appendix~\ref{app:metric_conventions}. The paper also reports seed-wise stability statistics in Table~\ref{tab:training_paradigm} and lower-confidence-bound diagnostics based on standard errors in Table~\ref{tab:planner_offline_branch_matrix}.

\item {\bf Experiments compute resources}
    \item[] Question: For each experiment, does the paper provide sufficient information on the computer resources (type of compute workers, memory, time of execution) needed to reproduce the experiments?
    \item[] Answer: \answerYes{}
    \item[] Justification: Appendix~\ref{app:training_recipe} reports the GPU type, number of GPUs, distributed training setup, and effective batch sizes. Table~\ref{tab:training_paradigm} additionally reports relative GPU-hours, plateau update counts, seed stability, and seed-wise closed-loop variability for the compared training paradigms.
    
\item {\bf Code of ethics}
    \item[] Question: Does the research conducted in the paper conform, in every respect, with the NeurIPS Code of Ethics \url{https://neurips.cc/public/EthicsGuidelines}?
    \item[] Answer: \answerYes{}
    \item[] Justification: The work is conducted on public autonomous-driving benchmarks and simulation-based stress tests, without collecting private data, involving human subjects, or deploying a real-world driving system. The adversarial component is constrained by feasibility and realism gates as described in Section~\ref{sec:experiment} and Appendix~\ref{app:exp_setups}.

\item {\bf Broader impacts}
    \item[] Question: Does the paper discuss both potential positive societal impacts and negative societal impacts of the work performed?
    \item[] Answer: \answerYes{}
    \item[] Justification: The Broader Impact paragraph discusses benefits for robustness evaluation and training in simulation, risks of misuse and false confidence, and mitigation through feasibility and realism checks and separate reporting of attack strength and planner performance.
    
\item {\bf Safeguards}
    \item[] Question: Does the paper describe safeguards that have been put in place for responsible release of data or models that have a high risk for misuse (e.g., pre-trained language models, image generators, or scraped datasets)?
    \item[] Answer: \answerNA{}
    \item[] Justification: The current paper does not release high-risk data or a broadly deployable model, and the experiments are limited to benchmark simulation. The adversarial traffic model is evaluated under feasibility and realism constraints rather than released as an unrestricted deployment artifact.

\Needspace{6\baselineskip}
\item {\bf Licenses for existing assets}
    \item[] Question: Are the creators or original owners of assets (e.g., code, data, models), used in the paper, properly credited and are the license and terms of use explicitly mentioned and properly respected?
    \item[] Answer: \answerYes{}
    \item[] Justification: Appendix~\ref{app:training_recipe} credits the existing datasets, benchmarks, and baseline methods, and explicitly states the license and access constraints for \texttt{nuPlan}, \texttt{InterPlan}, and third-party assets. Raw restricted datasets are not redistributed and must be obtained from their original sources.

\item {\bf New assets}
    \item[] Question: Are new assets introduced in the paper well documented and is the documentation provided alongside the assets?
    \item[] Answer: \answerNA{}
    \item[] Justification: This paper does not release new assets. Appendix~\ref{app:interplan_longtail} describes the construction of \texttt{InterPlan-LongTail}; its scenario lists and modification configurations are not publicly released.

\item {\bf Crowdsourcing and research with human subjects}
    \item[] Question: For crowdsourcing experiments and research with human subjects, does the paper include the full text of instructions given to participants and screenshots, if applicable, as well as details about compensation (if any)? 
    \item[] Answer: \answerNA{}
    \item[] Justification: The research does not involve crowdsourcing experiments or human-subject studies; it uses existing autonomous-driving benchmarks and simulator-based evaluations.

\item {\bf Institutional review board (IRB) approvals or equivalent for research with human subjects}
    \item[] Question: Does the paper describe potential risks incurred by study participants, whether such risks were disclosed to the subjects, and whether Institutional Review Board (IRB) approvals (or an equivalent approval/review based on the requirements of your country or institution) were obtained?
    \item[] Answer: \answerNA{}
    \item[] Justification: The research does not involve study participants or human-subject data collection, so IRB or equivalent review for human-subject experiments is not applicable.

\item {\bf Declaration of LLM usage}
    \item[] Question: Does the paper describe the usage of LLMs if it is an important, original, or non-standard component of the core methods in this research? Note that if the LLM is used only for writing, editing, or formatting purposes and does \emph{not} impact the core methodology, scientific rigor, or originality of the research, declaration is not required.
    \item[] Answer: \answerNA{}
    \item[] Justification: The core method does not use LLMs as an important, original, or non-standard research component. Mentions of LLMs are used only as motivation for discrete autoregressive modeling, not as part of the proposed algorithm.

\end{enumerate}